%% file: main.tex
\documentclass[twoside,11pt]{article}

\usepackage{blindtext}
\usepackage[preprint]{jmlr2e}
\usepackage{amssymb}
\usepackage{pifont}
\usepackage[utf8]{inputenc} 
\usepackage[T1]{fontenc}    
\usepackage{url}            
\usepackage{booktabs}       
\usepackage{subcaption}
\usepackage{nicefrac}       
\usepackage{microtype}      
\usepackage{enumitem}
\usepackage{booktabs}
\usepackage{comment}
\usepackage{xcolor}
\hypersetup{
    colorlinks,
    linkcolor={red!50!black},
    citecolor={blue!50!black},
    urlcolor={blue!80!black}
}

\usepackage{todonotes}
\usepackage{amsfonts,amssymb,amsmath,xcolor,aliascnt}
\usepackage{cleveref}

\usepackage{wrapfig}
\usepackage[linesnumbered,vlined,ruled,commentsnumbered]{algorithm2e}

\SetCommentSty{mycommfont}

\usepackage{graphicx}

\newcommand{\todooptionnal}[1]{}
\newcommand{\tochangeinmain}[1]{\textcolor{black}{#1}}

\usepackage{lastpage}
\jmlrheading{23}{2022}{1-\pageref{LastPage}}{7/21; Revised -/-}{-/-}{21-0000}{A. Dieuleveut and G. Fort and M. Hegazy and H.-T. Wai}

\ShortHeadings{Federated Majorize-Minimization}{Federated Majorize-Minimization}
\firstpageno{1}

\input{macros}
\title{Federated Majorize-Minimization: Beyond Parameter Aggregation}

\author{\name Aymeric Dieuleveut \email aymeric.dieuleveut@polytechnique.edu \\
\addr CMAP, École polytechnique, \\
Insitut Polytechnique de Paris \\
 \AND
\name Gersende Fort \email gersende.fort@laas.fr \\
\addr LAAS-CNRS, Université de Toulouse, CNRS,\\
  Toulouse, France \\
  \AND 
\name Mahmoud Hegazy \email mahmoud.hegazy@polytechnique.edu \\
\addr CMAP, École polytechnique, \\
Insitut Polytechnique de Paris \\
  \AND 
\name Hoi-To Wai \email htwai@se.cuhk.edu.hk \\
\addr Department of SEEM, The Chinese University of Hong Kong \\
Shatin, Hong Kong
}

\editor{My editor}

\newtheorem{assumption}{\textbf{A}\!\!}

\def\B{\mathsf{B}}
\begin{document}

\maketitle

\begin{abstract}
This paper proposes a unified approach for designing stochastic optimization algorithms that robustly scale to the federated learning setting. Our work  studies a class of Majorize-Minimization (MM)  problems, which possesses a linearly parameterized family of majorizing surrogate functions. This framework encompasses (proximal) gradient-based algorithms for (regularized) smooth objectives, the Expectation Maximization algorithm, and many problems seen as variational surrogate MM. We show that our framework motivates a unifying algorithm called Stochastic Approximation Stochastic Surrogate MM (\SSMM), which includes previous stochastic MM procedures as special instances. We then extend \SSMM\ to the federated setting, while taking into consideration common bottlenecks such as data heterogeneity, partial participation, and communication constraints; this yields \QSMM. The originality of \QSMM\ is to learn locally and then aggregate information characterizing the \textit{surrogate majorizing function}, contrary to classical algorithms which learn and aggregate the \textit{original parameter}.  Finally, to showcase the flexibility of this methodology beyond our theoretical setting, we use it to design an algorithm for computing optimal transport maps in the federated setting. 
\end{abstract}

\input{0_intro}

\input{1_stochasticMM}

\input{2_fedmm}

\input{3_theory}

\input{4_experiments}

\input{4_optimal_transport}

\input{5_conclusion}
 
\appendix
\include{appendix}

\include{appendix_proof}

\bibliography{ref}

\end{document}

%% file: macros.tex
\newcommand{\norm}[1]{\left \lVert #1 \right \rVert}
\newcommand{\open}[1]{\left ( #1 \right )}
\newcommand{\closed}[1]{\left [#1 \right]}
\newcommand{\ens}[1]{\left \{ #1 \right \} }

\newcommand\inner[2]{\left \langle #1, #2 \right \rangle}
\newcommand{\ce}[3][]{\mathbb{E}_{#1}\left[ #2 \,\middle|\, #3 \right]}

 \newcommand{\Smem}{\mathsf{S}}
\def\xcal{\mathcal{X}}
\def\pcal{\mathcal{P}}
\def\qcal{\mathcal{Q}}

\def\proba{\mathsf{p}}
\def\set{\mathcal A}
\def\Q{\mathrm{Quant}}
\def\Qtilde{\widetilde{\mathrm{Quant}}} 
\def\rset{\mathbb{R}}
\def\nset{\mathbb{N}} 
 
\def\calH{\mathcal{H}}
\def\eqdef{:=} 

\def\map{\mathsf{T}}
\def\mapM{\mathsf{M}}
\def\PE{\mathbb{E}}

\def\pas{\gamma} 
\def\param{\theta}
\def\paramset{\Theta} 

\def\L{\mathsf{L}}
\def\loss{\ell} 
\def\tildeloss{\widetilde{\ell}}
\def\sample{Z} 
\def\barS{\bar{S}} 

\def\surspace{\mathcal{S}}
\def\hatS{\hat{S}} 

\def\calS{\mathcal{S}} 

\def\rmd{\mathsf{d}} 
 \def\Prox{\mathrm{prox}}
\def\lawPP{\mathsf{P}}  \def\kmax{T_{\max}}
\def\lbatch{\mathsf{b}} \def\batch{\mathcal{B}}
\def\inter{\mathrm{int}} \def\mf{\mathsf{h}}
 
\def\argmin{\mathrm{argmin}}
\def\argmax{\mathrm{argmax}}

\def\F{\mathcal{F}}

\def\MM{{\tt MM}}

\def\SSMM{{\tt SA-SSMM}}
\def\QSMM{{\tt FedMM}}
\def\fedot{{\tt FedMM-OT}}

\DeclareMathOperator{\expec}{\mathbb{E}}

\newcommand{\CPE}[3][]
{\ifthenelse{\equal{#1}{}}{{\mathbb E}\left[\left. #2 \, \right| #3 \right]}{{\mathbb E}_{#1}\left[\left. #2 \, \right | #3 \right]}}

\newcommand{\pscal}[2]{\left\langle#1,#2\right\rangle}

\newcommand{\ocint}[1]{\left(#1\right]}
\newcommand{\ooint}[1]{\left(#1\right)}
\newcommand{\ccint}[1]{\left[#1\right]}
\def\Id{\mathrm{I}}
\def\1{\mathsf{1}}

\def\error{\mathcal{E}}
\def\barerror{\mathcal{\bar{E}}}
\def\MM12{{\bf {\sf MM1\&2}}}
\newcommand{\ncal}{\mathcal{N}}
\newcommand{\tcal}{\mathcal{T}}
\newcommand{\rcal}{\mathcal{R}}

\newcommand{\old}[1]{\textbf{\textcolor{lightgray}{}}}

\newtheorem{example}{Example} 
\newtheorem{theorem}{Theorem}
\newtheorem{lemma}{Lemma} 
\newtheorem{proposition}{Proposition} 
\newtheorem{remark}{Remark}
\newtheorem{corollary}{Corollary}

\Crefname{assumption}{A\hspace{-3pt}}{A\hspace{-3pt}}
\crefname{assumption}{A}{A}
\Crefname{algocf}{Algorithm}{Algorithms}
\crefname{algocf}{alg.}{algs.}



%% file: 0_intro.tex
\section{Introduction}

Consider the following optimization problem
\begin{equation} \label{eq:opt}
\argmin_{\param \in \rset^d} ~W(\param) \eqdef \left( f(\param) + g(\param) \right),
\qquad \text{with} \quad f(\param) \eqdef \PE_{Z\sim \pi} \left[ \ell(\sample,
  \param) \right];
\end{equation}
where $g: \rset^d \to \ocint{- \infty, + \infty}$ is a proper lower
semi-continuous convex function with domain $\paramset \subseteq
\rset^d$, $\ell: \rset^p \times \paramset\rightarrow\rset $ is a
measurable function such that $\PE_\pi\left[ |\ell(\sample, \param)
  |\right] < + \infty$ for any $\param \in \paramset$ and $\pi$ denotes
a probability measure on $\rset^p$. This setting covers both the {\it batch learning} when $\pi$ is the empirical distribution of a (possibly large) data set $\{\sample_1, \cdots, \sample_N\}$ and we solve 
$
\argmin_{\param \in \rset^d} (N^{-1} \sum_{j=1}^N
\ell(\sample_j, \param) + g(\param) )
$
and the {\it online learning} when we have to process an i.i.d.~stream of data
$\{\sample_1, \sample_2, \ldots \}$ with the  distribution $\pi$. 

We are interested in the setting where \textit{majorizing functions} can be computed on $f$, which can then be minimized. This is a framework referred to as Majorize-Minimization (MM)~\citep{lange:2013}, or surrogate optimization~\citep{lange2000optimization}. 
More concretely, starting from the parameter $\param_t$, a standard MM algorithm constructs a surrogate $U_t:\Theta\rightarrow \rset$ such that $f(\param_t) = U_t(\param_t)$ and $f(\param)\leq U_t(\param)$ for any $\param \in \Theta$. Then, a parameter is constructed such that $\param_{t+1}\in \argmin \;U_t(\theta) + g(\theta)$. 
Many methods can be interpreted from this point of view, e.g., gradient-based and proximal methods~\citep{beck2009fast,wright2009sparse}, Expectation-Maximization (EM)~\citep{dempster:etal:1977}, boosting~\citep{collins2002logistic,della2001duality}, and  matrix factorization (MF) algorithms such as non-negative MF~\citep{paatero1994positive,lee1999learning} or clustering~\citep{jain1999data}.

In a series of papers that inspired our work, \citet{mairal:2013, mairal:2015} investigated the use of MM methods in large-scale machine learning. The author proposed to incrementally build the sequence of surrogates by leveraging $U_1,\ldots,U_t$ in constructing $U_{t+1}$.  We build on this approach and demonstrate that in many examples, \textit{the space of the surrogate functions can be parameterized by a finite-dimensional space} $\calS$; i.e., each $s \in \calS$ represents the mirror parameter characterizing the surrogate function $U( \cdot , s )$.  We then perform (stochastic) incremental updates in the space of surrogate functions instead of  the space of parameters $\param$. 

This paper is particularly drawn to the Federated Learning (FL) setting \citep{kairouz_advances_2019}, where data is shared between several agents and a central server aggregates the contributions of each agent. More concretely, we consider the extension of \eqref{eq:opt} to the federated setting, where $n$ different \textit{agents} collaborate to solve 
\begin{equation}\label{eq:intro:opt-fed}\textstyle
\argmin_{\param \in \rset^d}   \sum_{i =1}^n \mu_i f_i(\param)  + g(\param) ,
\end{equation}
where $f_i(\theta) \eqdef \expec_{\pi_i}\closed{\ell\open{Z,\theta}}$ and $\mu_i > 0$ is a weighting parameter, $\sum_{i=1}^n \mu_i= 1$. Here, for $i\in \{1, \cdots, n\}$, we denote by $\pi_i$ the probability distribution of the data of agent $\# i$. In addition, for the $i$-th agent, we consider the surrogate $U:\Theta\times \calS \rightarrow \rset$ satisfying:  for all $\param$, there exists $s_\theta^i \in \calS$ such that $f_i(\theta) = U(\theta,s_\theta^i)$ and $f_i(\cdot) \leq U\open{\cdot, s_\theta^i }$ for all $\theta \in \Theta$.

Under this setting, our main motivation is that for many problems, the surrogate function $U(\theta, s)$ is linear in the mirror parameter $s$. Thus, a direct aggregation in the surrogate space corresponds to aggregation in the functional space of upper bounds. This observation enables us to solve \eqref{eq:intro:opt-fed} via aggregating the mirror parameters in a distributed manner, resulting in a MM-like algorithm for the FL setting. 
On the other hand, a naive FL extension of a centralized optimization algorithm through aggregating the clients' parameters $\theta$ (e.g., FedAvg \citep{mcmahan_communication-efficient_2017}) may not result in a stable algorithm. 
From a methodological standpoint, our key message, which we further detail in~\Cref{sec:SspaceORThetaspace}, is: 
\begin{center}
    \textit{In numerous FL problems, the natural way to aggregate agents' contributions is by first aggregating on the surrogate space to construct a single surrogate function and then to minimize/majorize the common surrogate on the server}
\end{center}

In FL, the communication between agents and the central server has been identified as a crucial bottleneck, that has been widely tackled in the gradient-based literature~\citep[e.g.][]{mcmahan2016communication,alistarh_qsgd_2017}. Moreover, the heterogeneity between the multiple agents' distributions often hinders convergence~\citep{li_convergence_2019,sattler_robust_2019}. We leverage the stochastic surrogate MM approach in FL to propose a communication-efficient method. Following the works of \citep{mishchenko_distributed_2019}, we incorporate control variates to handle heterogeneity. This enables us to propose a generic FL-MM algorithm. In particular cases, these algorithms can be seen as modifications or extensions encompassing multiple existing federated algorithms such as  Federated EM~\citep{dieuleveut:etal:FedEM:2021,marfoq2021federated} and  Federated MF~\citep{wang:chang:2022}.

\paragraph{Outline} For the rest of this paper, \Cref{sec:linearly_param} will start by considering the centralized case, where we will introduce an MM framework that unifies many common settings. Using this framework, we show that multiple celebrated algorithms can be seen through the lens of surrogate optimization. We provide \SSMM{}, a Stochastic Surrogate MM method relying on Stochastic Approximation \citep{robbins_stochastic_1951}, for the centralized case. Then, we give the natural extension, named \QSMM, of the algorithm to the federated setting in \Cref{sec:FLMM}. Further, \Cref{sec:related_work} discusses the related works in the Federated Learning literature. Then, in \Cref{sec:theory}, we provide our theoretical results. To evaluate our method, we consider a federated dictionary learning problem in \Cref{sec:dict_learn}. Finally, \Cref{sec:ot} shows how our approach can be regarded as a more general methodology for the design of federated algorithm. To this end, we consider a federated optimal transport problem and provide an algorithm to tackle it.   

\paragraph{Notations} For two vectors $a,b \in \rset^q$, $\pscal{a}{b}$
denotes the Euclidean scalar product; we set $\|a\| \eqdef\sqrt{\pscal{a}{a}}$.  All vectors are column-vectors; for a matrix $A$, $A^\top$ is its transpose. For two $p \times q$ matrices $A$ and $B$, $\pscal{A}{B} \eqdef \mathrm{Trace}\left(A^\top B\right)$; the Frobenius norm of $A$ is also denoted by $\|A\|$. $\mathcal{M}_K^+$ is the set of the  $K \times K$ positive semi-definite  matrices; $A \succ 0$ denotes that $A$ is positive-definite.

The subdifferential of a proper function $f: \rset^d \to \ocint{-\infty, +\infty}$  is denoted by $\partial f$. The Jacobian matrix of a differentiable vector-valued function $f: \rset^d \to \rset^q$ is a $q \times d$ matrix,  denoted by
$J_f$; when $q=1$, we denote by $\nabla f(\cdot)$ the $d \times 1$
gradient vector. For $n\in \mathbb N$, set  $[n]:=\{1, \dots, n\}$.    $\inter \paramset$ is the interior of a set $\paramset$.

%% file: 1_stochasticMM.tex
\section{Linearly Parameterized Surrogate functions}\label{sec:linearly_param}
We first consider an MM framework where the surrogate functions are linearly parameterized by focusing on problems satisfying \textbf{\sf MM-1} and \textbf{\sf MM-2}.  
\begin{description}
  \item[\textbf{\sf MM-1}:] there exist a convex subset $\calS$ of $\rset^q$ and
    measurable functions $\barS: \rset^p \times \paramset \to \calS$,
    $\psi: \paramset \to \rset$ and $\phi: \paramset \to \rset^q$ such
    that for any $\tau \in \paramset$, $\PE_\pi
  \left[ \| \barS(\sample, \tau) \| \right]  < \infty $ and 
\begin{equation} \label{eq:defMM:surrogate}
f(\cdot) \leq f(\tau) + \psi(\cdot) - \psi(\tau) - \pscal{\PE_\pi
  \left[ \barS(\sample, \tau) \right]}{\phi(\cdot) - \phi(\tau)} \qquad \text{on $\paramset$}.
\end{equation}
\item[\textbf{\sf MM-2}:] the application $\map: \calS \to \paramset$ defined by 
  $
\map(s) \eqdef \argmin_{\param \in \rset^d} \left( g(\param) +
\psi(\param) - \pscal{s}{\phi(\param)} \right),
  $
  is well-defined (i.e., there is a unique minimizer for all $s \in  \calS$) and measurable.
\end{description}
In words, the surrogate function given by the right-hand side in \eqref{eq:defMM:surrogate} admits a \textit{parametric form}. 

This {\bf {\sf MM}} framework encompasses many examples of optimization problems. Hereafter, we give three examples, which are  detailed in \Cref{app:exampleMM1-2}.

\begin{example}[Quadratic surrogate] \label{ex:gdt}  Let $f: \paramset \to \rset$ be continuously differentiable on a neighborhood of
$\paramset$ and  its gradient be $L_f$-Lipschitz. Assume that there exists a
measurable  function $G$ such that $\nabla f(\cdot) =
\PE_\pi\left[G(\sample, \cdot) \right]$ on $ \paramset$. Then, for any  $\rho \in \ocint{0,\tochangeinmain{1/L_f}}$,   {\bf {\sf MM-1}} and {\bf {\sf MM-2}} are satisfied with
\begin{equation}
 \psi(\param) \eqdef \frac{1}{2 \rho } \|\param\|^2, \qquad \phi(\param)
\eqdef \frac{1}{ \rho } \param, \qquad \barS(\sample, \tau) \eqdef \tau
- \rho \, G(\sample, \tau).   
\end{equation}
 $\map$ computes 
$\map(s)  = \Prox_{\rho \, g}(s)$ where $\Prox_{\rho \, g}(s) \eqdef \mathrm{argmin}_{\param \in \rset^d}  \left( \rho g(\param) +
\nicefrac{1}{2}   \|\param - s \|^2 \right)$  is the Moreau-proximity operator of the function $\rho \, g(\cdot)$ at $s$ \citep{Moreau:1962}.  
\end{example}

\begin{example}[Jensen surrogate] \label{ex:EM}
Assume that $\loss(\sample,\param)$ is a function of the form
\begin{equation}
    \loss(\sample,\param) \eqdef - \log \int_\calH p(\sample,h,\param) \mu(\rmd h), 
\end{equation}
where $(\sample,h,\param) \mapsto
p(\sample, h, \param)$ is a measurable positive function, and $\mu$ is a $\sigma$-finite positive measure on the measurable set $\calH$. Assume also that $- \log p(\sample, h, \param)$ can be written in the form $\psi(\param) - 
\pscal{S(\sample,h)}{\phi(\param)}$ up to an additive term independent of $\param$. Then {under integrability conditions}, {\bf {\sf MM}-1} holds with the functions $\psi, \phi$ given by $- \log p$ and with 
\begin{equation} \label{ex:EM:barS}
     \barS(\sample, \tau) \eqdef \int_{\mathcal{H}} S(\sample, h) \mu(\rmd h \vert \sample, \tau), 
\end{equation}
where for any $(\sample, \tau)$, $\mu(\cdot \vert \sample, \tau) $ is a probability distribution on $\calH$ defined by
\[
\mu(\rmd h \vert \sample, \tau) \eqdef  \frac{\exp(\pscal{S(\sample,h)}{\phi(\tau)})}{\int_{\mathcal{H}} \exp(\pscal{S(\sample,u)}{\phi(\tau)})\, \mu(\rmd u) } \, \mu(\rmd h).  
\]

An instance of such a loss function $\ell(\sample, \param)$ is the negative log-likelihood of the observation $\sample$  in a latent variable model when the complete likelihood $\param \mapsto p(\sample,h,\param)$ is from the vector exponential family (see e.g. \citet{brown:1986}). $S(\sample,h)$ is the sufficient statistics, and $\exp(\psi)$ acts as a normalizing constant and is fully determined by $S(\sample,h)$ and $\phi$. The distribution $\mu(\rmd h \vert \sample, \tau)$ is the a posteriori distribution of the latent variable $h$ given the observation $\sample$, when the model parameter is $\tau$.
\end{example}

\begin{example}[Variational surrogate] \label{ex:SurrogateCvx} 
Assume that $\loss(\sample, \param) \eqdef \mathrm{min}_{h \in \calH} \tildeloss(\sample, h, \param)$ with 
\[
\qquad \tildeloss(\sample, h, \param)
\eqdef \psi(\param) - \pscal{S(\sample,h)}{\phi(\param)} + \xi(\sample,h);
\]
and the optimization problem $(\sample, \param) \mapsto \mapM(\sample,
\param) \eqdef \mathrm{argmin}_{h \in \calH} \tildeloss(\sample,
\param, h)$ possesses a unique solution. Then under integrability conditions, {\bf {\sf MM}-1} is
satisfied with $\barS(\sample, \tau) \eqdef S(\sample, \mapM(\sample,
\tau))$.\end{example}

\old{In DL, an example of loss function is $\tilde \ell(\sample, \param, h) \eqdef \| \sample - \param h\|^2$ where $\sample \in \rset^p$, $h \in \rset^K$ and $\param \in \rset^{p \times K}$ is the dictionary; in this case, $S(\sample, h)$ collects the two matrices $h h^\top$ and $\sample h^\top$, and $\phi(\param)$ collects $-\param \param^\top$ and $2 \param$.  The map $\mapM(\sample, \param)$ computes the best code $h$ associated to the example $\sample$ when the current dictionary is $\param$.}     

\subsection{MM as a root-finding algorithm}

\bigskip 

The first consequence of {\bf {\sf MM1\&2}} is that for any $\tau \in \paramset$, there exists a
majorizing function of the objective function which equals
$(f+g)$ at $\tau$:
\[
(f+g)(\cdot) \leq f(\tau) + \psi(\cdot)- \psi(\tau) -
\pscal{\PE_\pi\left[ \barS(\sample, \tau) \right]}{ \phi(\cdot) -
  \phi(\tau)} + g(\cdot),~ \quad \text{on} \, \paramset.
\]
Based on this property, the MM algorithm defines a $\paramset$-valued sequence by repeating two steps: given the current iterate $\param_t \in \paramset$, \textit{(i)} identify a surrogate function by computing the {\it mirror} parameter $\PE_\pi\left[ \barS(\sample, \param_t) \right]$; \textit{(ii)} define the next iterate as the minimizer of the surrogate function
\begin{equation}\label{eq:MM:Tspace}
\param_{t+1}  = \map \left( \PE_\pi\left[ \barS(\sample, \param_t) \right]\right). 
\end{equation}
This sequence satisfies $(f+g)(\param_{t+1}) \leq (f+g)(\param_t)$ \cite[Chapter 12]{lange:2013}.

Equivalently, the MM algorithm defines a $\calS$-valued sequence  which characterizes a sequence of surrogate functions  in the functional space spanned by $g+\psi$ and $\phi_1, \cdots, \phi_q$: 
given $s_t \in \calS$ of the form $s_t = \PE_\pi\left[ \barS(\sample, \tau)\right]$ for some $\tau \in \paramset$, \textit{(i)} identify the {\it mirror parameter} $\map(s_t)$  defined as the minimizer  of the surrogate function identified by $s_t$; \textit{(ii)} define  the next surrogate function by computing its parameter
\begin{equation}\label{eq:MM:Sspace}
s_{t+1} = \PE_\pi \left[ \barS(\sample, \map(s_t))\right].
\end{equation}
This sequence satisfies $(f+g)(\map(s_{t+1})) \leq (f+g)(\map(s_{t}))$.

The limiting points of the two MM methods are equivalent. Define the {\it mean field function}  \begin{equation}
  \label{eq:def:meanfield}
\mf: \calS \to \rset^q \qquad  \mf(s) \eqdef \PE_\pi \left[
  \barS(\sample,\map(s)) \right] - s;
\end{equation}
note that a root of $\mf$ is a fixed point of the iterative scheme \eqref{eq:MM:Sspace}. \Cref{prop:equiv:theta:s}  states that  a root of $\mf$ can be used along the mapping $\map$ to find a fixed point of the iterative scheme \eqref{eq:MM:Tspace}. Conversely, a fixed point of the iterative scheme \eqref{eq:MM:Tspace} can be used along the mapping $\PE_\pi \left[ \barS(\sample,\cdot)\right]$ to find a root of $\mf$. This is the main message of \Cref{prop:equiv:theta:s}, which is established
under the following assumptions.
\begin{assumption} \label{hyp:surrogate:convexC1} For any $s \in \calS$,  $\param \mapsto \psi(\param) - \pscal{s}{\phi(\param)}$
      is convex. In addition, $\inter \paramset \neq \emptyset$ and $\psi$ and $\phi$ are $\mathcal C^1$ on  $\inter \paramset$.          \end{assumption}
\begin{assumption} \label{hyp:objective:C1} The function $f$ is $\mathcal C^1$  on $\inter \paramset$. 
\end{assumption}
\begin{proposition}  \label{prop:equiv:theta:s}
    Assume  {\bf {\sf MM-1}}, {\bf {\sf MM-2}},  A\ref{hyp:surrogate:convexC1} and A\ref{hyp:objective:C1}. For any $\tau \in \inter \paramset$,
\begin{equation} \label{eq:prop:equiv:theta:s:state2}
0 \in \nabla f(\tau) + \partial g(\tau) \Longleftrightarrow \map\left( \PE_\pi\left[ \barS(\sample,\tau) \right] \right) - \tau = 0.
\end{equation}
It holds that for any $s \in \calS$ such that $\map(s) \in \inter \paramset$, 
\begin{equation} \label{eq:prop:equiv:theta:s:state1} - J_\phi(\map(s)) \,  \mf(s) \in \nabla f(\map(s))+ \partial g(\map(s)).
\end{equation}
In addition,  if $s_\star \in \calS$ satisfies $\mf(s_\star) =0$ and $\map(s_\star) \in \inter \paramset$, then $\param_\star \eqdef \map(s_\star)$ satisfies $\map\left( \PE_\pi\left[ \barS(\sample,\param_\star) \right] \right) - \param_\star =0$. Conversely, if $\map\left( \PE_\pi\left[ \barS(\sample,\param_\star) \right] \right) - \param_\star = 0$ and $\param_\star \in \inter \paramset$, then $s_\star \eqdef \PE_\pi \left[\barS(\sample, \param_\star) \right]$ satisfies $\mf (s_\star) = 0$.
   
\end{proposition}
We defer the proof of \Cref{prop:equiv:theta:s} to \Cref{sec:proof:prop:equiv:theta:s}. A similar result was proved in 
\citet[Lemma 1]{nguyen:etal:2022}.

\subsection{A Stochastic Surrogate MM algorithm}
In the large-scale learning setting considered here, the expectation $\PE_\pi \left[ \barS(\sample, \tau) \right]$ is not explicit but stochastic oracles are available, whatever $\tau \in \paramset$; on the other hand, the minimizer $\map(s)$ still has a closed-form expression whatever $s \in \calS$. Thus, we propose a stochastic version of MM which replaces the deterministic fixed point algorithm by a Stochastic Approximation scheme (SA, see e.g. ~\cite{robbins_stochastic_1951}, \cite{benveniste:etal:1990,borkar:2008}) designed to solve the root of a vector field  {$\upsilon: \rset^u \to \rset^u$; it repeats $x_{t+1} = x_t +\pas_{t+1} U_{t+1}$ where $\{\pas_t, t \geq 1 \}$ is a sequence of positive deterministic step sizes and $U_{t+1}$ is a random oracle of $\upsilon(x_t)$.  Two properties of the oracles are crucial in the convergence analysis: its conditional bias given the past of the algorithm and its conditional variance (see e.g. \cite{andrieu:moulines:priouret:2005,dieuleveut:etal:2023:TSP}).} In light of {\sf MM-1} and {\sf MM-2}, two strategies are possible for tackling \eqref{eq:opt}: define a $\paramset$-valued sequence targeting the roots of $\param \mapsto  \chi(\param) \eqdef \map(\PE_\pi\left[ \barS(\sample, \param) \right]) - \param$;  or define a $\calS$-valued sequence targeting the roots of $\mf$ (see \eqref{eq:def:meanfield}). The main advantage of the second approach is to conserve the unbiasedness property of the oracle: when $\Smem$ is an unbiased oracle of $\PE_\pi\left[ \barS(\sample, \tau) \right]$, then $\map(\Smem)-\tau$ is a biased oracle of $\chi(\tau)$ - except in the specific case when $\map$ is linear; while when $\Smem$ is an unbiased oracle of $\PE_\pi\left[ \barS(\sample, \map(s)) \right]$, then $\Smem-s$ is also an unbiased oracle of $\mf(s)$. Biased oracles may  ruin  the convergence of the algorithm and have to be controlled with convenient design of the oracles, usually inducing a higher computational cost (see e.g.\cite[Section IV.B.3.c]{dieuleveut:etal:2023:TSP}).

This yields the {\it Stochastic Approximation Stochastic Surrogate {\sf MM}} ({\SSMM}) algorithm in \Cref{algo:SS-MM}.

\begin{algorithm}[H]\DontPrintSemicolon
  \caption{Stochastic Approximation Stochastic Surrogate {\sf MM} (\SSMM)\label{algo:SS-MM}}
  \KwData{ $\kmax \in \nset^\star$; a $\ocint{0,1}$-valued sequence
    $\{\pas_{t}, t \in [\kmax]\}$;  initial value $\hatS_0 \in
    \calS$} \KwResult{  \SSMM\ sequence $\{\hatS_{t}, t \in
    \{0, \ldots, \kmax \} \}$ and its mirror sequence $\{\map(\hatS_t), t \in  \{0, \ldots, \kmax \} $\}}
  \For{$t=0, \ldots, \kmax-1$}{ 
  Sample $\Smem_{t+1}$ in $\calS$, a random oracle for $\PE_\pi [\barS(\sample, \map(\hatS_t))]$\; 
      Set $\hatS_{t+1} =  \hatS_t +
    \pas_{t+1} (\Smem_{t+1} - \hatS_t )
    $ \label{line:SSMM:updateS} \; }
\end{algorithm}
At the $t$-th iteration, Algorithm \ref{algo:SS-MM} utilizes a stochastic oracle of $\mf(\hatS_t)$, the mean field function evaluated at the current iterate,  defined by $\Smem_{t+1} - \hatS_t$, where  $\Smem_{t+1}$  approximates   $\PE_\pi [ \barS(\sample, \map(\hatS_t)) ]$.  In {\it online
  learning}, given a new example $\sample_{t+1}$, one can
set $\Smem_{t+1} = \barS(\sample_{t+1}, \map(\hatS_t))$. In {\it large
  batch learning} with $N$ examples, a possible oracle is $\Smem_{t+1}
= \lbatch^{-1} \sum_{j \in \batch_{t+1}} \barS(\sample_j, \map(\hatS_t))$
where $\batch_{t+1}$ is a mini-batch of size $\lbatch$ sampled at
random in $[N]$. 

Since $\calS$ is convex (see {\sf \bf MM}-1)  and $\pas_{t+1} \in \ocint{0,1}$, we have $\hatS_t +\pas_{t+1} (\Smem_{t+1} - \hatS_t) \in \calS$ and $\map(\hatS_{t+1})$ is well-defined.  Note that in the special case of constant step sizes $\pas_{t+1} = \pas$ for all $t$, {we have $\hatS_{t+1} = (1-\pas)^{t+1} \hatS_0 + \pas \sum_{j=0}^t (1-\pas)^j \hatS_{t+1-j}$:} the iterate $\hatS_{t+1}$ forgets the past oracles at a geometric rate. When $\pas_{t+1} = 1/(t+1)$, $\hatS_{t+1}$ is the empirical expectation of the oracles: $\hatS_{t+1} = (t+1)^{-1} \sum_{j=1}^{t+1}\Smem_j$.  

\subsection{\SSMM\ on few examples and Related work} \label{subsec:ssmmeg}

We now show that {\SSMM} is related to multiple stochastic algorithms in the literature.

\paragraph{Stochastic Gradient Descent and Stochastic Proximal Gradient Descent.}
The iterative algorithm \eqref{eq:MM:Tspace} applied to \Cref{ex:gdt} defines the deterministic sequence $\param_{t+1} = \Prox_{\rho \, g}\left( \param_t - \rho  \nabla f(\param_t) \right)$ where $\Prox_{\rho \, g}$ is the proximity operator of the proper lower semicontinuous convex function $\rho \, g$.  When $g \equiv 0$, it is the Gradient Descent algorithm with constant step size; and otherwise, it is the Proximal Gradient algorithm \citep{beck:teboulle:2010,combettes:pesquet:2011,parikh:boyd:2013}.  Stochastic versions of the deterministic algorithm \eqref{eq:MM:Tspace} addressing the case when $\nabla f(\param)$ is an expectation with no closed-form expression, are proposed in the literature (see e.g. \cite{atchade:fort:moulines:2017}). 

The iterative algorithm \eqref{eq:MM:Sspace} applied to \Cref{ex:gdt} defines the deterministic sequence  
\[
s_{t+1} =  s_{t+\nicefrac{1}{2}} - \rho \nabla f(s_{t+\nicefrac{1}{2}}), \qquad \text{with} \quad 
s_{t+\nicefrac{1}{2}} \eqdef  \Prox_{\rho g}(s_t). 
\]  The update mechanism in line \ref{line:SSMM:updateS} of 
\Cref{algo:SS-MM}  gets into $\hatS_{t+1} = \hatS_{t} - \pas_{t+1} ( \hatS_{t+\nicefrac{1}{2}} - \rho  G_{t+1}  - \hatS_{t})$ 
where $\hatS_{t+\nicefrac{1}{2}} \eqdef \Prox_{\rho \,g}(\hatS_t)$ and $G_{t+1}$ is a random oracle for the intractable quantity $\nabla f(\hatS_{t+\nicefrac{1}{2}})$; $G_{t+1}$ can cover many stochastic approximations of
the exact gradient  possibly defined to include
variance reduction schemes.

Let us discuss what the mirror sequence $\param_t \eqdef \map(\hatS_t) = \Prox_{\rho \, g}(\hatS_t)$ is. By iterating line \ref{line:SSMM:updateS} in 
\Cref{algo:SS-MM}, we have 
\begin{equation}\label{eq:algo3:iteratefrom0}
\hatS_{t+1} = \Gamma_{1:t+1} \hatS_0 + \sum_{j=1}^{t}   \pas_{j} \Gamma_{j+1:t+1}  \,  \Smem_{j} + \pas_{t+1} \Smem_{t+1}, \qquad \Gamma_{r:r'} \eqdef \prod_{\ell=r}^{r'} (1-\pas_\ell). 
\end{equation}  From \eqref{eq:algo3:iteratefrom0}, we have $\param_{t+1} = \Prox_{\rho \, g}(\tilde \param_t - \rho \tilde G_{t+1})$ where 
$\tilde \param_t \eqdef \sum_{j=1}^{t}   \pas_{j} \Gamma_{j+1:t+1}  \,  \param_{j-1} + \pas_{t+1} \param_t$ is a weighted average of the previous mirror parameters, and $\tilde G_{t+1} \eqdef \Gamma_{1:t+1} \hatS_0 + \sum_{j=1}^{t}   \pas_{j} \Gamma_{j+1:t+1}  \,  G_j
+ \pas_{t+1} G_{t+1}$ is a weighted average of the past oracles $G_1, \ldots, G_t$ and the current one $G_{t+1}$. As a conclusion, the mirror sequence $\{\map(\hatS_t), t \geq 0\}$ produced by \Cref{algo:SS-MM} is among the stochastic Proximal-Gradient algorithms,  here the gradient step takes benefit of the full history of the algorithm.

\paragraph{Stochastic Approximation Expectation Maximization.}
The iterative procedure \eqref{eq:MM:Tspace} applied to \Cref{ex:EM}  defines the sequence 
\begin{equation}\label{eq:EM:Tspace}
\param_{t+1}  = \mathrm{argmax_\param} \pscal{\PE_\pi\left[ \int S(\sample, h) \mu(\rmd h \vert \sample, \param_t)\right]}{\phi(\param)} - g(\param) -  \psi(\param).
\end{equation}
When $g=0$, this algorithm is known 
as the Expectation Maximization algorithm (EM) 
\citep{dempster:etal:1977} 
for the case of complete log-likelihood from a vector exponential family; 
the case $g \neq 0$ is the Maximum A Posteriori EM (also known as Bayesian EM,  \cite[Section 6.5]{maclachlan:2008}).  The computation of the double expectation is the E-step, while the optimization step is the M-step.  

EM is a popular algorithm for maximizing the likelihood of observations  in a latent variable model (see e.g. \cite{maclachlan:2008}). 
We defer to \Cref{app:exampleEM} two concrete examples of EM, stressing they are instances of \Cref{ex:EM}. In many situations, the E-step has no closed form expressions. It  happens that the inner expectation w.r.t. the distribution $\mu(\rmd h \vert \sample,\param)$ is not explicit typically when this a posteriori distribution is known up to a normalizing constant. The outer expectation w.r.t. $\pi$ also happens to be intractable: for example in online learning, and in large batch learning when $\pi$ is the empirical distribution $N^{-1} \sum_{i=1}^N \delta_{\sample_i}$ with Dirac mass at the $N$  independent examples $\sample_i$.  Stochastic versions of \eqref{eq:EM:Tspace} were proposed: they consist in replacing the intractable expectations with a stochastic oracle (see e.g. the Stochastic EM \cite{celeux:diebolt:1992}; and the Monte Carlo EM, \cite{Wei:tanner:1990, Fort:moulines:2003}). 

If the deterministic sequence \eqref{eq:MM:Tspace} corresponds to the usual description of the Bayesian-EM algorithm which outputs the successive M-steps, the deterministic sequence given by \eqref{eq:MM:Sspace} is the output of the successive E-steps.

\Cref{algo:SS-MM} applied to \Cref{ex:EM} 
covers many popular stochastic EM algorithms: let us cite the Stochastic Approximation EM proposed by \cite{Delyon:etal:1999} when the inner expectation is not explicit;  in the case $\pas_t =1$ for any $t$, the Incremental EM by \cite{Neal:hinton:1998} when the outer expectation is not explicit; in the case $\pas_t \in \ooint{0,1}$, the Online EM proposed by \cite{sato:ishii:2000} and generalized by \cite{cappe:moulines:2009} when the outer expectation is not explicit;  and many possible oracles proposed e.g.  by \cite{chen:etal:2018,gach:fort:moulines:2020,fort:moulines:wai:2021,fort:moulines:2023}, 
when possibly both inner and outer expectations are intractable and  a variance reduction scheme is included.

\paragraph{Dictionary Learning and Matrix Factorization by Variational Surrogate.} 
Applications of the variational surrogate setting are, among others, the Huber loss regression \citep[e.g.,][Section 2.3.4]{mairal:2015}, and the Dictionary Learning and Matrix Factorization (MF) problems. In the dictionary learning problem,  we want to learn  a dictionary $\param \in \rset^{p \times K}$ by solving 
 \begin{equation} \label{eq:dictlearning:problem}
\mathrm{Argmin}_{\param \in \rset^{p \times K}} \left(  \PE_{\pi} \left[
\min_{h \in \rset^K} \left(  \L(\sample, \theta h)+ \chi(h) \right)  \right] +  g(\param) \right),
 \end{equation}
{where  $\L: \rset^p \times \rset^p \to \rset$ is a loss function, and   $\chi: \rset^K \to \ocint{-\infty,  +\infty}$ and  $g: \rset^{p \times K} \to \ocint{-\infty, +\infty}$ are penalty  and/or constraint functions. In MF, the vectors of coefficients $h$ are also learnt. Again, $\pi$ is the distribution of the i.i.d. stream of data $\{\sample_1, \sample_2, \cdots \}$ in online learning, and $\pi$ is the empirical distribution $N^{-1} \sum_{i=1}^N \delta_{\sample_i}$ in batch learning with $N$ i.i.d. observations $\{\sample_1, \cdots, \sample_N\}$.  

When $\L(z,z')$ is the quadratic loss $\|z-z'\|^2$, the problem \eqref{eq:dictlearning:problem}  is of the form \eqref{eq:opt} and is an example of variational  surrogate (see \Cref{ex:SurrogateCvx})  with  
\[
\tilde \ell(\sample, h, \param) \eqdef \sample^\top \sample + \chi(h) - 2
\sample^\top \param h + h^\top \param^\top \param h,
\]
which is of the form $  \psi(\param)- \pscal{S(\sample,h)}{\phi(\param)}
+  \xi(\sample,h)$ by setting
\begin{equation}\label{eq:dict_learn_mm1}
     \psi(\param) =0,
    \qquad S(\sample,h) \eqdef \left[  \begin{matrix} h h^\top \\
        \sample h^\top
      \end{matrix} \right], \qquad \phi(\param) \eqdef \left[  \begin{matrix} -\param^\top  \param \\
      2   \param
      \end{matrix} \right], \qquad \xi(\sample,h) \eqdef \sample^\top \sample + \chi(h).    
\end{equation}
$S(\sample,h)$ and $\phi(\param)$ collect two matrices: a $K \times K$ one and a $p \times K$ one respectively. Here, \begin{equation} \label{eq:MF:tool1}
      \mapM(\sample,\tau) \eqdef \argmin_{h \in \rset^K}  \left( \chi(h) + \mathrm{Trace}(\tau^\top \tau h h^\top) - 2 \pscal{\tau^\top \sample}{h} \right),
\end{equation}
assuming that the minimizer exists and is unique.  In \textbf{\sf MM-1}, set $\calS \eqdef \mathcal{M}_K^+ \times \rset^{p \times K}$.  Finally, for $s =(s^{(1)}, s^{(2)}) \in  \calS$, 
\begin{equation} \label{eq:MF:tool2}
\map(s) \eqdef \mathrm{argmin}_{\param \in \rset^{p \times K}} \left( g(\param) +
\mathrm{Trace}(\param^\top \param  s^{(1)}) - 2
\mathrm{Trace}(\param^\top s^{(2)}) \right). 
\end{equation}
Let us describe a definition of oracles in  \SSMM.   At iteration $\#(t+1)$, choose a minibatch of observations $\{\sample_i, i \in \mathcal{B}_{t+1} \}$  of size $\mathsf{b}$,  where $\mathcal{B}_{t+1}$ is sampled at random in $\{1, \cdots, N\}$ in batch learning, or collects new observations in online learning; then, define the oracle $\Smem_{t+1} = (\Smem^{(1)}_{t+1}, \Smem^{(2)}_{t+1})$  in $\calS$ by   
\begin{equation}  \label{eq:MF:tool3}
\Smem^{(1)}_{t+1} \eqdef  \mathsf{b}^{-1} \sum_{i \in \mathcal{B}_{t+1}}  H_{i,t}
H_{i,t}^\top, \quad \Smem^{(2)}_{t+1} \eqdef \mathsf{b}^{-1} \sum_{i \in \mathcal{B}_{t+1}}\sample_{i} H_{i,t}^\top, \quad \text{with} \quad H_{i,t} \eqdef
\mapM\left(\sample_{i}, \param_t \right).
\end{equation}

In a specific setting of \SSMM\ for online learning, the mirror sequence $\{\param_t, t \geq 0\}$ defined by $\param_t \eqdef \map(\hatS_t)$ is the  {\it online dictionary learning}  proposed by \citet[Algorithm
1]{mairal:etal:2010}. This setting is the case when $\chi(h) \eqdef \lambda \| h \|_1$ for some $\lambda >0$,  $g(\param)$ is the $\{0, +\infty\}$-valued indicator function of the set $\paramset \eqdef \{\param= (\param_{\cdot 1}, \cdots, \param_{\cdot
  K}) \in \rset^{p \times K}:  \| \param_{\cdot k} \|^2 < 1 \
\text{for all} \ k \in \ccint{K} \}$;
$\mathsf{b}=1$ and $\pas_{t+1}  \eqdef \nicefrac{1}{(t+2)}$.  Note that in this setting,  $\mapM(\sample, \param)  \eqdef \argmin_{h} \left( \|\sample
  - \param h\|^2 + \lambda \|h\|_1 \right)$  corresponds to a $L^1$-regularized linear least-squares problem,  an oracle for $\mapM(\sample,\param)$ is obtained by applying LARS-Lasso \citep{Efron_2004}, or Proximal-Gradient Descent (see e.g. \Cref{sec:dict_learn}) among examples.

\paragraph{Structured tangent majorizing functions of $\ell(Z, \cdot)$}\label{main:StoSurMM-increm_learn} 
Assume that there exists a convex subset $\calS$ of $\rset^q$ and
    measurable functions $\barS: \rset^p \times \paramset \to \calS$,
    $\psi: \paramset \to \rset$ and $\phi: \paramset \to \rset^q$ such
    that for any $\tau \in \paramset$, $\PE_\pi
  \left[ \| \barS(\sample, \tau) \| \right]  < \infty $ and 
\begin{equation} \label{eq:defMM:surrogate:onell}
\ell(z,\cdot) \leq \ell(z,\tau) + \psi(\cdot) - \psi(\tau) - \pscal{
  \barS(z, \tau)}{\phi(\cdot) - \phi(\tau)} \qquad \text{on $\paramset$};
\end{equation}
see e.g. \Cref{ex:EM} (Eq. \eqref{eq:toolex2} in \Cref{sec:detailedex2})  and \Cref{ex:SurrogateCvx} (Eq. \eqref{eq:toolex3bis} in \Cref{sec:detailedex3}).
Then the assumption \textbf{\sf MM-1} is satisfied. Assume also that   \textbf{\sf MM-2} holds and consider the mirror sequence $\param_t \eqdef \map(\hatS_t)$ of the sequence $\{\hatS_t, t \geq 0\}$ given by \Cref{algo:SS-MM}.  We have $\param_{t+1}  = \mathrm{argmin_{\param \in \rset^d}} \, \bar U_{t+1}(\param)$ where
\begin{align*}
\bar U_{t+1}(\cdot) & \eqdef g(\cdot) +  \psi(\cdot) - \pscal{\hatS_{t+1}}{\phi(\cdot)} \\
& = (1-\pas_{t+1})  \left( g(\cdot) + \psi(\cdot) - \pscal{\hatS_{t}}{\phi(\cdot)} \right) + \pas_{t+1}  \left( g(\cdot) + \psi(\cdot) - \pscal{\Smem_{t+1}}{\phi(\cdot)} \right) \\
& = (1-\pas_{t+1}) \bar U_t(\cdot) + \pas_{t+1} U_{t+1}(\cdot)
\end{align*}
where $U_{t+1}(\cdot) \eqdef g(\cdot) + \psi(\cdot) - \pscal{\Smem_{t+1}}{\phi(\cdot)}$. This discussion evidences that, in the setting \eqref{eq:defMM:surrogate:onell}, \SSMM\ defines iteratively a sequence of surrogate functions $\{\bar U_{t}, t \geq 0 \}$. 

Let us consider the case of online learning, when $\pi$ in \eqref{eq:opt} is an unknown distribution but  a stream of i.i.d. examples $\sample_1, \sample_2, \cdots$ is available, $\sample_1 \sim \pi$. Define the oracle $\Smem_{t+1}$ by $\barS(\sample_{t+1}, \param_t)$.   Upon noting that $\argmin_{\param \in \rset^d} \, \bar U_{t+1}(\param)$ is invariant if a constant is added to $\bar U_{t+1}$, we can assume without loss of generality that the function $U_{t+1}$ is tangent to the function $\param \mapsto  \ell(\sample_{t+1},\param) + g(\param)$ at the point $\param_t$: the  surrogate function $\bar U_{t+1}$ is updated at each iteration as a convex
combination of the previous surrogate function $\bar U_{t}$ and a novel majorizing and tangent  function of the penalized loss function $ \ell(\sample_{t+1}, \cdot) + g(\cdot)$. 
This establishes that the mirror sequence $\{\param_t,t \geq 0\}$ of the output of \Cref{algo:SS-MM} defines the sequence  $\{\param_t, t \geq 0\}$  given in  \citet[Algorithm 1]{mairal:2013} for the minimization of $\param \mapsto \PE_\pi\left[ \ell(\sample,\param)\right] + g(\param)$.

%% file: 2_fedmm.tex
\section{Federated MM}\label{sec:FLMM}
In the following, we extend our scope to the Federated Learning (FL) framework in which $n$ different \textit{agents} (or \textit{workers}) hold parts of the data. For $i\in [n]$, we denote by $\pi_i$ the probability distribution of the data of agent $\# i$; we focus on the heterogeneous regime, in which $\pi_i$ (strongly) depends on $i$. In the finite sample regime,  let $N_i$ be  the number of examples on agent $\# i$ and  $N \eqdef \sum_{i\in [n]} N_i$  be the total number of examples in the system; $\mu_i \eqdef N_i/N$ is the fraction of examples on agent  $\# i$.
In the online case (where each worker can access an infinite number of points),  we set by default $\mu_i\eqdef n^{-1}$, the uniform weight over the $n$ agents.  The objective function can be written
\begin{equation}\label{eq:opt-fed}\textstyle
\argmin_{\param \in \rset^d} \left(  \sum_{i\in [n]} \mu_i f_i(\param) + g(\param) \right),
\qquad \text{with} \quad f_i(\param) \eqdef \PE_{\sample \sim \pi_i} \left[ \ell(\sample,
  \param) \right].
  \end{equation}
This federated setting, applied to \Cref{ex:gdt,ex:SurrogateCvx,ex:EM}, corresponds to solving minimization problems by a MM approach based on respectively a quadratic surrogate, a Jensen surrogate or a variational surrogate, by using the union of examples over all workers while keeping locally each batch of observations.  

\subsection{Aggregate in the $\calS$-space or in the $\paramset$-space ?}
\label{sec:SspaceORThetaspace}

Let's consider an ideal setting where  all clients are always online, there are no communication constraints, and exact expectation w.r.t. $\pi_i$ can be computed and the assumptions {\bf {\sf MM1\&2}} hold as follows: {\bf {\sf MM1}} is satisfied for each function $\param \mapsto \PE_{\pi_i}\left[ \ell(\sample, \param) \right]$, and $(\phi, \psi, S, \calS)$ do not depend on $i$; {\bf {\sf MM2}}  is satisfied and $\map$ does not depend on $i$. 

Following the Majorize-Minimization approach, there are two ways to solve \cref{eq:opt-fed} in a federated manner depending on how the aggregation step is performed. The first one relies on an aggregation in the parameter ($\paramset$) space: (1) broadcast the current $\paramset$-valued parameter $\param_t$ to all clients, (2) compute $\param_{t+1, i} \eqdef  \map\left(\PE_{\pi_i}  \left[\barS(\sample, \param_t)\right]\right)$ at each client, (3) aggregate by setting 

\begin{equation}\label{eq:agg_in_Theta}\textstyle
  \param_{t+1} \eqdef \sum_{i\in [n]} \mu_i  \param_{t+1,i} = \sum_{i\in [n]} \mu_i  \map \left( \PE_{\pi_i}  \left[\barS(\sample, \param_t)\right]\right), \qquad \text{for }t\geq 0. 
  \end{equation}
Here, the optimization steps are solved by each local agent but could also be run by the central server prior to aggregation. Note that special cases of \eqref{eq:agg_in_Theta} have appeared in \citep{marfoq2021federated, wang:chang:2022}.
However, the scheme \eqref{eq:agg_in_Theta} no longer reduces to the centralized case and the sequence $\{\theta_t, t \geq 0 \}$ may not generally converge. See \Cref{rem:aggr_in_S}.

The second one relies on aggregation in the surrogate ($\calS$) space: (1) from the current $\calS$-valued parameter $s_t$, compute the mirror $ \theta_{t+1} = \map(s_t)$ and broadcast it to all clients, (2) compute $s_{t+1, i} \eqdef \PE_{\pi_i}  \left[\barS(\sample, \map(s_t))\right]$ at each client, (3) aggregate by setting 
\begin{equation} \label{eq:agg_in_S} \textstyle
s_{t+1} \eqdef \sum_{i \in [n]} \mu_i s_{t+1,i} = \sum_{i \in [n]}  \mu_i  \PE_{\pi_i}  \left[\barS(\sample, \map(s_t))\right], \qquad \text{for }t\geq 0.
\end{equation} 
Note that the optimization step is solved by the central server. As a result, due to the linearity of the surrogate parameterization, this approach is equivalent to centrally solving the problem with the mixture distribution $\pi = \sum_{i \in [n]} \mu_i \pi_i$. 

{
The following remark shows that the natural way to aggregate each worker's participation is in the space $\calS$, and that aggregation should not be performed in the space $\paramset$. 
\begin{remark}\label{rem:aggr_in_S} 

As a toy example, consider the real-valued loss function $\ell(\sample, \param)  \eqdef \sample \param + 1/\param$ and the constraint $g(\param) \eqdef \iota_{\rset_{>0}}(\param)$ (the convex indicator of  $\rset_{>0}$) i.e.~$\param \in \ooint{0,+\infty}$; assume also that $\sample>0$ $\pi_i$-almost surely  for all $i$. Here,  the federated objective corresponds to
\begin{equation*}
    \argmin_{\param \in \rset_{>0}}  \left( \sum_{i\in [n]} \mu_i\expec_{\pi_i} \left[\sample \right] \right) \, \param+\frac{1}{\param}
\end{equation*}
whose solution is  $\theta_\star 
\eqdef \sqrt{{1}/{\sum_{i \in [n]} \mu_i \PE_{\pi_i}[Z]}}$. The assumption {\bf \sf MM-1}  is satisfied for all the functions $\param \mapsto \PE_{\pi_i}[\ell(\sample,\param)]$  by setting $\phi(\param) \eqdef -\param$, $\psi(\param) \eqdef 1/\param$, $\barS(\sample,\param) \eqdef \sample$ and $\calS\eqdef \rset_{>0}$.  The assumption {\bf \sf MM-2}  is satisfied with $\map(s) \eqdef 1/\sqrt{s}$. Aggregation in the $\calS$-space defines a constant sequence $s_t \eqdef \sum_{i \in [n]} \mu_i \PE_{\pi_i} \left[ \sample \right]$ whose mirror sequence is $\param_t^{\mathrm{mir}} \eqdef \map(s_t) = \param_\star$. Aggregation in the $\paramset$-space defines a constant sequence  given by $\param_t \eqdef \sum_{i \in [n]} \mu_i /\sqrt{\PE_{\pi_i}\left[ \sample\right]}$ (see     \eqref{eq:agg_in_Theta}); we have $\param_t \neq \param_\star$. \end{remark}}

\subsection{\QSMM: a Federated Majorize-Minimization algorithm}
\label{sec:fedmm:algo}

Our idea is to develop a Federated Majorize-Minimization (\QSMM) algorithm based on \eqref{eq:agg_in_S}. To develop a practical algorithm, several challenges have to be overcome.

First, it is generally impossible to compute $\PE_{\pi_i}  \left[\barS(\sample,\cdot)\right]$, because the number of observations (on each worker) is large or even infinite; we thus have to rely on stochastic updates, for example as in \Cref{algo:SS-MM}. 

Second, the communication cost between workers is typically an important bottleneck. In the classical federated-optimization framework, two main solutions have been proposed to reduce this communication cost. One can either perform local iterations as in \textit{LocalSGD}~\citep{stich_local_2019,woodworth_is_2020} and \textit{FedAvg}~\citep{mcmahan2016communication}, or reduce the communication cost either by performing  a compression step \citep{alistarh_qsgd_2017,konecny_federated_2016-1} or by sub-sampling workers at each step (referred to as Partial Participation - PP). In both situations, the data heterogeneity hinders the convergence and new algorithms have been proposed to ensure the convergence despite heterogeneity~(for local iterations, see e.g.~\citep{li_convergence_2019,karimireddy_scaffold_2019}, for compressed FL, see e.g.~\citep{wangni2018gradient,karimireddy_error_2019,mishchenko_distributed_2019,philippenko_preserved_2021,khaled_gradient_2020,richtarik2021ef21}).

In the following, we propose a unified algorithm, \QSMM\ that has the following features:  (i) the aggregation step is performed in the space $\mathcal S$; (ii) the algorithm uses stochastic oracles on $\PE_{\pi_i}  \left[\barS(\sample, \map(\hatS_t))\right]$, e.g., computed on a batch of examples at each iteration;  (iii) Partial Participation (PP) is incorporated, and there is compression of the information computed by the active workers before communication to the central server.
In the PP setting, data heterogeneity across clients poses a significant challenge. The global fixed point $s$ (where $\PE_{\pi} [\bar S(Z, \map(s))] = s$ i.e. $\mf(s)=0$) is an average over client-specific expectations, meaning that for any individual client $i$, it does not generally hold that $\PE_{\pi_i} [\bar S(Z, \map(s))] = s$. If only a subset of clients participates in an update, a naive aggregation would be biased towards the participating clients' local fixed points, potentially destabilizing the algorithm. Similarly, lossy compression of client updates, especially if the compression error correlates with client drift from the global objective, can further exacerbate the impact of this underlying heterogeneity by distorting the aggregated information. To counteract this, (iv) control variates are introduced. These client-specific variables track and correct for the deviation of local updates from a global perspective, effectively acting as a form of client-wise variance reduction. This mechanism helps to mitigate the adverse effects of heterogeneity when employing PP or compression, as explored in works like~\citep{mishchenko_distributed_2019,dieuleveut:etal:FedEM:2021}.
Our approach takes its roots in \texttt{DIANA}~\citep{mishchenko_distributed_2019}, an algorithm for heterogeneity mitigation for federated optimization, and in  \texttt{FedEM}~\citep{dieuleveut:etal:FedEM:2021}, a specific instance of the general MM framework addressed in this paper.

\begin{algorithm}[t]\DontPrintSemicolon
  \caption{Federated Majorize-Minimization   \QSMM 
    } \label{alg:fedss} \KwData{$\kmax \in \nset_{>0}$;   $\alpha>0$; $\proba \in \ocint{0,1}$; 
    a $\ocint{0,1}$-valued sequence
     $\{\pas_t, t \in [\kmax]\}$;  a distribution  $\lawPP$ on the subsets of $\{1, \cdots, n\}$;
  initial values $\hatS_0 \in \calS$ and  $V_{0,i} \in \rset^q$ for all $i \in [n]$; a probability distribution $\{\mu_i, i \in [n]\}$   . \label{algo:FedSS-MM}} \KwResult{ Sequence: $\{\hatS_{t}, t \in [\kmax]\}$} Set $V_0 =
  \sum_{i=1}^n \mu_iV_{0,i}$ \; 
  \For(\tcp*[f]{for all iterations}){$t=0, \ldots, \kmax-1$}{
    Sample $\set_{t+1} \sim \lawPP$ \label{line:participation_sampling}
    \label{line:PPworker} \; 
    Send $\hatS_{t}$
    and $\map (\hatS_{t})$ to the active workers \label{line:mapT_broadcast} \; 
    \For(\tcp*[f]{on active worker $\# i$}){$i \in \mathcal \set_{t+1} $ }{  Sample
      $\Smem_{t+1,i}$,  oracle for $ \PE_{\pi_i}[
        \barS(\sample, \map(\hatS_t))]$  \label{line:SampleSk}\;
    Set $\Delta_{t+1,i} = \Smem_{t+1,i} - \hatS_t -
      V_{t,i}$  \label{line:diffVk}  \; 
     Set $V_{t+1,i} = V_{t,i} + \nicefrac{\alpha}{\proba}
      \, \Q_{t+1,i}(\Delta_{t+1,i})$.   \label{line:upV} \; 
       Send
      $\Q_{t+1,i}(\Delta_{t+1,i})$ to the central
      server\; \label{line:compDelta} } 
      \For(\tcp*[f]{on inactive worker $\# i$}){$i \notin \mathcal
      \set_{t+1} $  }
      {
      Set $V_{t+1,i} =
      V_{t,i}$ \tcp*{no update} } 
      \tcp*{on the central server} Set
    $H_{t+1} = V_t +  \nicefrac{1}{\proba} \sum_{i\in {\cal A}_{t+1}}  \mu_i
    \Q_{t+1,i}(\Delta_{t+1,i})$
 
    \label{line:reconstruct_H_central} \; 
    Compute $\hatS_{t+ \nicefrac{1}{2}} = \hatS_t + \pas_{t+1}H_{t+1}$ \label{line:hatS:temp} \;
    Project it on $\calS$, given $\B_t \succ 0 $: $\hatS_{t+1}  = \argmin_{s\in \surspace}  \left(s-\hatS_{t+ \nicefrac{1}{2}}\right)^\top \B_t \left(s-\hatS_{t+ \nicefrac{1}{2}}\right)$  \label{line:upModel_central}\;
    Set $V_{t+1} = V_t +
    \nicefrac{\alpha}{\proba} \sum_{i\in {\cal A}_{t+1}}  \mu_i
    \Q_{t+1,i}(\Delta_{t+1,i})$ 
     \label{line:upV_central} \; }
\end{algorithm}

\paragraph{Description of \QSMM, \Cref{algo:FedSS-MM}.} At iteration $(t+1)$,    a set of  active agents $\set_{t+1}$  is selected at random from a probability distribution $\lawPP$ on the subsets of $\{1, \cdots, n\}$, such that the mean  number of active workers at each iteration is $n \proba$ ($\proba \in \ocint{0,1}$).  The server broadcasts back $\hatS_{t}$ and its mirror $\map (\hatS_{t})$ to the active workers. It runs the optimization step~$\map$; in practice, the computation of the map $\map$ at line~\ref{line:mapT_broadcast} can be replaced by an approximated optimization step, e.g., a few gradient steps. Nonetheless, our convergence analysis requires the exact computation of $\map$ (see \Cref{theo:main}).  \\

 Then, each active agent $\# i$ samples a local surrogate oracle $\Smem_{t+1,i}$ of $\PE_{\pi_i}[ \barS(\sample, \map(\hatS_t))]$. The agent's goal is to construct and communicate a low-variance, compressed estimate of its contribution to the global update direction, which is related to the mean field $\mf(\hatS_t)$ from \eqref{eq:def:meanfield}. To achieve this, the agent first computes its local update estimate, $\Smem_{t+1,i} - \hatS_t$. However, due to data heterogeneity, this local estimate is a biased estimator of the global mean field $\mf(\hatS_t)$. To correct for this client drift, a control variate $V_{t,i}$ is subtracted. This results in the error-corrected vector $\Delta_{t+1,i} = \Smem_{t+1,i} - \hatS_t - V_{t,i}$ (line~\ref{line:diffVk}). It is this locally-debiased quantity that is then compressed using a (potentially random) operator $\Q_{t+1,i}$ and sent to the server. These local subtractions are critical as they construct a more accurate update vector \textit{before} the lossy compression step. The client's control variate is then updated using the compressed information (line~\ref{line:upV}), with a step size $\alpha>0$. Inactive agents (i.e. agents $\# i$ with $i \notin \set_{t+1}$) do not perform any updates. \\

Finally, the central server aggregates all contributions and  compensates the local subtraction of control variates by adding $V_t$ (line~\ref{line:reconstruct_H_central}): it computes $H_{t+1}$, an oracle of $\PE_\pi\left[\barS(\sample, \map(\hatS_t)) \right] - \hatS_t$ (line~\ref{line:reconstruct_H_central}). Then it performs an SA-type update step on the iterate and obtains  $\hatS_{t+\nicefrac{1}{2}}$ (see line \ref{line:hatS:temp}) and projects back onto the space $\calS$ (line~\ref{line:upModel_central}). Crucially, this projection may not be performed according to the Euclidean geometry but possibly according to the geometry defined through a positive definite matrix $\B_t$,  allowed to depend on the  iteration index $t$ and on $\hatS_t$. In our experiments, we will simply set $\B_t=\Id$, the identity matrix, recovering the Euclidean geometry; nonetheless, our convergence result (see \Cref{theo:main}) requires a finer specification of this matrix.

\section{{MM within FL: Related works in the framework {\bf {\sf MM1\&2}} }}
\label{sec:related_work}

The crucial  difference between \QSMM\ and competing algorithms in the literature is that the aggregation steps  are performed in the $\calS$-space (see \Cref{sec:SspaceORThetaspace}). In most federated algorithms to perform personalization, EM, or matrix factorization, aggregation steps are performed in the initial parameter space $\paramset$.

\paragraph{Federated Composite Optimization.} The general optimization problem \eqref{eq:opt}, structured as $\argmin_{\param \in \rset^d}$ $f(\param) + g(\param)$, is a form of composite optimization.
In instances where $f$ is smooth and $g$ is a proper, lower semi-continuous convex function, our MM framework, as exemplified by the quadratic surrogate in \Cref{ex:gdt}, naturally leads to updates where the $\map(s)$ operator becomes the proximal operator of $g(\param)$.
In the literature on federated optimization, methods like ProxSkip/Scaffnew by \citet{MishchenkoProxSkip22} and RandProx/RandProx-FL by \citet{condat2022randprox} investigate scenarios where regularization functions and their proximal operators are used primarily to enforce the consensus mechanism among clients.
These works propose performing the proximal step (which ensures consensus) periodically, rather than at every communication round, to reduce communication overhead.
While their primary focus is on consensus, this framework can be conceptually extended: the periodic server-side proximal step, which enforces consensus, shares a methodological similarity with our \QSMM, where the $\map$ operator (which itself could be a proximal operator for $g(\cdot)$ in specific MM instances  - see e.g. \Cref{ex:gdt}) is also evaluated centrally at the server, and not as part of the local client activity.
However, the fundamental premise of these methods, even when extended, remains distinct from our MM-based surrogate aggregation, as their core design centers on periodic randomized consensus enforcement to reduce communication overhead rather than the iterative minimization of global majorizing functions derived from client-specific surrogate parameters.

A different perspective on aggregation for composite problems is offered by \citet{yuan2021federated}.
They observe that for problems with sparsity-inducing regularizers, such as $g(\cdot)$ being the $\ell_1$ norm, if clients compute their local model updates (involving a local client-side proximal step for $g$) and these updated model parameters are then averaged in the $\paramset$-space, the resulting global model may lose its intended sparsity.
As a remedy, they propose {\tt FedDualAvg}, where aggregation occurs in what they term the "dual space".
In the context of $\ell_1$-regularization, the dual variables often correspond to the subgradient of the $\ell_1$-norm, which directly relates to the parameters defining a particular type of surrogate or upper bound on the regularizer.
Aggregating these dual variables is thus conceptually akin to aggregating parameters in a surrogate space $\mathcal{S}$, aligning with our paper's methodological argument against naive $\paramset$-space aggregation when a more natural, structure-preserving aggregation space exists.
Our \QSMM\ framework provides a more general and direct approach by explicitly identifying and aggregating in the functional space of linearly parameterized functions.

\paragraph{Federated EM.}
The Expectation-Maximization (EM) algorithm (see \Cref{subsec:ssmmeg} and \Cref{app:exampleEM}) has been adapted by several works for the federated learning setting. \citet{marfoq2021federated} considered a federated setting where the data at each agent comes from a mixture of $L$ underlying distributions. These $L$ distributions are shared among clients, and data heterogeneity arises from each client having a different mixture weight. The primary motivation is to facilitate federated personalization by assigning each client a distinct mixture weight and, consequently, varying their models. The authors demonstrate that several federated personalization approaches can be viewed as specific instances of their formulation, such as \citep{t2020personalize,dinh2021fedu,hanzely2020lower}.
Most interestingly, the {\tt FedEM} algorithm of \citet{marfoq2021federated} is presented as a particular instance of a more general Federated Surrogate optimization framework \citep[Algorithm 3 and Appendix F,][]{marfoq2021federated} . However, in their approach, the parametric structure of the surrogate space (i.e., the E-step sufficient statistics) is leveraged only locally by each client for their surrogate optimization step with a locally (depending only on local client data) executed approximation of the $\map$ map.  The subsequent aggregation of updates occurs in the original parameter space~$\Theta$, rather than in the surrogate space $\mathcal{S}$. \cite{marfoq2021federated} provide convergence guarantees that assume a bound on the heterogeneity between client data distributions \citep[Assumption 7,][]{marfoq2021federated}.

In contrast, the Federated EM algorithm proposed by \citet{dieuleveut:etal:FedEM:2021} directly employs $\mathcal{S}$-space aggregation. Their method introduces and analyzes a federated Majorization-Minimization algorithm specifically for EM, where Assumption \textbf{\sf MM-1} is satisfied with an unconstrained surrogate space $\calS = \rset^q$. Consequently, their algorithm can be seen as a specific instance of our \QSMM\ without the projection step (line \ref{line:upModel_central} in Algorithm \ref{alg:fedss}).

More recently, \citet{tiantowards2024icml} (with {\tt FedGrEM})  tackled a similar formulation to \citet{marfoq2021federated} and extended the federated EM framework to manage adversarial data and outlier tasks. Their "federated gradient EM" involves clients executing local E-steps to compute sufficient statistics, which are then used in local gradient-based M-steps (approximating a local version of the $\map$ operator at client $\# i$, depending on client's data distribution.). The resulting client-specific model parameters (in $\paramset$) are subsequently aggregated at the server, differing from \QSMM's central $\map$ operation post-$\mathcal{S}$-aggregation. \citet{tao2024convergence} studied federated EM for Mixtures of Linear Regressions ({\tt FMLR}), characterizing convergence rates based on client-to-data ratios. While focused on statistical rates for this specific model, their analysis \citep[Propositions 2 and 3]{tao2024convergence} implicitly involves $\mathcal{S}$-space aggregation of sufficient statistics.

\paragraph{Federated Matrix Factorization (FMF).}  In batch learning, FMF consists in solving (see \eqref{eq:dictlearning:problem})
\begin{equation}\label{eq:FMF:problem}
\argmin_{\theta \in \Theta}  \frac{1}{N} \sum_{i=1}^n \sum_{j=1}^{N_i} \argmin_{h_{ij} \in \mathcal{H}_i} \left\{ \| \sample_{ij} - \theta h_{ij} \|^2 + \chi(h_{ij}) \right\} + g(\theta),
\end{equation}
where $\sample_{ij} \in \rset^p$ is the $j$-th observation of the client $\# i$,  $N_i$ is the  number of observations  available at agent $\# i$, $\Theta \subseteq \rset^{p \times K}$, $\mathcal{H}_{i} \subseteq \rset^{K}$, and $\chi,g$ are penalty terms on the coefficients $h$ and on the dictionary $\theta$ respectively; $N \eqdef \sum_{i=1}^n N_i$. 
 The problem of FMF has attracted a lot of attention~\citep{flanagan2020federated,gao2020privacy,hegedHus2019decentralized,singhal2021federated}, an application being federated clustering (clustering agents \textit{or} clustering observations); such clusters may constitute another way to perform personalization~\cite[e.g.,][]{mansour2020three,sattler2020clustered,ghosh2020efficient}. 

As discussed in \Cref{subsec:ssmmeg} (see \eqref{eq:dict_learn_mm1} to \eqref{eq:MF:tool3}), the assumptions \textbf{{\sf MM1-2}} are satisfied and \QSMM\ applies: at iteration $\#t$ given the current model $\theta_t \eqdef \map(\hatS_t)$, the active local agents compute an oracle of the matrices
$N_i^{-1} \sum_{j=1}^{N_i} h_{ij}^{\star,t} (h_{ij}^{\star,t})^\top$ and $N_i^{-1} \sum_{j=1}^{N_i}  \sample_{ij} (h_{ij}^{\star,t})^\top$
where $h^{\star,t}_{ij}$ solves
\begin{equation} \label{eq:FMF:maplocal}
\argmin_{h \in \mathcal{H}_i}    \left( \| \sample_{ij} - \theta_t  h \|^2  + \chi(h) \right).
\end{equation}

\texttt{FedMGS}, introduced by 
\citet{wang:chang:2022} with an application to unsupervised data clustering via $K$-means through specific choices of $\mathcal{H}_i$, $\chi$ and $g$,  is an alternative to \QSMM. In \texttt{FedMGS},  the oracle of the matrices  is obtained by substituting the exact solution $h_{ij}^{\star,t}$ for the output of few iterations of a gradient algorithm designed to solve \eqref{eq:FMF:maplocal}.   \texttt{FedMGS} is based on {\it gradient sharing} principles: the central server aggregates the oracles of the  matrices so that   \texttt{FedMGS }aggregates in  the $\calS$-space similarly to \QSMM. Contrary to \QSMM, \texttt{FedMGS} does not include a control variate technique and a quantization step run by local agents, but the central server carries out a control variate scheme for heterogeneity reduction; the selected number of agents at each iteration is constant; an agent which is not selected still updates locally the vectors $h_{ij}$; the local agent ensures that the quantities sent to  the central server are in $\calS$. For both  \QSMM\ and \texttt{FedMGS}, the  optimization step  \textbf{\sf MM-2} is run by the central server; in \texttt{FedMGS}, $\map$ is approximated by many iterations of a projected gradient algorithm.

\paragraph{Similar Methodologies: Aggregation of Functional Approximations.}
The principle of aggregating functional parameters extends beyond our specific Majorize-Minimization framework to the more general setting of surrogate optimization, which relaxes our MM setting by allowing surrogate functions to not necessarily be global majorizing functions. From our methodological standpoint, the key to designing federated algorithms lies in preserving the structure of the underlying surrogate-based method. An ideal federated surrogate algorithm, under perfect conditions (e.g., full client participation, no communication constraints, exact computation of $\map$ operator), should reduce exactly to its centralized counterpart. Our linearly parameterized MM framework achieves this naturally: aggregating the surrogate parameters in the space $\mathcal{S}$ is equivalent to averaging the surrogate functions themselves, ensuring the central optimization step operates on a true global surrogate. 

This perspective helps contextualize other methods that follow a similar philosophy. For example, second-order (Newton-type) methods in FL, as explored by \citet{islamov2021distributed}, can be viewed through this lens. Clients compute local gradients and Hessians, which are the parameters defining local quadratic approximations of the objective function. These functional parameters are sent to the server and aggregated to form a global quadratic model. Only then is the optimization step, the Newton update, is performed on the server using this aggregated information. This process is directly analogous to our approach of aggregating parameters of local majorizing functions (our surrogates) in $\mathcal{S}$-space before the central minimization step is performed using the $\map$ operator.

%% file: 3_theory.tex
\section{Theoretical Analysis and Open Problems}\label{sec:theory}

We now proceed to provide a convergence analysis along with the required assumptions.

\paragraph{Stochasticity and Compression:} We assume that the oracles $\{\Smem_{k+1,i}, k \geq 0\}$ available at
each local agent $\# i$ satisfy the following assumption. First, the local oracles are independent between
local agents; and second, given an agent $\#
i$, the oracles are unbiased with a variance allowed to be specific to
each agent.

\begin{assumption}\label{hyp:oracle}
For all $t \geq 0$, the oracles $\{\Smem_{t+1,i}, i \in
\ccint{n}\}$ are conditionally independent of the past history
of the algorithm $\F_t$\footnote{see \Cref{sec:proof:filtration} for a rigorous definition of the filtration.}. Moreover, for any $i \in [n]$,
$\CPE{\Smem_{t+1,i}}{\F_{t}}= \PE_{\pi_i}[ \barS(\sample,
  \map(\hatS_t)) ]$ and there exists $\sigma_i^2>0$ such that
for any $k \geq 0$, $\PE[\| \Smem_{t+1,i} - \PE_{\pi_i}[
    \barS(\sample, \map(\hatS_t)) ]\|^2|\F_t] \leq \sigma_i^2$.
\end{assumption}

Next, we make the following assumption on the compression operator 
and on partial participation. In the analysis, we consider a \textit{specific} model on the partial participation, that assumes that the agents are selected independently at random at each step.

\begin{assumption}[$\omega$]\label{assumption:URVB}\label{hyp:URVB} \ For all $s \in \rset^q$, the random variables $\{\Q_{t+1,i}(s), t \geq 0, i \in [n] \}$ are i.i.d. with  the same distribution as $\Q(s)$.
There exists $\omega \geq 0$ such
       that for all $s \in \rset^q$, 
       \begin{equation}\label{eq:Quant}
       \PE\left[  \Q(s)\right]=s, \qquad \PE\left[ \|\Q(s)-s\|^2\right]\le
         \omega \, \|s\|^2.
         \end{equation}
         Each  compression is applied independently of any other randomness in the past of the algorithm. 
\end{assumption}
{The case $\omega=0$ corresponds to the case of no compression.} 

\Cref{assumption:URVB}($\omega$)  encompasses operators that have been widely used in the literature, for example Sparsification \citep{wangni2018gradient} and random (scalar or vector) quantization~\citep{alistarh_qsgd_2017,dai2019hyper,leconte2021texttt}. In particular \eqref{eq:Quant} describes the individual
behavior of each compression operator, specifically the fact that it
is unbiased and has a relatively bounded variance (i.e., the variance
 of $\Q(s)$ scales as $\omega\|s\|^2$). {See e.g \cite[Supplemental Section B]{dieuleveut:etal:FedEM:2021} for the block-$p$-quantization example.}

\begin{assumption}[$\proba$] \label{assumption:PP}  At each step, each worker is selected with probability $\proba$, $\proba \in \ocint{0,1}$, independently of the past and of the other workers. 
\end{assumption}

{Independent compression steps in \Cref{assumption:URVB}($\omega$) and independent selection in \Cref{assumption:PP}($\proba$)}
allow us to tackle the workers' selection as a particular form of unbiased compression. It is shown in \Cref{sec:PP2fullP} that compression and partial participation are equivalent:  partial participation can be modeled as another compression operation on top of $\Q$, which sends the output of $\Q$ with probability $
 \proba$ and nothing with probability $1-\proba$.

\paragraph{Model.}  We state our regularity assumptions on the functions 
$\mf_i(\cdot)$ for all $i \in
\ccint{n}$ on the set  $\calS$, and on 
$\phi(\map(\cdot))$.

\begin{assumption}\label{hyp:lipschitz}
 \textit{(i)} For all $i \in \ccint{n}$, the function $\mf_i(s)
  \eqdef \PE_{\pi_i}\left[ \barS(\sample, \map(s)) \right] - s$ is globally $L_i$-Lipschitz on $\calS$. {\textit{(ii)}  $\sup_{\calS} \| \mf(\cdot) \| < \infty$. }
\end{assumption}

\begin{assumption} \label{hyp:DL2} \textit{(i)}  There exist  positive constants $C_\star$  and $C_{\star \star}$  and for any $s \in \calS$, there exists  a $q \times q$ positive-definite matrix $\B(s)$ such that for any $s' \in \calS$,
\begin{align*}
& \left\| \phi(\map(s')) - \phi(\map(s)) - \B(s) (s'-s) \right\| \leq C_\star \|s'-s\|^2,  \\
& |\pscal{s'-s}{\phi(\map(s')) - \phi(\map(s)) }| \leq C_{\star \star} \|s'-s\|^2.
\end{align*}
\textit{(ii)} There exist $0< v_{\min} \leq v_{\max}$ such that for any $s \in \calS$ and
$u \in \rset^q$, $ v_{\min} \|u\|^2 \leq u^\top \B(s) u \leq v_{\max} \| u\|^2$. 
\end{assumption} 
{We verify \Cref{hyp:DL2} for a few examples, instances of  the general settings described in \Cref{ex:gdt} and  \Cref{ex:EM}.  See the discussions in \Cref{sec:checkDL2examples}.}

\paragraph{The projection.} Finally, we proceed to specify the required assumptions on the projection in \Cref{alg:fedss} line \ref{line:upModel_central}. 

\begin{assumption}\label{assumption:projection} \textit{(i)} $\calS$ is a closed convex subset of $\rset^q$. \textit{(ii)} Choose $\B_t \eqdef \B(\hatS_t)$ in  \QSMM, where $\B$ is given by \Cref{hyp:DL2}.
\end{assumption}
To facilitate our discussions later, we set 
\begin{equation} \label{eq:projector:Pi} 
\Pi(\hat{s}, \B) \eqdef \mathrm{argmin}_{s \in \calS} 
(s-\hat{s})^\top \B (s- \hat{s}), \qquad \hat{s} \in
\rset^q; \qquad \|u\|_\B \eqdef \sqrt{u^\top \B u}.
\end{equation}

\paragraph{Finite time analysis.} We now provide an explicit finite time analysis convergence of \QSMM.

To quantify the convergence of the latter, we concentrate on the control of
\begin{equation*}
    \barerror_{t+1} \eqdef \!\!  \frac{\|\Pi( \hatS_t + \pas_{t+1} \mf(\hatS_t), \B_t) - \hatS_t \|^2_{\B_t}}{ \pas_{t+1}^2}
\end{equation*}
which is equal to $\| \mf(\hatS_t)\|^2_{\B_t}$ when $\Pi(s,\B) = s$ {and is lower bounded by $v_{\min} \| \mf(\hatS_t) \|^2$ under \Cref{hyp:DL2}}; $\barerror_{t+1}$ measures how far the {current iterate $\hatS_t$} is from the set $\{s \in \calS: \mf(s) = 0 \}$. {Two cases are considered in turn: first, the case when  there is no compression ($\omega=0$ in \Cref{assumption:URVB}) and the full participation regime is considered ($\proba =1$ in \Cref{assumption:PP}); second, the general case when $\omega >0$ or $\proba \in \ooint{0,1}$.} The proof of \Cref{theo:main}  is in \Cref{sec:proof:maintheo}.

\begin{theorem}\label{theo:main}
Assume \Cref{hyp:objective:C1,hyp:surrogate:convexC1,hyp:oracle,assumption:URVB,assumption:PP,hyp:lipschitz,hyp:DL2,assumption:projection}. Let $\{\hatS_t, t \geq 0
\}$ be the sequence given by \QSMM.  Set 
\[
\omega_\proba \eqdef \omega + (1+\omega) \frac{1-\proba}{\proba}, \quad \bar C \eqdef   C_{\star \star} + \sup_{\calS}\| \mf(\cdot) \| \, 
\, C_\star, \quad  \sigma^2 \eqdef n \sum_{i=1}^n \mu_i^2 \sigma_i^2,
\quad L^2 \eqdef n \sum_{i=1}^n \mu_i^2 L_i^2. 
\]
\textit{(A)} Case $\omega_\proba =0$.  Run \QSMM\ with    $\alpha \in \ccint{0,1}$ and step sizes $\pas_{t+1}  \in \ooint{0, v_{\min}/ (2 \bar C)}$.  For any $T \in \nset_{>0}$
\begin{align*}
\sum_{t=1}^T   \frac{\pas_{t}}{4} \left(  1  -
\frac{2 \bar C}{  v_{\min}} \pas_{t} \right)  \PE\left[\barerror_{t} \right] &  \leq 
\PE\left[ (f+g)(\map(\hatS_{0})) 
- \min (f+g) \right] +    v_{\max}   \frac{\sigma^2}{n}  \sum_{t=1}^T \pas_t. 
\end{align*}
\noindent \textit{(B)} Case $\omega_\proba >0$.  Run \QSMM\ with  $\alpha \in \ocint{0, 1/(1+\omega_\proba)}$  and step sizes satisfying
\[
\pas_{t+1} \leq \pas_t  \leq \pas_1 \qquad 1  -
  \frac{2 \bar C}{v_{\min}}\pas_{t}   -  64 \frac{\omega_\proba}{\alpha^2} \frac{v_{\max}}{ v_{\min}}  \frac{L^2}{n}
 \pas_{t}^2 > 0, \qquad \pas_1\leq \frac{\alpha}{4 \sqrt{\omega_\proba}} \frac{\sqrt{v_{\min}}}{\sqrt{v_{\max}}} \frac{\sqrt{n}}{L}.
\]
For any $T \in \nset_{>0}$, 
\begin{align}
 \sum_{t=1}^T  & \frac{\pas_{t}}{4} \left(  1  -
  \frac{2 \bar C}{v_{\min}}\pas_{t}   -  64 \frac{\omega_\proba}{\alpha^2} \frac{v_{\max}}{ v_{\min}}  \frac{L^2}{n}
 \pas_{t}^2  \right)  \PE\left[\barerror_{t} \right]  \nonumber \\
 &\leq 
\PE\left[ (f+g)(\map(\hatS_{0}))   - \min (f+g) \right] +  \frac{4 \omega_\proba v_{\max}}{ \alpha} \pas_1  \sum_{i=1}^n \mu_i^2 \PE\left[ \|V_{0,i} - \mf_i(\hatS_0)\|^2\right] \nonumber \\
  &+  2  v_{\max}   (1+5\omega_\proba)    \frac{\sigma^2}{n}  \sum_{t=1}^T \pas_t.  \label{eq:maintheo:sigma2}
\end{align} 
  \end{theorem}

In the case the step sizes are constant ($\pas_t = \pas$ for all $t$), we obtain the following corollary. Here, $\pas$ is chosen small enough to ensure that the LSH is positive.
\begin{corollary}[of \Cref{theo:main}, when $\omega_{\proba}$>0.]  Let $\pas >0$ be such that 
\[
 1  -
  \frac{2 \bar C}{v_{\min}}\pas   -  64 \frac{\omega_\proba}{\alpha^2} \frac{v_{\max}}{ v_{\min}}  \frac{L^2}{n}
 \pas^2 \geq \frac{1}{2}.
\]
Let $\tau$ be a uniform random variable on $\{1, \cdots, T\}$, independent of the path of the algorithm.
    \begin{align*} \PE\left[\barerror_{\tau} \right] & \leq 
\frac{8}{T \pas}  \PE\left[ (f+g)(\map(\hatS_{0}))   - \min (f+g) \right]  +  \frac{32 \omega_\proba v_{\max}}{T\alpha}   \sum_{i=1}^n \mu_i^2 \PE\left[ \|V_{0,i} - \mf_i(\hatS_0)\|^2\right]  \\
& +   16 v_{\max}   (1+5\omega_\proba)    \frac{\sigma^2}{n}.
\end{align*} 

\end{corollary}

\textit{(Parameters and constants).}
 Constants $\sigma^2$ (resp. $L^2$) are defined so that when $\mu_i =1/n$ (i.e., in batch learning with equal number of data at each worker, or in online learning), and $\sigma_i$ (resp. $L_i$) are constant, then $\sigma=\sigma_i$ (resp. $L=L_i$).
 The step sizes $\pas_t$ have to be chosen smaller than an absolute constant depending on the regularity of the problem and the compression factor. \Cref{theo:main} does not require any convexity assumption on $f$ in {\sf MM-1}  though {\sf MM-2}  is satisfied when the surrogate is strongly convex even if not explicitly assumed.

  \textit{(Convergence guarantee).}
  Controlling the quantity at a random stopping time $\tau$, uniformly chosen over all iterations is classical to obtain convergence rates in non-convex optimization~\citep{ghadimi:lan:2013}.  The bound can be decomposed as the sum of an initial quantity, divided by the number of iterations, and a noise term, divided by the number of workers and the batch size $\lbatch$.

  \textit{(Heterogeneity).}
   The result is valid without any assumption on the data heterogeneity ($\pi_i\neq\pi_j$). In the batch learning case, a natural choice for the initial control variates is $V_{0,i}= \mf_i(\hat S_0)$  (at the cost of evaluating the expectation once).  Complexity guarantee  is then not affected (thus robust to) by the heterogeneity.  

  \textit{(Case $\alpha=0$).} When $\omega_\proba=0$, we can choose $\alpha=0$: control variates play no role even in the heterogeneity setting. 
  When $\omega_\proba>0$, the bound is obtained under the assumption that $\alpha>0$. Finite time analysis can be run in the case $\omega_\proba>0, \alpha=0$ but they provide bad conditions on the step sizes $\pas_t$ and poor upper bounds (see \Cref{sec:theowithnullalpha}).   We also illustrate numerically that, in the heterogeneity setting,  $\alpha=0$ may lead to a worsened convergence rate, see \Cref{fig:dict_learn:alpha} in \Cref{sec:dict_learn}.

\paragraph{Challenges in the analysis.} While \Cref{theo:main} provides a first result, the Federated MM framework that we introduce raises numerous theoretical challenges.
\begin{itemize}[noitemsep,leftmargin=*,wide,topsep=0pt]
    \item First and foremost,  the surrogate space $\mathcal S$ can be  constrained (see, e.g., \Cref{ex:SurrogateCvx,ex:EM} in \Cref{sec:checkDL2examples}). The updates resulting from the aggregation of the workers' contributions  may not stay within $\mathcal S $. We are thus required to perform a projection step. This issue is reinforced by variance reduction schemes, that often result in performing non-convex combinations (e.g., the difference of a candidate surrogate with a control variate). It is also reinforced by the compression step, the choice of the compression operator $\Q$ may be impactful; for example, a scalar quantization of a positive semidefinite matrix is not ensured to be a positive semidefinite matrix. Yet, these variance reduction schemes and compression steps are necessary to reduce the impact of heterogeneity and the computational cost, and helpful to accelerate convergence. 
    \item In order to guarantee that $\map(\hatS_{t+1})$ is well defined, \QSMM\ requires a projection of the point $\hatS_t + \pas_{t+1} H_{t+1}$ on $\calS$, yielding $\hatS_{t+1}$. Quite often, this projection step is applied to ensure numerical stability but is not taken into account when deriving a theoretical analysis of the algorithm. On the contrary, the projection is part of our convergence analysis provided in \Cref{theo:main}. Nevertheless, our current analysis requires the projection to be performed w.r.t. the geometry of the matrix  $\B(\hatS_t)$ which can be interpreted as a local approximation of the geometry of the mapping $s\mapsto \phi(\map(s))$ - see \Cref{hyp:DL2}. 
    \item From an analytical standpoint,   the projection introduces bias to the oracle (roughly speaking, $\PE[\Pi(\hat S,\B)]\neq \Pi (\PE[\hat S], \B)$). This bias makes the analysis more complicated  and degrades the convergence rate (the limit variance obtained in \Cref{theo:main} actually resembles the one in \citep[Eq.~(41)]{gorbunov2021marina} for biased online non-convex federated optimization). Moreover, the (convex) space $\mathcal S$ can also be considered as a Riemannian manifold: the projections could also be performed w.r.t.~this geometry.

\end{itemize}

%% file: 4_experiments.tex
\section{Dictionary Learning Experiments}\label{sec:dict_learn}\label{sec:experiments}

We apply the  \QSMM~algorithm to the following stochastic dictionary learning problem 
\begin{equation} \label{eq:dictlearn}
\textstyle \argmin_{ \theta \in \rset^{p \times K} }~\frac{1}{n}\sum_{i=1}^n \PE_{\pi_i} \left[ \min_{ h \in \rset^{K} } \left\{ \frac{1}{2} \| \sample - \theta h \|^2 + \lambda \| h \|_1 \right\} + \eta   \norm{\param}^2  \right],
\end{equation}
 where $\sample \in \rset^p$ is the random observation, $\|\cdot\|_1$ denotes the $L_1$-norm, and $\lambda, \eta > 0$ are regularization parameters.  As detailed in \Cref{subsec:ssmmeg}, this problem fits our variational surrogate framework (\Cref{ex:SurrogateCvx}) and \textbf{\QSMM} can be specialized by: (1) The surrogate space for aggregation is set to $\mathcal{S} \eqdef \mathcal{M}_K^+ \times \rset^{p \times K}$;  (2) At each iteration, a participating client receives  $\map(\hatS_t)$ and uses it to compute its local surrogate oracle $\Smem_{t+1,i} = (\Smem_{t+1,i}^{(1)}, \Smem_{t+1,i}^{(2)})$ - see \eqref{eq:MF:tool3}.   

We consider three data settings, two synthetic and one on real data. For the synthetic setting, we consider two scenarios: homogeneous distributions $\pi_i=\pi$ and heterogeneous distributions $\pi_i \neq \pi_j$. Both start by generating a matrix $\theta_*\in \rset^{p\times K}$ such that $\open{\theta_*}_{ij}\sim \ncal\open{0,1}$. Then, we generate the set $\ens{Z_t\eqdef \theta_* h_t}_{t=1}^\text{tot}$, where, for each $t$, $h_t$ is a sparse vector with $20\%$ of randomly selected non-zero entries, distributed as ${\cal N}(0,1)$. For the scenario where clients have homogeneous distributions, we set $\text{tot}=250$ and give each client a copy of the full data. For the heterogeneous case, with $\text{tot} =5000$,  constrained k-means \citep{bradley2000constrained} is used to split $\ens{Z_t}_{t=1}^\text{tot}$ into $n$ equal clusters while maximizing the average distance between different clusters. Then, each client is assigned a unique cluster. For our experiments, we used $n=20$, $K=15, \lambda=0.1$, and $\eta=0.2$.

For the real data, we used the MovieLens 1M dataset \citep{harper2015movielens}, which consists of the movie ratings of $6000$ users for $4000$ movies. We subsample the dataset to have 5000 user ratings for $500$ movies. Thus, we recover a dataset of $5000$ vectors embedded in $\rset^p$, for $p=500$. To generate the federated clients' data, we split the $5000$ vectors, using constrained k-means, among $n=20$ clients and set the data distributions $\pi_i$ for client $i$ to be the uniform distribution over the {$N_i=$} $ 250$ vectors received. Finally, we set the embedding dimension $K$ to $50$. Similar to the synthetic setting, we used $\lambda=0.1$, and $\eta=0.2$.

For all experiments, we used a decaying step size $\gamma_t = \beta/\sqrt{\beta+t}$ with $\beta$ tuned for each algorithm and setting from the range $[0.001, 0.05]$. For the projection step, line (\ref{line:upModel_central}), we set the projection matrix $\B_t$ to the identity. Nonetheless,  note that our theoretical analysis requires another selection of $\B_t$ (see  \Cref{hyp:DL2} and \Cref{app:hyp:DL}).

To evaluate our methodology, we compare it with a naive algorithm, where at step $t+1$, each active agent $\#i$ computes a new surrogate and optimizes it to obtain a new local parameter $\theta_{t+1}^i\in \Theta$. Then, it communicates $\theta^i_{t+1}$ to the server, which aggregates the parameters of active agents to construct $\theta_{t+1}$.  This naive algorithm exactly mirrors \QSMM{}, except that the communications and the server aggregation step occur in the parameter space and not in the surrogate space. 

Besides tracking the loss value in \eqref{eq:dictlearn}, we report two other metrics. First,  we track the normalized surrogate update squared norm: at iteration $t+1$, for \QSMM{}, we report  $\mathcal{E}^s_{t+1} \eqdef \nicefrac{\norm{\hat{S}_{t+1}-\hat{S}_t}^2}{\gamma_{t+1}^2}$. Meanwhile, for the naive algorithm aggregating in the parameter space, we report  $\mathcal{E}^{s,p}_{t+1} \eqdef \nicefrac{\norm{\tcal(\theta_{t+1})-\tcal(\theta_t)}^2}{\gamma_{t+1}^2}$, where $\tcal\open{\theta} \eqdef \nicefrac{1}{n}\sum_{i\in[n]}\PE_{\pi_i} \left[ \barS(\sample, \theta) \right]$. Similarly, we also track the parameter space update squared norm. In particular, for the algorithm aggregating in the parameter space, we track $\mathcal{E}^p_{t+1}\eqdef\nicefrac{\norm{\theta_{t+1}-\theta_{t}}^2}{\gamma_{t+1}^2}$. Meanwhile, for \QSMM{}, we track $\mathcal{E}^{p,s}_{t+1} \eqdef \nicefrac{\norm{\map(\hatS_{t+1})-\map(\hatS_{t})}^2}{\gamma_{t+1}^2}$. For all the figures, we report the average over $10$ independent runs, obtained with different seeds.

\paragraph{Impact of Aggregation Space:} Our first set of experiments is reported in \Cref{fig:new_matfac_parti}. We set the number of participating clients at each iteration to be $10$ ($\proba = 0.5$). Each active client computes $\Smem_{t+1,i}$ using $50$ examples, sampled at random among the local examples. We considered a setting where the communication between the clients and the server is compressed using quantization with $8$ bits per coordinate. We set the control variates stepsize to $\alpha=0.01$.

First, for the objective value \eqref{eq:dictlearn}, we observe that, when using \QSMM{}, it monotonically decays for all settings. Meanwhile, for the naive algorithm aggregating in the parameter space, it diverges in the synthetic heterogeneous setting. In addition, \QSMM{} achieves better objective value across all settings of data.

For $\mathcal{E}^{p}_{t+1}$ and $\mathcal{E}^{p,s}_{t+1}$, both the naive algorithm and \QSMM{} decay.   For the synthetic data, we observe faster decay for the algorithm designed to aggregate in the parameter space. Meanwhile, for the MovieLens dataset, \QSMM{} initially achieves faster decay but is eventually beaten by the naive algorithm.

Finally, for $\mathcal{E}^{s}_{t+1}$ and $\mathcal{E}^{s,p}_{t+1}$, \QSMM{} decays across all datasets. Meanwhile, the naive algorithm diverges on synthetic data. In addition, for MovieLens dataset, the naive algorithm initially decays steadily but becomes more unstable as the number of communication rounds increases. We note that $\mathcal{E}^{s}_{t+1}$ mirrors the quantity controlled by \Cref{theo:main} (up to a different projection geometry). Thus, the experimental results mirror the theoretical conclusions on convergence up to an error term. In addition, we highlight that even though the naive algorithm converges w.r.t. $\mathcal{E}^{p}_{t+1}$, it diverges in the surrogate space (see $\mathcal{E}^{s,p}_{t+1}$).

\begin{figure}[htb]
  \centering  \includegraphics[width=\textwidth]{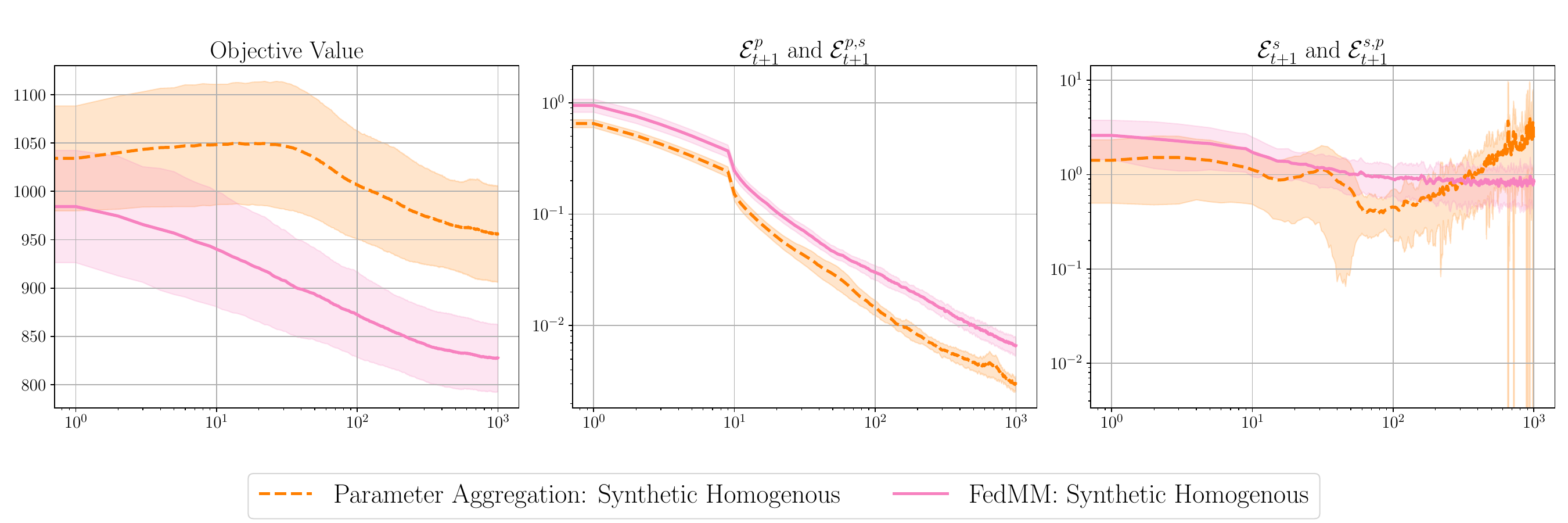}
  
  \includegraphics[width=\textwidth]{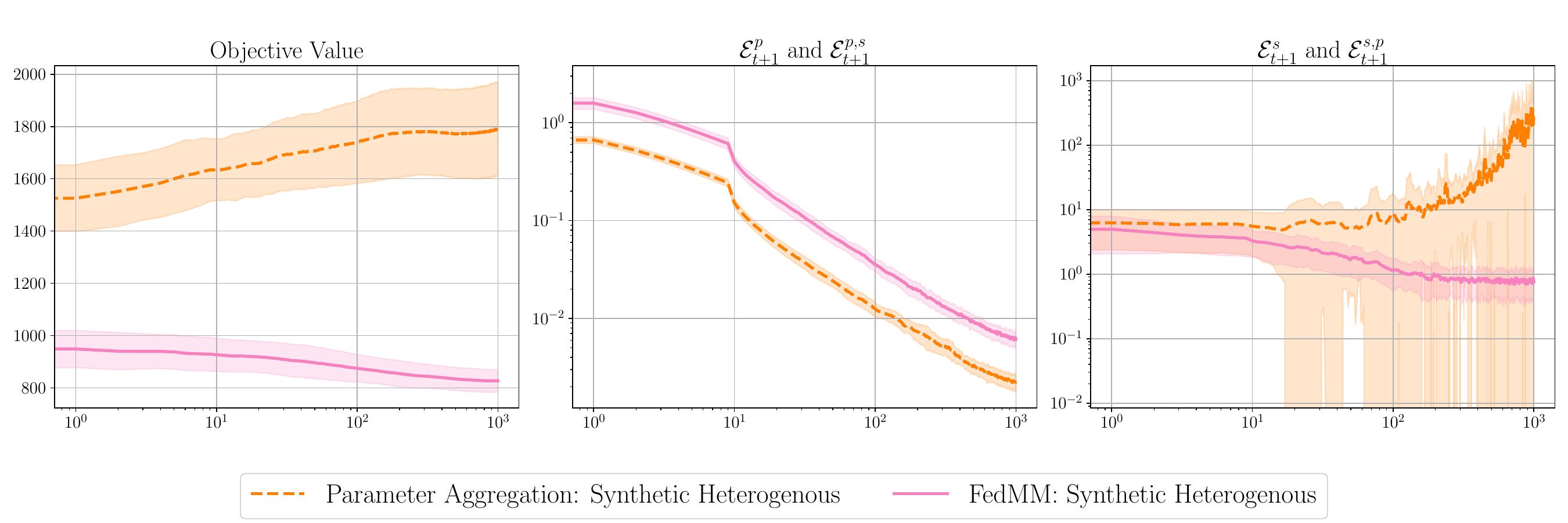}

   \includegraphics[width=\textwidth]{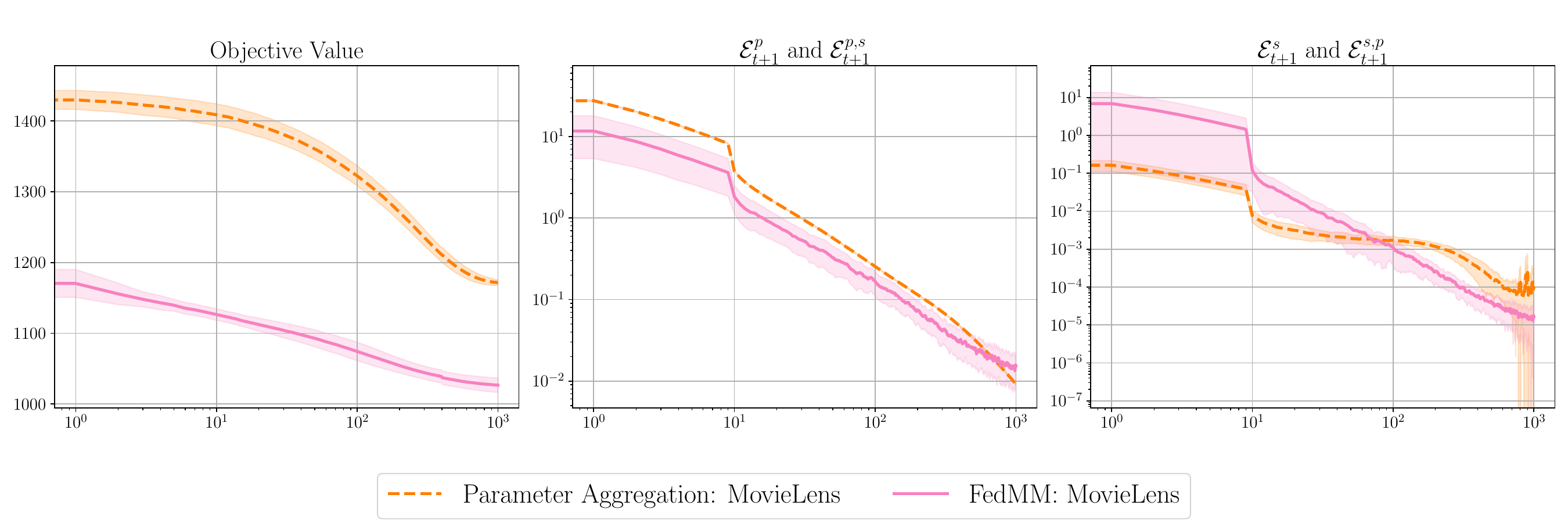}

 \caption{\textbf{Impact of Aggregation Space.} The columns (left to right) display the objective value, parameter update size ($\mathcal{E}_t^{p}$ and $\mathcal{E}_t^{p, s}$), and surrogate update size ($\mathcal{E}_t^{s}$ and $\mathcal{E}_t^{s,p}$) as a function of communication rounds, $t\in 1,\ldots, T_{\mathrm{max}}$. The rows correspond to different data settings (from top to bottom): synthetic homogeneous, synthetic heterogeneous, and MovieLens 1M. Results are averaged over $10$ runs, with the shaded regions indicating one standard deviation of the reported value across the $10$ runs. \label{fig:new_matfac_parti}}
\end{figure}

\paragraph{Impact of Control Variates:}   Our second set of experiments aimed to investigate the effect of control variates in mitigating the consequences of data heterogeneity and partial participation; they are presented in \Cref{fig:dict_learn:alpha}. In this set, we only run $\QSMM$. To investigate and isolate the effect of control variates, we initialize the control variates to zero, i.e. $V_{0,i} = 0$ and set their update step size to $\alpha\in\ens{0,0.01}$. As a result, when $\alpha=0$, control variates are not used. At each step $t$, $10$ agents are active, i.e. $\proba=0.5$. To remove the effect of stochasticity, apart from partial participation, we make each active agent $\#i$  use all their local examples, i.e. $\Smem_{t+1,i} = \PE_{\pi_i}\left[ \barS(\sample, \param_t) \right]$.

Across all data settings, we observe no effect of using control variates on the objective value.  For the synthetic homogeneous dataset and $\alpha=0.01$, given that all the clients share the same data and that there is no communication compression, the control variates exactly cancel out and result in no effect. Meanwhile, for the heterogeneous synthetic data and the MovieLens dataset, we observe that the use of control variates improves the rate of decay of both $\mathcal{E}_{t+1}^{s}$ and $\mathcal{E}_{t+1}^{p,s}$ allowing them to reach significantly lower values.  

\begin{figure}[htb]
  \centering  
  \includegraphics[width=\textwidth]{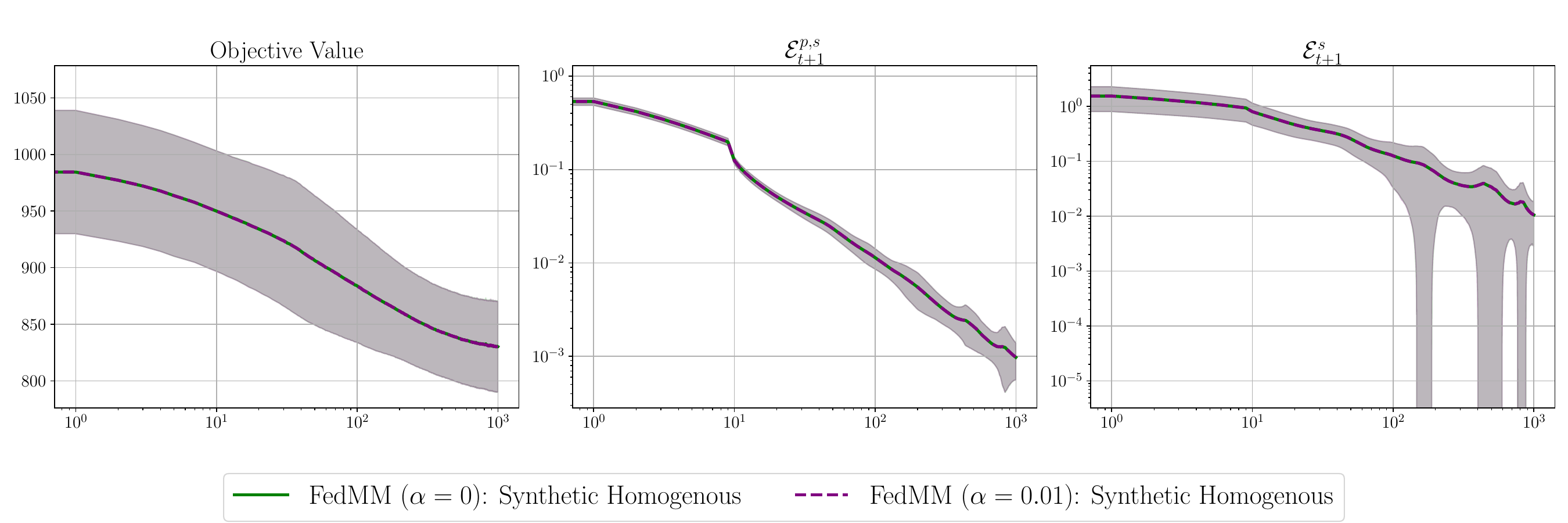}
  \includegraphics[width=\textwidth]{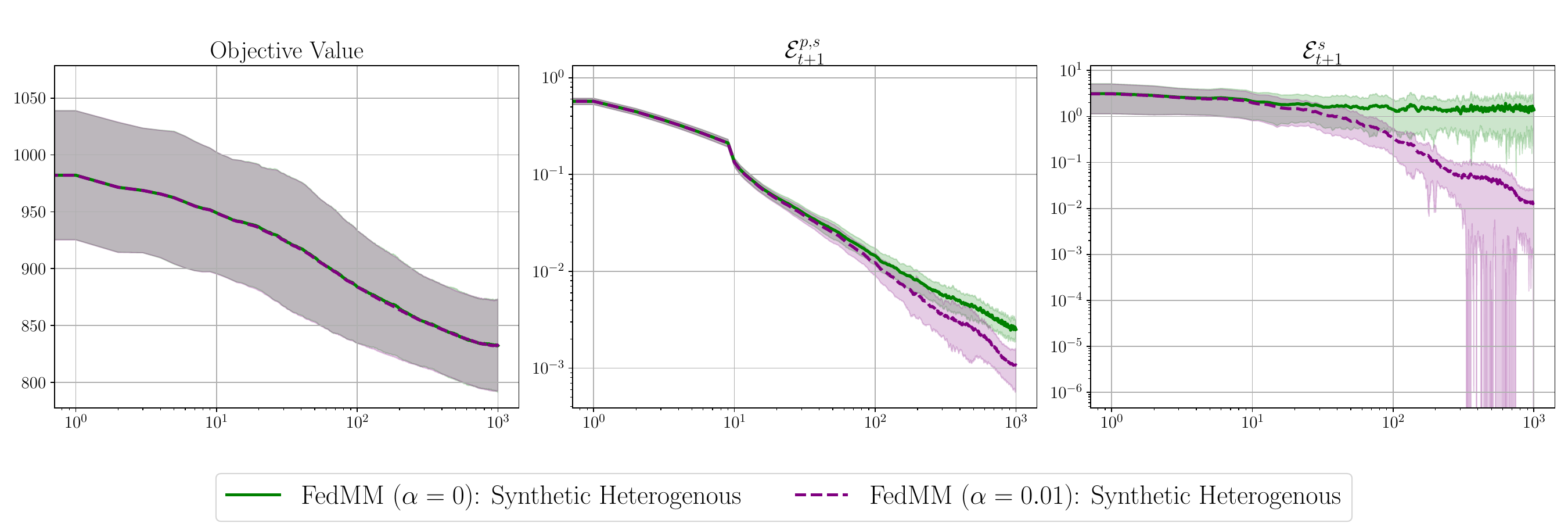}
   \includegraphics[width=\textwidth]{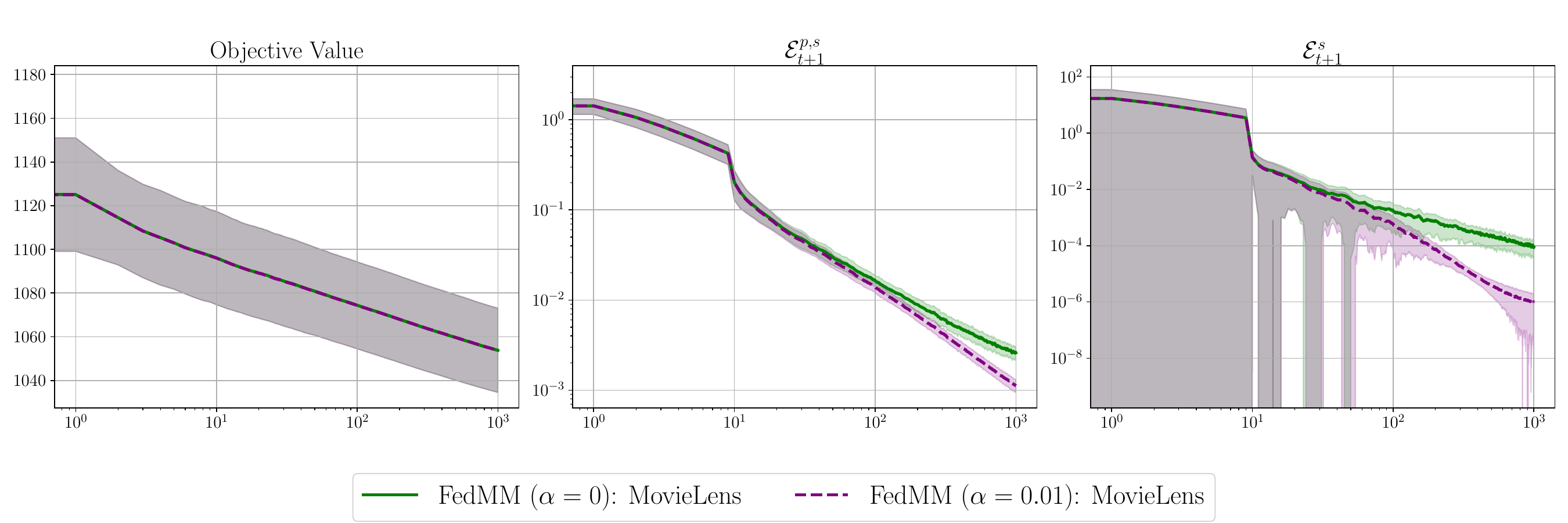}

    \caption{\textbf{Impact of Control Variates.} The columns (left to right) display the objective value, parameter update size ($\mathcal{E}_t^{p,s}$), and surrogate update size ($\mathcal{E}_t^{s}$) as a function of communication rounds, $t\in 1,\ldots, T_{\mathrm{max}}$. The rows correspond to different data settings (from top to bottom): synthetic homogeneous, synthetic heterogeneous, and MovieLens 1M. Results are averaged over $10$ runs, with the shaded regions indicating one standard deviation of the reported value across the $10$ runs. }
  \label{fig:dict_learn:alpha}
\end{figure}

%% file: 4_optimal_transport.tex
\section{Federated Optimal Transport Map with an MM Approach}\label{sec:ot}

We now move beyond the scope of the framework characterized by {\MM12} to show how our methodology of aggregating in the surrogate space may be generalized beyond linearly parameterized surrogates. In particular, to illustrate how our methodology may be used to design novel federated algorithms,  we consider a federated version of the problem of computing parametric optimal transport (OT) maps between two measures according to the framework of \cite{korotin2019wasserstein}. Before presenting the federated extension and our MM-based approach, we first describe the problem in the centralized setting. Consider two measures $\mathcal{P}$ and $\mathcal{Q}$ on $\xcal \subset \rset^d$ and denote by $\mathbb{M}(\xcal)$ the set of measurable functions on $\xcal$  taking values in $\xcal$, the Wasserstein-2 transport map is defined as
\begin{equation}\label{eq:ot_map_primal}
    m^* \in  \mathop{\argmin}_{\substack{\ens{m\circ \pcal = \qcal}\\ m \in \mathbb{M}(\xcal)}}\expec_{\pcal}\ens{\norm{X-m(X)}^2},
\end{equation}
where the condition $\ens{m\circ \pcal = \qcal}$ indicates that the push-forward measure obtained by applying $m$ to $\pcal$ is equal to $\qcal$, i.e. for any Borel set $B\subseteq \xcal$, it holds that $\qcal(B) = \pcal\open{m^{-1}(B)}$. In practice, $m^*$ is computed by first solving the Kantarovich dual problem 
\begin{equation}
    f^* = \mathop{\argmax}_{f\in \text{CVX}(\xcal)}-\expec_\pcal\closed{f(X)} - \expec_\qcal\closed{f^c(Y)},
\end{equation}
where $\text{CVX}(\xcal)$ is the set of convex functions on $\xcal$ and $f^c$ is the convex conjugate of $f$. Under the condition that $\pcal$ and $\qcal$ admit continuous probability density functions, it is known that the solution of the primal formulation may be constructed from the dual solution as $ m^*=\nabla f^*$   \citep{villani2021topics}. In addition, the inverse map $(m^*)^{-1}$ satisfies $(m^*)^{-1} = \nabla\open{f^*}^c$. Given the structure of the problem,  \cite{makkuva20a_icnn_ot} aimed to find the optimal transport map as the gradient of an Input Convex Neural Network (ICNN) \citep{amos17bICNN}, leading to the problem 
\begin{equation*}
    \omega^*\in \mathop{\argmin}_{\omega\in \Theta} \expec_{\pcal, \qcal}\closed{f_\omega(X)+\max_{x\in \xcal}  \{ \inner{x}{Y}-f_\omega(x)\}},
\end{equation*}
where $\expec_{\pcal,\qcal}$ indicates that $X\sim \pcal$ and $Y\sim \qcal$. Here, $\Theta$ represents a parametric family of ICNN. Then, the transport map $m^*$ can be approximated using the map $x\mapsto \nabla_{\!\! x} f_{\omega^*}(x)$. \cite{korotin2019wasserstein} proposes to relax the min-max formulation by using a second ICNN to parameterize the convex conjugate leading to the minimization problem 
\begin{equation}\label{eq:two_icnn}
    \omega^*, \theta^* \in \mathop{\argmin}_{\omega, \theta\in \Theta^2} \ens{\expec_{\pcal}\closed{f_\omega(X)}+\expec_{\qcal}\closed{\inner{\nabla_{\!\! y} f_\theta(Y)}{Y}-f_\omega(\nabla_{\!\! y}f_\theta(Y))}}, 
\end{equation}
where we used the fact that the convex conjugate $f^c$ of a strictly convex function $f$ satisfies 
\begin{equation*}
    \nabla f^c(u) = \argmax_x (\langle u, x \rangle - f(x)).
\end{equation*}
However, a solution to \eqref{eq:two_icnn} may not necessarily guarantee that $f_\theta$ is the convex conjugate of $f_\omega$. Thus, \cite{korotin2019wasserstein} further leverage the structure of the problem and add the regularizer 
\begin{equation*}
       R_\mathcal{Q}(\omega,\theta) \eqdef \expec_{ \mathcal{Q}}\closed{\norm{\nabla f_\omega\circ\nabla {f_\theta}(Y)-Y}^2} ,
\end{equation*}
and thus recovering the formulation
\begin{equation}\label{eq:final_centrak_icnn}
    \omega^*, \theta^* \in \mathop{\argmin}_{\omega, \theta\in \Theta^2} \ens{\expec_{\pcal}\closed{f_\omega(X)}+\expec_{\qcal}\closed{\inner{\nabla f_\theta(Y)}{Y}-f_\omega(\nabla f_\theta(Y))}  +R_\mathcal{Q}(\omega,\theta).} 
\end{equation}

\subsection{Extension to the Federated Setting}

We are interested in extending this problem to the federated setting. Consider a scenario where $n$ clients, each possessing data from a distinct local distribution $\pcal_i$, aim to learn a single, shared optimal transport map to a common target distribution $\qcal$. This formulation is highly relevant for federated domain adaptation. For instance, imagine a consortium of hospitals (clients) wishing to use a diagnostic AI model trained on a large, public dataset (with distribution $\qcal$). Each hospital's local data has a local distribution $\pcal_i$, which may or may not be similar to other hospitals in the consortium. The goal is to leverage the consortium data to learn a transport map that transports the hospital's local distribution onto $\qcal$ on which the model was trained, thus improving its performance without retraining it.

In this federated context, the objective is to find a single pair $(\omega^*, \theta^*)$ that performs well on average across all clients:
\begin{equation}\label{eq:two_icnn_fed}
    \omega^*, \theta^* \in \mathop{\argmin}_{\omega, \theta\in \Theta^2} W(\omega ,\theta)\eqdef \sum_{i\in [n]}\mu_i l_i(\omega, \theta) +\lambda \rcal_\mathcal{Q}(\omega, \theta),
\end{equation}
where 
\begin{equation*}
    l_i(\omega, \theta)\eqdef \ens{\expec_{\mathcal{P}_i}\closed{f_\omega(X)}+\expec_{\qcal}\closed{\inner{\nabla f_\theta(Y)}{Y}-f_\omega(\nabla f_\theta(Y))}}.
\end{equation*}

When $\omega$ and $\theta$ parameterize an ICNN, the objective in \eqref{eq:two_icnn_fed} is differentiable in $(\omega, \theta)$, making it directly tractable by available first-order federated algorithms \citep{reddi2021adaptive,wang2021field}. Nonetheless, we will demonstrate a more effective alternative approach, which is based on our methodology of surrogate aggregation. In particular, the central thesis of our work is that for many optimization problems, the most effective federated approach is to have clients compute local surrogate functions, aggregate these surrogates on the server to form a global surrogate, and then perform the minimization step centrally. We now demonstrate how this design philosophy can be generalized to problems, like the OT formulation above, that do not strictly fit the linearly parameterized framework of {\MM12}.

First, we note that the objective in \eqref{eq:two_icnn_fed} decouples the data dependencies between the two sets of parameters, $\theta$ and $\omega$. In particular, once $f_\omega$ is fixed, the optimization problem for the conjugate network parameters $\theta$, only depends on the shared public distribution $\qcal$. This structural insight allows us to interpret the proposed alternating optimization scheme as a pseudo-MM algorithm that directly applies the principles developed earlier in this paper.

To formalize this, consider the objective function for a single client $i$, which we denote as $\min_{\omega,\theta}W_i(\omega, \theta)$, where $ W_i(\omega,\theta) \eqdef l_i(\omega, \theta) + \lambda \rcal_\qcal(\omega, \theta)$. This local objective can be alternatively viewed as a function of $\theta$ alone as it is equivalent to  $\min_\theta\mathcal{W}_i(\theta)$, where $\mathcal{W}_i\open{\theta}\eqdef \min_\omega W_i(\omega, \theta)$.  In particular, for an iterate  $\theta_t$, consider the client's local best-response $\omega_i(\theta_t) \eqdef \argmin_\omega W_i(\omega, \theta_t)$. This step implicitly defines a majorizing function for the client's true objective. Specifically, the function $U_{i,t}(\theta) \triangleq W_i(\omega_i(\theta_t), \theta)$ is a majorizing surrogate for $\mathcal{W}_i(\theta)$, as it satisfies the two key properties: (1) It is  tangent at the point of construction: $U_{i,t}(\theta_t) = W_i(\omega_i(\theta_t), \theta_t) = \mathcal{W}_i(\theta_t)$; (2) It upper bounds $\mathcal{W}_i$ as $U_{i,t}(\theta) = W_i(\omega_i(\theta_t), \theta) \ge \min_\omega W_i(\omega, \theta) = \mathcal{W}_i(\theta)$ for all $\theta$.

Within our proposed framework, the parameter $\omega_i(\theta_t)$ plays a role directly analogous to the surrogate parameter $\PE_{\pi_i}\left[  \barS(\sample, \theta)\right]$ in the linearly parameterized setting in {\MM12}: it is the finite-dimensional object that fully defines the client's surrogate function $U_{i,t}(\cdot)$. Thus, analogously to \QSMM{}, it is possible to develop an algorithm, which aims to aggregate $\omega_1(\theta_t),\ldots, \omega_n\open{\theta_t}$ to construct a single surrogate at the server, which is then minimized at the server. We operationalize this approach in \Cref{alg:fedrc}, which exploits this idea to solve the federated objective \eqref{eq:two_icnn_fed}.

\begin{algorithm}[H]
    \DontPrintSemicolon
    \caption{Pseudo-MM for Federated OT Map (\fedot{})}\label{alg:fedrc}
    \KwData{$\kmax \in \nset_{>0}$;   $\alpha>0$; $\proba \in \ocint{0,1}$; 
    a $\ocint{0,1}$-valued sequence $\{\pas_t, t \in [\kmax]\}$;  a distribution  $\lawPP$ on the subsets of $\{1, \cdots, n\}$;
    initial values  $\omega_0, \theta_0 \in \Theta$ and  $V_{0,i} \in \Theta$ for all $i \in [n]$} 
    \KwResult{Sequence $\ens{\open{\omega_t, \theta_t}, t\in \closed{T_{\max}}}$}
    Set $V_0 =
    \sum_{i=1}^n \mu_iV_{0,i}$  \;
    \For{$t = 0, \ldots, \kmax-1$}{
        Sample   $\set_{t+1} \sim \lawPP$ \;
        Broadcast $(\omega_t, \theta_t)$ to the agents $\#i$, $i \in \set_{t+1}$\;
        \For(\tcp*[f]{on active worker $\# i$}){$i \in \set_{t+1} $ }{
        Compute $\omega_i(\theta_t) = \argmin_\omega W_i(\omega, \theta_t)$\label{line:ot:min_client}\;
        Set $\Delta_{t+1,i} = \omega_i(\theta_t)-\omega_t-V_{t,i}$\;
        Set $V_{t+1,i} = V_{t,i} + (\nicefrac{\alpha}{\proba}) \Delta_{t+1,i}$ \;
        Send $\Delta_{t+1,i}$ to the central server \label{line:ot:compDelta} \;
        }
        \For(\tcp*[f]{on inactive worker $\#i$}){$i \notin \set_{t+1}$}{
            $V_{t+1,i} = V_{t,i}$\;
        }
        \tcp*{on the central server}
        Set $H_{t+1} = V_t +  (\nicefrac{1}{\proba})\sum_{i\in \set_{t+1}}  \mu_i \Delta_{t+1,i}$ \;
        Compute $\omega_{t+1} =   \omega_t + \pas_{t+1} H_{t+1}$\label{line:ot:update_theta}\;
        Project $\omega_{t+1}$ onto $\Theta$ \; 
        Update  $\theta_{t+1} = \argmin_{\theta\in \Theta} \;
        W(\omega_{t+1},\theta)$\label{line:ot:min_server}  \;
        Set $V_{t+1} = V_t +\nicefrac{\alpha}{\proba} \sum_{i \in \set_{t+1}}\mu_i \Delta_{t+1,i}$ \;
    }
\end{algorithm}

A single round of the algorithm is initiated by the server broadcasting the current global parameters, $(\omega_t,\theta_t)$. Each active client $i \in \set_{t+1}$ is tasked with constructing its local surrogate by computing the best-response parameter $\omega_i(\theta_t)$ for its local potential function. Then, clients transmit a control variate corrected version of their functional parameter update $\omega_i(\theta_t)$.  The server, in turn, performs the aggregation and global minimization steps. It first aggregates the updates received from all active clients to form a new global potential parameter $\omega_{t+1}$.  Using this newly formed $\omega_{t+1}$, the server then performs the global minimization for the conjugate parameter 
$$\theta_{t+1} \in \mathop{\argmin}_{\theta \in \Theta} W(\omega_{t+1}, \theta).$$
In practice, we note that the exact minimization steps in both lines (\ref{line:ot:min_client}) and (\ref{line:ot:min_server}) can be relaxed and approximated by performing multiple gradient steps.

Crucially, the surrogate constructed at the server, parameterized by $\omega_{t+1}$, is not guaranteed to be a majorizing function of the the global objective, and thus \fedot{} does not strictly constitute an MM algorithm.  As such, we term this approach as ``pseudo-MM''. Nonetheless, it follows the identical principle advocated in this paper: local computations produce surrogate parameters, which are then aggregate centrally.

\subsection{Evaluation of Optimal Transport Map}\label{sec:ot_eval}

Evaluating OT map solutions is challenging due to the high sample complexity required for accurate Wasserstein distance estimation, specifically $O(\varepsilon^{-d})$ samples for $O(\varepsilon)$ accuracy \citep{dudley1969speed}, which becomes computationally prohibitive in high dimensions. To address this, we adopt the standardized benchmark methodology for evaluating continuous quadratic-cost OT maps of \cite{korotin2021do}.

This benchmark  leverages ICNNs to construct pairs of continuous measures ($\mathcal{P}_d, \mathcal{Q}_d$) where the ground truth OT map $m^*_{\mathrm{true}}$ is analytically known. In particular, for each dimension $d$, the benchmark starts by defining the following measures: $\mathcal{P}_d$ is a Gaussian mixture with three components, while $\mathcal{Q}^1_d, \mathcal{Q}^2_d$, each a Gaussian mixture with ten components. Using two ICNNs, a push forward function $m_{\mathrm{true}}^*:\xcal\rightarrow\xcal$ is learned to approximate $T^*\circ\pcal \approx \nicefrac{1}{2}\open{\qcal^1_d+\qcal^2_d}$. Finally, the benchmark sets $\qcal_d = T^*\circ\pcal$ and, as such, the transport map between $\pcal_d$ and $\qcal_d$ is explicitly known. 

To quantify the performance of a fitted parametric transport map $m: \mathcal{X} \to \mathcal{X}$ against this ground truth, the benchmark employs the $L_2$ Unexplained Variance Percentage  metric ($L_2^{\text{UVP}}(m)$), which is formally defined as 
 
$$ L_2^{\text{UVP}}(m) \eqdef 100 \cdot \frac{\|m -m^*_{\mathrm{true}}\|^2_{L^2(\mathcal{P}_d)}}{\text{Var}(\mathcal{Q}_d)},$$
where $\text{Var}(\mathcal{Q}_d)$ is a scalar-valued notion  of the spread of $\mathcal{Q}_d$. We use the original implementation\footnote{https://github.com/iamalexkorotin/Wasserstein2Benchmark} to have $\text{Var}(\qcal_d)$ being the $L_1$ norm of the covariance matrix of $\qcal_d$.

The $L_2^{\text{UVP}}$ value serves as an indicator of how accurately the computed transport map $m$ approximates the true optimal transport map $m^*_{\mathrm{true}}$ .  

\subsection{Experimental Results}

To benchmark the effectiveness of \Cref{alg:fedrc}, we again compare against an algorithm that aims to solve the federated objective without aggregation of surrogate parameters. In particular, the objective in \eqref{eq:two_icnn_fed} is differentiable in $(\omega, \theta)$, making it amenable to federated first-order methods \citep{mcmahan_communication-efficient_2017, reddi2021adaptive}. As such, we compare \Cref{alg:fedrc} with FedAdam \citep{reddi2021adaptive}.

We adapted the experimental setting and hyperparameters used by \citet[MMv2]{korotin2021do}. In particular, we use a three-layer dense ICNN \citep{korotin2019wasserstein} for $f_\theta$ and $f_\omega$. In addition, except for the learning rate, which we tune, we use the {{\tt Adam}} \citep{kingma_adam_2017} hyperparameters used in the centralized setting in the code of \cite{korotin2021do} for {{\tt FedAdam}}.

For \Cref{alg:fedrc}, we relax the optimization procedure on the clients' side in line (\ref{line:ot:min_client}) to be a single gradient step. As such, both \Cref{alg:fedrc} and {{\tt FedAdam}} use similar computational resources on the client side. Finally, for the server side optimization process in line (\ref{line:ot:min_server}) to be ten gradient steps using the {{\tt Adam}} optimizer.

We report our results on the benchmark from \cite{korotin2021do} (described in \Cref{sec:ot_eval}), using the $L_2^{\text{UVP}}$ metric. To adapt this benchmark to the federated setting, we generate $\pcal_d$ as previously detailed. Then, for $n=10$ clients, we create individual client distributions by splitting samples from $\pcal_d$ using constrained k-means. Finally, in \eqref{eq:two_icnn_fed}, we set $\mu_i = \nicefrac{1}{n}$. The results for different dimensions are presented in \Cref{fig:ot_map_res}. Across all dimensions, we observe that our approach significantly outperforms the direct application of {{\tt FedAdam}}.

\begin{figure}[h!]
  \centering  \includegraphics[width=\textwidth]{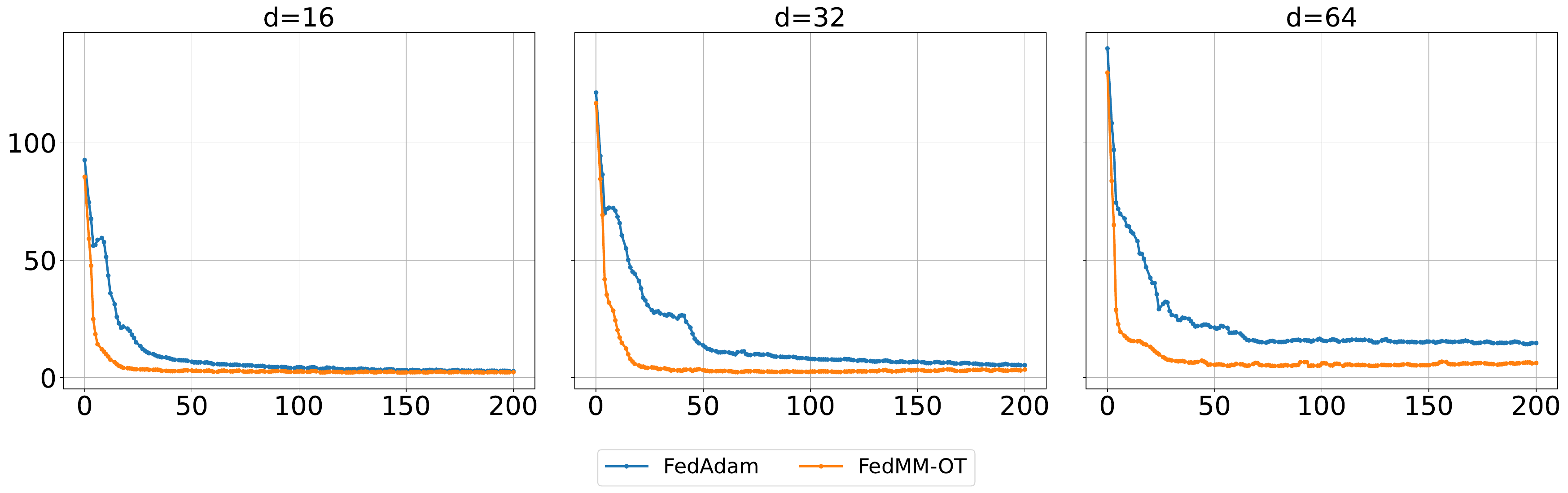}
 \caption{Performance is measured by ($L_2^{\text{UVP}}$), where lower is better, versus communication rounds for dimensions $d=16$, $d=32$, and $d=64$. In all settings, our proposed \fedot{} algorithm converges faster than the {{\tt FedAdam}} baseline.}
 \label{fig:ot_map_res}
\end{figure}

%% file: 5_conclusion.tex
\section{Conclusion}

This paper investigates Stochastic Surrogate MM methods that iteratively construct surrogate functions in a \textit{parametric space}.
This principle allows us to introduce \QSMM{}, a unified framework for a novel class of FL algorithms designed to address the key bottlenecks of communication and heterogeneity. We believe that this simple yet elegant idea opens up many algorithmic directions, and simultaneously poses exciting theoretical difficulties.

Finally, a point of interest is the link with Mirror Descent (MD).  \citet{kunstner2021homeomorphic} showed that  a (deterministic) MM - namely, the EM algorithm -  can be seen as a MD algorithm. This remark strongly relies on a canonical parameterization $\phi(\theta)\equiv \theta$ of the surrogate functions, see {\sf MM-2}  making the analysis simpler \citep[eq. (3),][]{kunstner2021homeomorphic}. Understanding the impact of the parameterization on the behavior of Federated MM is an interesting aspect left unexplored.

 \section*{Acknowledgments} The work of Aymeric Dieuleveut and Mahmoud Hegazy is supported by  French State aid managed by the French Agence Nationale de la Recherche (ANR) under France 2030 program with the reference ANR-23-PEIA-005 (REDEEM project), and by ANR-23-IACL-0005 (Hi! PARIS AI Cluster), in particular Hi! PARIS \textit{FLAG} chair. The work of Gersende Fort is partly supported by ANR under the projects ANR-23-CE48-0009 and ANR-24-CE23-1529.
 The work of Hoi-To Wai is supported by project \#MMT-p5-23 of the Shun Hing Institute of Advanced Engineering, The Chinese University of Hong Kong

%% file: appendix.tex
\newpage

\section{The conditions {\bf {\sf MM-1}} and {\bf {\sf MM-2}}}
\label{app:exampleMM1-2}
\subsection{\Cref{ex:gdt}}
Let $\rho \in \ocint{0,\tochangeinmain{1/L_f}}$.  For any $\tau \in \paramset$, it holds on $\paramset$
\[
f(\cdot) \leq f(\tau) + \pscal{\nabla f(\tau)}{\cdot- \tau} +
\frac{L_f}{2} \|\cdot - \tau \|^2 \leq f(\tau) + \pscal{\nabla
  f(\tau)}{\cdot - \tau} + \frac{1}{2 \rho} \|\cdot- \tau \|^2 .
\]
Since
\[ 
\pscal{\nabla f(\tau)}{ \param-\tau} + \frac{1}{2 \rho} \|\param -
  \tau \|^2  = \frac{1}{2\rho} \left( \|\param \|^2 - \|\tau \|^2 \right) -  \frac{1}{\rho} \pscal{\tau - \rho \nabla f(\tau)}{\param-\tau}, 
\]
 {\bf {\sf MM-1}}  is satisfied with
\[
\psi(\param) \eqdef \frac{1}{2\rho} \|\param\|^2, \qquad \phi(\param)
\eqdef \frac{1}{\rho} \param, \qquad \barS(\sample, \tau) \eqdef \tau
- \rho G(\sample, \tau).
\]
Upon noting that $\nicefrac{1}{2} \| \param \|^2 - \pscal{s}{\param}= \nicefrac{1}{2} \| \param - s \|^2 - \nicefrac{1}{2} \| s \|^2 $, it holds  
\begin{align*}
\map(s) &\eqdef \mathrm{argmin}_\param \left( g(\param) +
\frac{1}{2\rho} \|\param\|^2 - \frac{1}{\rho}\pscal{s}{\param} \right) = \mathrm{argmin}_\param \left( \rho g(\param) +
\frac{1}{2}   \|\param - s \|^2 \right).
\end{align*}

\subsection{\Cref{ex:EM}}
\label{sec:detailedex2}
Let $\tau\in \paramset$. For any $\param \in \paramset$, we write
  \[
\loss(\sample,\param) - \loss(\sample, \tau) =  - \log \int_\calH
\frac{p(\sample,h,\param)}{p(\sample,h,\tau)}
  \mu(\rmd h \vert \sample, \tau).
  \]
  By the Jensen's inequality, we have  for any $\param \in \paramset$
  \[
\loss(\sample,\param) - \loss(\sample, \tau) \leq -  \int_\calH
\log \left( \frac{p(\sample,h,\param)}{p(\sample,h,\tau)} \right)  \mu(\rmd h \vert \sample, \tau).
\]
Since $-\log p(\sample, h, \param) = \psi(\param) - \pscal{S(\sample,h)}{\phi(\param)}$, this yields
\begin{equation} \label{eq:toolex2}
\loss(\sample,\param)  \leq \loss(\sample, \tau) + \psi(\param) - \psi(\tau) - \pscal{\barS(\sample,\tau) }{\phi(\param) -\phi(\tau)}
\end{equation}
where $\barS(\sample,\tau)$ is given by  \eqref{ex:EM:barS}. 
We now apply the expectation $\PE_\pi$, and this concludes the proof of {\bf \sf MM-1}.

\subsection{\Cref{ex:SurrogateCvx}}
\label{sec:detailedex3}
Let $\tau \in \paramset$. First, observe that by definition of $
\mapM(\sample, \tau)$, it holds for any $\sample$
\begin{equation}\label{eq:toolex3}
\ell(\sample, \tau) = \psi(\tau) - \pscal{S(\sample, \mapM(\sample,\tau))}{\phi(\tau)} + \xi(\sample, \mapM(\sample,\tau)). 
\end{equation}
We have $\loss(\sample, \cdot)
 \leq \tildeloss(\sample, h, \cdot)$ for any $h \in \calH$. Apply this
 inequality with $h \leftarrow \mapM(\sample,\tau)$ and use \eqref{eq:toolex3}: this yields
  \begin{align}
  \loss(\sample, \cdot) & \leq \tildeloss(\sample,\mapM(\sample, \tau), \cdot) =  \psi(\cdot)-
  \pscal{S(\sample, \mapM(\sample, \tau))}{\phi(\cdot)} + \xi(\sample, \mapM(\sample,\tau)), \qquad \forall \sample \in \rset^p, \\
  & = \ell(\sample, \tau) + \psi(\cdot)- \psi(\tau) - 
  \pscal{S(\sample, \mapM(\sample, \tau))}{\phi(\cdot) - \phi(\tau)} \label{eq:toolex3bis} 
  \end{align}
From this inequality and \eqref{eq:toolex3}, it follows  under integrability conditions that
  \[
f(\cdot) \leq f(\tau) +
 \psi(\cdot)  - \psi(\tau) - \pscal{\PE_\pi \left[S(\sample, \mapM(\sample, \tau))
    \right]}{\phi(\cdot) - \phi(\tau)}.
\]
Therefore, {\bf \sf MM-1} is satisfied with $\barS(\sample, \tau) \eqdef S(\sample, \mapM(\sample, \tau))$
and the functions $\phi, \psi$ defined by $\tilde \loss$.

\section{Proof of \Cref{prop:equiv:theta:s}}
\label{sec:proof:prop:equiv:theta:s}
By definition of $\map$ (see {\sf MM-2}) and the Fermat's rule (see e.g. \cite[Theorem 16.3]{bauschke:combettes:2011}), for any $s \in \calS$, we have 
\begin{equation}\label{prop:equiv:theta:s:tool1}
0 \in \partial \left( g(\cdot) + \psi(\cdot) - \pscal{s}{\phi(\cdot)} \right)\left(\map(s)\right). 
\end{equation}
Let $s \in \calS$. Under \Cref{hyp:surrogate:convexC1}, both $g(\cdot)$ and $\psi(\cdot) - \pscal{s}{\phi(\cdot)}$ are  proper lower semicontinuous convex functions, and  $\paramset \cap \inter \mathrm{dom}(\psi(\cdot) - \pscal{s}{\phi(\param)} \supset \inter \paramset$ is non empty. Therefore, by \cite[Corollary 16.50]{bauschke:combettes:2011}, for any $\tau \in \paramset$
\begin{equation}\label{prop:equiv:theta:s:tool2}
 \partial \left( g(\cdot) + \psi(\cdot) - \pscal{s}{\phi(\cdot)} \right)\left(\tau\right) = \partial g(\tau) + \partial\left( \psi(\cdot) - \pscal{s}{\phi(\cdot)} \right)(\tau).
\end{equation}
\Cref{hyp:surrogate:convexC1} and \cite[Proposition 17.31]{bauschke:combettes:2011} imply that for any $\tau \in \paramset$,
\begin{equation}\label{prop:equiv:theta:s:tool3}
  \partial\left( \psi(\cdot) - \pscal{s}{\phi(\cdot)} \right)(\tau) = \{ \nabla \psi(\tau) - J_\phi(\tau) \, s \}. 
\end{equation}
Therefore, from \eqref{prop:equiv:theta:s:tool1}, \eqref{prop:equiv:theta:s:tool2} and \eqref{prop:equiv:theta:s:tool3}, for any $s \in \calS$, we have \begin{equation}\label{prop:equiv:theta:s:tool4}
   J_\phi(\map(s)) \, s  - \nabla \psi(\map(s)) \in \partial  g(\map(s)).
\end{equation}
On the other hand,  {\bf {\sf MM-1}} implies that for any $\tau \in \paramset$, $ \param \mapsto  f(\param) - f(\tau) - \psi(\param) + \psi(\tau) + \pscal{\PE_\pi\left[ \barS(\sample, \tau)\right]}{\phi(\param) - \phi(\tau)} \leq 0$ and is equal to zero at $\param = \tau$. Therefore, under \Cref{hyp:surrogate:convexC1}  and \Cref{hyp:objective:C1}, for any $\tau \in \inter \paramset$, it holds 
\begin{equation}\label{prop:equiv:theta:s:tool5}
  \nabla f(\tau) - \nabla \psi(\tau) + J_\phi(\tau) \, \PE_\pi\left[ \barS(\sample, \tau)\right] = 0.
\end{equation}
For any $\tau \in \inter \paramset$, we have  
\[
\begin{split}
\nabla f(\tau) + \partial g(\tau) & \overset{\eqref{prop:equiv:theta:s:tool5}}{=} \nabla \psi(\tau) - J_\phi(\tau) \, \PE_\pi\left[ \barS(\sample, \tau)\right] + \partial g(\tau) \\
& \overset{\eqref{prop:equiv:theta:s:tool3}}{=}  \partial\left( \psi(\cdot) - \pscal{s}{\phi(\cdot)} \right)(\tau) + \partial g(\tau)  \\
& \overset{\eqref{prop:equiv:theta:s:tool2}}{=} \partial\left( g (\cdot) + \psi (\cdot) -\pscal{\PE_\pi\left[ \barS(\sample, \tau)\right]}{\phi(\cdot)}\right)(\tau).
\end{split}
\]
Combined with the Fermat's rule, this implies that $0 \in \nabla f(\tau) + \partial g(\tau)$ iff $\tau$ is a minimizer of $g(\cdot) + \psi(\cdot) -\pscal{\PE_\pi\left[ \barS(\sample, \tau)\right]}{\phi(\cdot)}$, i.e., under {\bf {\sf MM-2}}, $\tau = \map(\PE_\pi\left[ \barS(\sample, \tau)\right])$. This concludes the proof of \eqref{eq:prop:equiv:theta:s:state2}.

On the other hand, the result \eqref{eq:prop:equiv:theta:s:state1} follows by observing that as $\map(s) \in \inter \paramset$,
\[
\begin{split}
- J_\phi(\map(s)) \,  \mf(s) & = J_\phi(\map(s)) \left( s - \PE_\pi\left[ \barS(\sample,  \map(s) )\right] \right) \\
& \overset{\eqref{prop:equiv:theta:s:tool5}}{=} \nabla f( \map(s) ) - \nabla \psi( \map(s) ) + J_\phi(\map(s)) \, s \\
& \overset{\eqref{prop:equiv:theta:s:tool3}}{\in} \nabla f( \map(s) ) + \partial g( \map(s) ).
\end{split}
\]

Let $s_\star \in \calS$ be such that $\mf( s_\star ) = 0$ and $\map(s_\star) \in \inter \paramset$, then  \eqref{eq:prop:equiv:theta:s:state2} and \eqref{eq:prop:equiv:theta:s:state1} imply that $\param_\star = \map( s_\star )$. 
On the other hand, let
$\param_\star\in \inter \paramset$ be such that $\param_\star = \map\left( \PE_\pi \left[\barS(\sample, \param_\star) \right]\right)$, then \eqref{eq:prop:equiv:theta:s:state2} and \eqref{eq:prop:equiv:theta:s:state1} imply $\param_\star = \map(s_\star)$, and furthermore $s_\star = \PE_\pi \left[\barS(\sample, \map(s_\star)) \right]$. This yields $\mf(s_\star) = 0$.

\section{Examples of Jensen's surrogate: the Expectation Maximization algorithm}
\label{app:exampleEM}

We derive the surrogate parameterization for two concrete EM examples in \Cref{subsec:ssmmeg}.

\subsection{First Example: Poisson observations}\label{app:em_first_example}
\paragraph{Statistical model.} Consider a Poisson
random variable; conditionally to
some hidden random variable $h$, the intensity of the Poisson
variable  is $\exp(\param + h)$. We assume that the distribution of $h$ is 
$\mu$. Based on examples $\sample_1, \cdots, \sample_j, \cdots$, we want to estimate $\param \in \rset$ as the Maximum A Posteriori (MAP) when the prior is proportional to $\param \mapsto \exp(-\lambda \exp(\param))$ for some regularization parameter $\lambda >0$. The complete likelihood i.e. the likelihood  of the pair $(\sample, h)$ is proportional to $\param \mapsto \exp(\sample (\param + h) - \exp(\param +
  h))$. 
A MAP estimator is a minimizer of 
\begin{equation}\label{eq:EMexample:Poisson}
\param \mapsto   \PE_\pi \left[ - \log \int_{\rset} \exp\left( \param  \sample  - \exp(\param) \exp(h)
\right) \mu(\rmd h)  \right]  +  \lambda \exp(\param);
\end{equation} 
it is of the form $\param \mapsto f(\param) + g(\param)$ with $g(\param) \eqdef \lambda \exp(\param)$. 
In batch learning with i.i.d. observations $\sample_j$, $\pi$ is the empirical distribution of the $N$ examples in the batch: $\pi(\rmd z) \eqdef \nicefrac{1}{N} \sum_{j=1}^N \delta_{\sample_j}(\rmd z)$.  In online learning with a stream of independent examples having the same distribution as $\sample$, $\pi$ denotes the distribution of $\sample$.

\paragraph{Setting {\bf \sf MM-1} and {\bf \sf MM-2}.}  
 From \eqref{eq:EMexample:Poisson} and  \Cref{ex:EM}, {\bf \sf MM-1} is  satisfied with 
\[
\psi(\param) \eqdef  0, \quad \phi(\param) \eqdef  \begin{bmatrix} \param \\ \exp(\param) \end{bmatrix}, \quad  S(\sample,h) \eqdef \begin{bmatrix} \sample \\ -\exp(h) \end{bmatrix}, \quad
\paramset \eqdef \rset, \quad \calS \eqdef \rset_{>0} \times \rset_{<0}.
\]
{\bf \sf MM-2} is satisfied: the optimization operator $\map$ is defined on $\calS$ by   
\[
\map(s_1,s_2) \eqdef \mathrm{argmin}_{\param \in \rset}\left \{ \lambda \exp(\param) -
s_1 \param - s_2 \exp(\param) \right\} = \ln \left( \frac{s_1}{\lambda - s_2}\right).
\]
Upon noting that the first component of $\PE_\pi\left[ \barS(\sample, h, \tau) \right]$ does not depend on $\tau$ and is equal to $\PE_\pi[\sample]$, let us propose a second strategy when this expectation is explicit (typically in batch learning) possibly at a large computational cost (when the size $N$ of the batch is large). The function $\psi(\cdot)$ contains the terms depending upon $\param$ which have a closed form expression.

\paragraph{Setting {\bf \sf MM-1} and {\bf \sf MM-2} -  when $\PE_\pi\left[\sample \right]$ is explicit.} 
 {\bf \sf MM-1} is satisfied with 
\[
\psi(\param) \eqdef - \param \, \PE_\pi\left[\sample \right], \quad \phi(\param) \eqdef  \exp(\param) , \quad  S(\sample,h) \eqdef -\exp(h), \quad
\paramset \eqdef \rset, \quad \calS  \eqdef  \rset_{<0}.
\]
In {\bf \sf MM-2}, the optimization operator $\map$ is defined on $\calS$ by 
\[
\map(s) \eqdef \mathrm{argmin}_{\param \in \rset}\left \{ \lambda\exp(\param) -
\PE_\pi\left[ \sample \right] \param - s \exp(\param) \right\} = \ln \left( \frac{ \PE_\pi\left[ \sample \right]}{\lambda - s}\right).
\]

\subsection{Second example: mixture of Gaussian distributions}\label{app:em_second_example}
\paragraph{The statistical model.} Consider  a mixture model of $L$ $\rset^p$-valued Gaussian distributions. The   means $m_1, \cdots, m_L$ are unknown; the weights $\nu_1, \cdots, \nu_L$ and the covariance matrices $\Gamma_1, \cdots, \Gamma_L$ are known. Based on observations $\sample_1, \cdots, \sample_j, \cdots$, we want to estimate $\param \eqdef (m_1, \cdots, m_L) \in \rset^{p \times L} $ as a penalized Maximum Likelihood (ML) estimator that minimizes
\[
\param \mapsto \PE_\pi \left[- \log \sum_{\ell=1}^L \frac{\nu_\ell}{\sqrt{\mathrm{Det}(\Gamma_\ell)}} \exp\left(- \frac{1}{2} (\sample-m_\ell)^\top \Gamma_\ell^{-1} (\sample -
m_\ell) \right) \right] + \frac{\lambda}{2} \sum_{\ell=1}^L m_\ell^\top m_\ell,
\]
where $\lambda >0$ is a regularization parameter.  $\pi$ is either $\pi(\rmd z) \eqdef N^{-1} \sum_{j=1}^N \delta_{\sample_j}(\rmd z)$, the empirical distribution on the $N$ i.i.d. examples $\sample_j$ in batch learning; or $\pi$ denotes the distribution of $\sample$ in online learning with a stream of i.i.d. samples. 

The criterion is of the form   $\param \mapsto f(\param) + g(\param)$ with $g(\param) \eqdef \nicefrac{\lambda}{2} \, \sum_{\ell=1}^L m_\ell^\top m_\ell$ and
\begin{align*}
f(\param) & \eqdef \PE_\pi \left[- \log \sum_{\ell=1}^L \frac{\nu_\ell}{\sqrt{\mathrm{Det}(\Gamma_\ell)}} \exp\left(- \frac{1}{2} (\sample-m_\ell)^\top \Gamma_\ell^{-1} (\sample -
m_\ell) \right) \right] \\ 
& = \PE_\pi \left[ - \log \int  \frac{1}{\sqrt{\mathrm{Det}(\Gamma_h)}} \exp\left(- \frac{1}{2} (\sample-m_h)^\top \Gamma_h^{-1} (\sample -
m_h) \right)   \mu(\rmd h)   \right]
\end{align*}
where  $\mu$ is the distribution on $\{1, \cdots, L\}$ defined by $\mu(\{\ell\}) \eqdef \nu_\ell$.

\paragraph{The setting {\bf \sf MM-1}-{\bf \sf MM-2}.}
It holds 
\[
 (\sample-m_h)^\top \Gamma_h^{-1} (\sample -
m_h) = \sample^\top \Gamma_h^{-1} \sample + m_h^\top \Gamma_h^{-1} m_h - 2 \pscal{ \sample}{\Gamma_h^{-1} m_h}.
\]
Since $1 = \1_{\{1\}}(h) + \cdots + \1_{\{L\}}(h)$, we have 
\[
\pscal{\xi}{\chi(h)} = \pscal{\begin{bmatrix} 
\xi \1_{\{1\}}(h) \\
\cdots \\
\xi \1_{\{L-1 \}}(h) \end{bmatrix}}{ \begin{bmatrix}
\chi(1) -\chi(L) \\ \cdots \\ \chi(L-1) -\chi(L) \end{bmatrix}} +\pscal{\xi}{\chi(L)} 
\]
from which we deduce that 
\begin{align*}
\pscal{\sample}{\Gamma_h^{-1}  m_h}   & =  \pscal{\begin{bmatrix}  \sample \, \1_{\{1\}}(h) \\ \cdots \\ \sample \, \1_{\{L-1\}}(h)\end{bmatrix}}{\begin{bmatrix} \Gamma_1^{-1} m_1-\Gamma_L^{-1} m_L \\ \cdots \\ \Gamma_{L-1}^{-1} m_{L-1}-\Gamma_L^{-1} m_L \end{bmatrix}}+ \pscal{\sample}{\Gamma_L^{-1} m_L}. \\
m_h^\top \Gamma_h^{-1} m_h   & = 
\pscal{ \begin{bmatrix}  1_{\{1\}}(h) \\ \cdots \\ 1_{\{L-1 \}}(h)\end{bmatrix}}{ \begin{bmatrix}m_1^\top \Gamma_1^{-1} m_1  - m_L^\top \Gamma_L^{-1} m_L \\ \cdots \\  m_{L-1}^\top  \Gamma_{L-1}^{-1} m_{L-1} - m_L^\top \Gamma_L^{-1} m_L  \end{bmatrix}  }+ m_L^\top \Gamma_L^{-1} m_L.
\end{align*}
Therefore, from  \Cref{ex:EM}, {\sf \bf MM-1} is satisfied with  $\psi(\param) \eqdef - m_L^\top \Gamma_L^{-1} \PE_\pi\left[ Z \right]  + \nicefrac{1}{2} m_L^\top
\Gamma_L^{-1} m_L$, 
\[\phi(\param) \eqdef \begin{bmatrix} \Gamma_1^{-1}  m_1  - \Gamma_L^{-1} m_L \\
\cdots \\
\Gamma_L^{-1}     m_L   - \Gamma_L^{-1} m_L \\ 
 - \nicefrac{1}{2} \left( m_1^\top \Gamma_1^{-1} m_1 -
m_L^\top \Gamma_L^{-1} m_L \right) \\
\cdots \\
 - \nicefrac{1}{2} \left( m_{L-1}^\top \Gamma_{L-1}^{-1} m_{L-1} -
m_L^\top \Gamma_L^{-1} m_L \right)  \end{bmatrix}, 
  \qquad  S(\sample, h) \eqdef
\ \begin{bmatrix} \sample \, \1_{\{1\}}(h) \\
\cdots \\ 
\sample \,  \1_{\{L-1\}}(h)  \\
 \1_{\{1\}}(h) \\
 \cdots \\ 
 \1_{\{L-1\}}(h)  \end{bmatrix}.
\]
The set $\calS$ collects the points of the form $s=(s_{1,1}, \cdots, s_{1,L-1}, s_{2,1}, \cdots, s_{2,L-1})$  where $s_{1,\ell} \in \rset^p$,  $s_{2,\ell} \in \ccint{0,1}$ and $\sum_{\ell=1}^{L-1} s_{2,\ell} \leq 1$. In {\bf \sf MM-2}, for any $s \in \calS$, $\map(s)$ is the minimizer of $g(\param) + \psi(\param) - \pscal{s}{\phi(\param)}$. $\map(s) \in \rset^{p \times L}$ 
and the column $\# \ell$, for $\ell=1, \cdots, L-1$, is given by 
\[
 \left[\map(s) \right]_{:,\ell}\eqdef \left(s_{2,\ell} \Id_p + \lambda \Gamma_\ell \right)^{-1} \, s_{1,\ell}; 
\]
the column $\# L$ is given by
\[
\left[\map(s) \right]_{:,L}\eqdef \left((1-\sum_{\ell=1}^{L-1} s_{2,\ell}) \Id_p   + \lambda  \Gamma_L 
\right)^{-1}    \left( \PE_\pi[Z] - \sum_{\ell=1}^{L-1} s_{1,\ell}\right). 
\]

%% file: appendix_proof.tex
\section{Proof of \Cref{theo:main}} \label{sec:proof:maintheo}

\subsection{Outline of the section} \label{sec:outlineproof}
The proof of \Cref{theo:main} can be carried for the \textit{full participation regime} up to adjusting the properties of the compression operator i.e. when \Cref{assumption:PP}($1$) holds, with an adjusted compression operator  satisfying \Cref{assumption:URVB}($\omega_\proba$) where $\omega_\proba \eqdef \omega + \nicefrac{(1-\proba)}{\proba} \, (\omega+1)$. Such a property is addressed in \Cref{sec:PP2fullP} together with the statement of \Cref{algo:FedSS-MM:fullP}, equivalent  under \Cref{assumption:PP} to \Cref{algo:FedSS-MM}.  This line of analysis has been frequently used in the literature, for example in \citet{koloskova2019decentralized}.

Stochastic quantities are successively introduced in the learning process: when selecting active workers; when producing oracles $\Smem_{t+1,i}$ and possibly in the compression step. Finally, the local randomnesses are aggregated by the central server when updating the  system $\hatS_{t+1}$. In \Cref{sec:proof:filtration}, a filtration  is introduced in echo to these steps.

Finally, the proof of \Cref{theo:main} is in \Cref{sec:proof:detailedprooftheo}; technical lemmas conclude this appendix (see \Cref{sec:prelimlemma}).

\subsection{From Partial Participation to Compression} \label{sec:PP2fullP}
Let 
\begin{equation*}
    \Qtilde(s) \eqdef \frac{U}{\proba}\Q(s), \qquad s \in \rset^q,
\end{equation*}
where $U$ is a Bernoulli random variable with probability of success $\proba$ and independent of $\Q(s)$.   \Cref{lem:PP} shows that if $\Q$ satisfies \Cref{assumption:URVB}($\omega$), then $\Qtilde$  satisfies \Cref{assumption:URVB}($\omega_\proba$). 
\begin{lemma}\label{lem:PP}
Let $\Q$ be a compression operator satisfying \eqref{eq:Quant} with the constant $\omega$. Then $\Qtilde$ satisfies \eqref{eq:Quant}, with a constant $\omega_\proba$ given by 
\begin{equation} \label{eq:adjusted_quant_omega}
    \omega_\proba \eqdef \omega + \frac{1-\proba}{\proba}(\omega+1).
\end{equation}
\end{lemma}
\begin{proof} Let $s \in \rset^q$. It holds:
\[\PE\left[\Qtilde(s)\right]  =  \frac{1}{\proba} \PE\left[U \right] \,  \PE\left[ \Q(s)\right] = \PE\left[ \Q(s)\right]
\]
where we used that $U$ and $\Q$ are independent, and the expectation of $U$ is $\proba$. Since $U$ and $\Q(s)$ are independent, it holds
\[
\PE\left[ \left( \frac{U}{\proba}-1 \right) \pscal{\Q(s)}{\Q(s) -s} \right] = 0.
\]
Therefore
\[
  \PE \left[\|\Qtilde(s)-s\|^2 \right] =   \PE \left[\left( \frac{U}{\proba}-1 \right)^2 \, \left\|\Q(s)\right\|^2 \right] + \PE  \left[\|\Q(s) -s\|^2 \right]. 
\]
Using again that $U$ and $\Q(s)$ are independent and $\PE\left[\Q(s) \right]= s$, the first term is equal to 
\[
\PE\left[\left( \frac{U}{\proba}-1 \right)^2 \right] \PE\left[\left\|\Q(s)\right\|^2 \right] = \frac{1-\proba}{\proba}  \, \left( \PE\left[\left\|\Q(s) -s\right\|^2 \right] + \|s\|^2 \right).
\]
By using \eqref{eq:Quant}, we have
\[
 \PE \left[\|\Qtilde(s)-s\|^2 \right] \leq \frac{1-\proba}{\proba} (\omega +1) \|s\|^2 + \omega \|s\|^2.
\]
Therefore, $\Qtilde$ satisfies \eqref{eq:Quant}  with the constant $\omega_\proba$. 
\end{proof}

As a corollary,  when clients are selected independently at random with probability $\proba$ (see \Cref{assumption:PP}($\proba$)),  \Cref{algo:FedSS-MM} is equivalent 
to the algorithm where all the workers are active at all steps and the compression operator $\Q$ is swapped with $\Qtilde$; in words, $\Qtilde$ encodes both the compression step and the partial participation step.  Such an equivalence follows  directly  upon noting that  \textit{(i)} identifying a set of active agents $\set_{t+1}$ whose size is $n \proba$ in expectation, is realized by sampling $n$ independent Bernoulli random variables $(U_{t+1,1}, \cdots, U_{t+1,n})$, with success probability $\proba$,  \textit{(ii)}  it holds
\[
\frac{1}{\proba} \sum_{i \in \set_{t+1}} a_i = \sum_{i=1}^n \frac{U_{t+1,i}}{\proba} a_i,
\]

We set $\Qtilde_{t+1,i}(s) \eqdef  \proba^{-1} U_{t+1,i} \, \Q_{t+1,i}(s)$ and move from \Cref{algo:FedSS-MM} to \Cref{algo:FedSS-MM:fullP}.  From \Cref{lem:PP},  it is easily seen that the r.v. $\{\Qtilde_{t+1,i}(s), t \geq 0, i \in [n] \}$ satisfy \Cref{hyp:URVB}($\omega_\proba$).

 \begin{algorithm}[t]\DontPrintSemicolon
   \caption{\QSMM\ 
   (when Partial Participation is part of $\Q$)  } \KwData{$\kmax \in \nset_{>0}$;   $\alpha>0$; 
     a $\ocint{0,1}$-valued sequence
      $\{\pas_t, t \in [\kmax]\}$; 
   initial values $\hatS_0 \in \calS$ and  $V_{0,i} \in \rset^q$ for all $i \in [n]$; a probability distribution  $\{\mu_i, i \in [n]\}$  on $[n]$. \label{algo:FedSS-MM:fullP}} \KwResult{ Sequence: $\{\hatS_{t}, t \in [\kmax]\}$} Set $V_0 =
   \sum_{i=1}^n \mu_i V_{0,i}$ \; 
  \For(\tcp*[f]{for all iterations}){$t=0, \ldots, \kmax-1$}{
      Send $\hatS_{t}$
     and $\map (\hatS_{t})$ to the  workers \;  \For(\tcp*[f]{for all workers}){$i=1, \ldots, n$}{ Sample
       $\Smem_{t+1,i}$,  oracle for $ \PE_{\pi_i}[
         \barS(\sample, \map(\hatS_t))]$ \;
     Set $\Delta_{t+1,i} = \Smem_{t+1,i} - \hatS_t -
       V_{t,i}$    \; 
      Set $V_{t+1,i} = V_{t,i} + \alpha
       \, \Qtilde_{t+1,i}(\Delta_{t+1,i})$   \; 
        Send
       $\Qtilde_{t+1,i}(\Delta_{t+1,i})$ to the central
       server\;}
       \tcp*{on the central server} Set
     $H_{t+1} = V_t +   \sum_{i =1}^n  \mu_i
     \Qtilde_{t+1,i}(\Delta_{t+1,i})$
 \; 
     Compute $\hatS_{t+ \nicefrac{1}{2}} = \hatS_t + \pas_{t+1}H_{t+1}$ \;
     Project it on $\calS$, given $\B_t \succ 0 $: $\hatS_{t+1}  = \argmin_{s\in \surspace}  \left(s-\hatS_{t+ \nicefrac{1}{2}}\right)^\top \B_t \left(s-\hatS_{t+ \nicefrac{1}{2}}\right)$  \;
     Set $V_{t+1} = V_t +
     \alpha \sum_{i=1}^n \mu_i
    \Qtilde_{t+1,i}(\Delta_{t+1,i})$ 
    \; }
 \end{algorithm}

\subsection{Filtrations}
\label{sec:proof:filtration}
For a random variable $U$, we denote by $\sigma(U)$ the sigma-field
generated by $U$ i.e. the smallest sigma-field that makes  $U$ measurable.  Set
\[
\F_0 \eqdef \sigma\left(\hatS_0;  V_{0,i} \ \text{for} \ i \in \ccint{n}\right).
\]
We define by induction the following filtration, which takes into
account the randomness introduced by the algorithm at different stages
 and possibly at each local worker $\# i$:
\begin{align*}
 \F_{t+\nicefrac{1}{3},i} & \eqdef \F_{t} \bigvee \sigma\left(
  \Smem_{t+1,i} \right) \\ \F_{t+\nicefrac{2}{3},i} & \eqdef \F_{t+\nicefrac{1}{3},i}\bigvee
  \sigma\left( \Qtilde(\Delta_{t+1,i}) \right), \\
  \F_{t+1} & \eqdef \bigvee_{i=1}^n    \F_{t+\nicefrac{2}{3},i}.
\end{align*}
$\F_{t+\nicefrac{1}{3},i}$ contains the past of the algorithm up to iteration $\#
t$ included, and the randomness introduced by the oracles
$\Smem_{t+1,i}$.  $\F_{t+\nicefrac{2}{3},i}$ adds to $\F_{t+\nicefrac{1}{3},i}$ the randomness
introduced by the compression mechanism.  $\F_{t+1}$ contains
all the randomness introduced by the $n$ local workers until their
iteration $\#(t+1)$ included. We have for all $i \in
\ccint{n}$ and $t \geq 0$, $\F_{t} \subseteq \F_{t+\nicefrac{1}{3},i}  \subseteq \F_{t+\nicefrac{2}{3},i} \subseteq \F_{t+1}$ 
and
\begin{align*}
 & 
  \Smem_{t+1,i} \in \F_{t+\nicefrac{1}{3},i},  \quad 
  \Delta_{t+1,i} \in \F_{t+\nicefrac{1}{3},i}, \quad \Qtilde(\Delta_{t+1,i}) \in
  \F_{t+\nicefrac{2}{3},i},  \quad   V_{t+1,i} \in
  \F_{t+\nicefrac{2}{3},i} \\
 & H_{t+1} \in \F_{t+1}, \quad  \hatS_{t+1} \in \F_{t+1}, \quad V_{t+1} \in \F_{t+1}. 
\end{align*}
Furthermore, note that (see \Cref{lem:PP})
\begin{equation}
 \CPE{\Qtilde(\Delta_{t+1,i})}{\F_{t+\nicefrac{1}{3},i}}
= \Delta_{t+1,i}, \label{eq:condexp:Q}
  \end{equation}
and under \Cref{hyp:oracle},
\[
\CPE{ \Smem_{t+1,i}}{\F_{t}} = \PE_{\pi_i}\left[ \barS(\sample, \map(\hatS_t))\right].
\]

\subsection{Notations} \label{sec:proof:notations}
Define the functions 
\[ \mf(s) \eqdef \PE_{\pi} \left[
   \barS(\sample, \map(s)) \right] - s = \sum_{i=1}^n \mu_i \PE_{\pi_i} \left[
   \barS(\sample, \map(s)) \right] - s, \qquad  \mf_i(s) \eqdef \PE_{\pi_i} \left[ \barS(\sample, \map(s)) \right] -
 s. 
\]
Set 
\[
\bar C \eqdef   C_{\star \star} + \sup_{\calS} \| \mf(\cdot) \| \, 
\, C_\star, \quad \omega_\proba \eqdef \omega + \frac{1-\proba}{\proba} (\omega+1), \quad \sigma^2 \eqdef n \sum_{i=1}^n \mu_i^2 \sigma_i^2,  \quad  L^2 \eqdef n \sum_{i=1}^n \mu_i^2 L_i^2.
\]
For any $t \geq 1$, define 
\[
\barerror_{t+1}  \eqdef \frac{\| \Pi(\hatS_{t} + \pas_{t+1} \mf(\hatS_t), \B_t) - \hatS_t \|^2_{\B_t}}{\pas^2_{t+1}};
\]
note that when $\Pi(\hatS_t + \pas_{t+1} \mf( \hatS_t), \B_t) =\hatS_t + \pas_{t+1} \mf( \hatS_t)$, then $\barerror_{t+1} = \| \mf(\hatS_t) \|^2_{\B_t}$. 

\subsection{Proof of \Cref{theo:main}} \label{sec:proof:detailedprooftheo}
We start by establishing \Cref{prop:lyapunov}, which provides a Lyapunov inequality for the function $(f+g) \circ \map$, for any timestep $t\geq 0$.

\begin{proposition}\label{prop:lyapunov} Assume {\tt MM-1}, {\tt MM-2} and \Cref{hyp:DL2}.    For any $t  \geq 0$, choose $\pas_{t+1} \in \ooint{0, v_{\min}/(2 \bar C)}$. 
With probability $1$, it holds for any $t \geq 0$
\begin{align*}
\ce{f(\map(\hatS_{t+1})) + g(\map(\hatS_{t+1}))}{\F_t} & \leq
  f(\map(\hatS_{t})) + g(\map(\hatS_{t})) \\
  & - \frac{\pas_{t+1}}{4} \left( 1  -
 2 \bar C  v_{\min}^{-1} \pas_{t+1} \right)  \CPE{\barerror_{t+1}}{\F_t}\\
  &+ \pas_{t+1}  \left( 1 -\bar C  v_{\min}^{-1} \pas_{t+1}  \right) \CPE{\|H_{t+1} -
  \mf(\hatS_t)\|^2_{\B_t}}{\F_t}.
\end{align*}
$\bar C$, $\barerror_{t+1}$ and $\mf$ are defined in \Cref{sec:proof:notations}. 
\end{proposition}
\begin{proof}
 We start with proving that  for any $s, s' \in \calS$,  it holds 
\begin{equation}\label{eq:coro:lipschitz}
    (f+g)\circ \map(s')  \leq (f+g)\circ \map(s)- \pscal{\mf(s)}{\phi(\map(s')) -
      \phi(\map(s))}  + \pscal{s'-s}{\phi(\map(s'))-\phi(\map(s))
      }. 
\end{equation}
 Let $s,s' \in \calS$.  From {\tt MM-1}, we have
  \[
  f \left( \map(s') \right) \leq f \left(\map (s)\right) +
  \psi(\map(s')) - \psi(\map(s)) - \pscal{\PE_\pi \left[
      \barS(\sample, \map(s))\right]}{\phi(\map(s')) - \phi(\map(s))}.
  \]
  This yields
  \begin{align*}
    f \left( \map(s') \right) & \leq f \left(\map (s)\right) -
    \pscal{\mf(s)}{\phi(\map(s')) - \phi(\map(s))} \\ & + \psi(\map(s'))
    - \psi(\map(s)) - \pscal{s}{\phi(\map(s')) - \phi(\map(s))}.
    \end{align*}
  Therefore 
   \begin{equation}
     (f+g)\circ \map(s')   \leq (f+g) \circ \map(s) - \pscal{\mf(s)}{\phi(\map(s')) -
             \phi(\map(s))}  + \Xi(s,s') \label{eq:proof:eq1}
    \end{equation}
    where $ \Xi(s,s') \eqdef  g(\map(s')) - g(\map(s))
+\psi(\map(s')) - \psi(\map(s)) - \pscal{s}{\phi(\map(s')) -
\phi(\map(s))}$.  Let us prove that $\Xi(s,s') \leq \pscal{s'-s}{\phi(\map(s'))-\phi(\map(s))
      }$. We write
   \[
\Xi(s,s')= \mathcal{L}(s',s') +
\pscal{s'-s}{\phi(\map(s'))} - \mathcal{L}(s,s), \quad  \mathcal{L}(s,s') \eqdef g(\map(s)) + \psi(\map(s)) -
\pscal{s'}{\phi(\map(s))}.
\]
The upper bound on $\Xi(s,s')$ follows from \Cref{lem:monotony}.  This concludes the proof of \eqref{eq:coro:lipschitz}; applied with $s' = \hatS_{t+1}$ and $s = \hatS_t$, this yields
    \begin{align*}f(\map(\hatS_{t+1}))+g(\map(\hatS_{t+1})) &\leq f(\map(\hatS_{t})) +g(\map(\hatS_{t})) 
         - \pscal{\mf(\hatS_{t})}{\phi(\map(\hatS_{t+1})) - \phi(\map(\hatS_{t}))}  \\
        & + \pscal{\hatS_{t+1}-\hatS_t}{\phi(\map(\hatS_{t+1}))-\phi(\map(\hatS_t))}.
    \end{align*}
 By \Cref{hyp:DL2}, we have 
 \[
      \pscal{\hatS_{t+1}-\hatS_t}{\phi(\map(\hatS_{t+1}))-\phi(\map(\hatS_t))} \leq  C_{\star \star} \| \hatS_{t+1}-\hatS_t \|^2  \leq    v_{\min}^{-1} \, C_{\star \star} \, \pas_{t+1}^2 \error_{t+1},
 \]
 where we set $\error_{t+1} \eqdef \pas_{t+1}^{-2} \| \hatS_{t+1} - \hatS_t \|^2_{\B_t}$. 
We write 
\begin{align}
    -  \pscal{\mf(\hatS_{t})}{\phi(\map(\hatS_{t+1})) - \phi(\map(\hatS_{t}))} & \leq  -\pscal{H_{t+1}}{
        \hatS_{t+1}-\hatS_{t}
        }_{\B_t} \label{prop:lyap:tool1} \\
        & \hspace{-0.5cm} + \left| \pscal{H_{t+1} - \mf(\hatS_t)}{
        \hatS_{t+1}-\hatS_{t}
        }_{\B_t}\right| \label{prop:lyap:tool2} \\
        & \hspace{-0.5cm}+ \left| \pscal{\mf(\hatS_t)}{\phi(\map(\hatS_{t+1}))-\phi(\map(\hatS_{t})) - \B_t (\hatS_{t+1}-\hatS_t))} \right|. \label{prop:lyap:tool3} 
        \end{align}
By \Cref{coro:monotony}, an upper bound for \eqref{prop:lyap:tool1} is $- \pas_{t+1}\error_{t+1}$. By using $|\pscal{a}{b}_{\B_t} |\leq \pas_{t+1} \|a\|_{\B_t}^2/2 + \|b\|^2_{\B_t} / (2 \pas_{t+1})$ and $\| \hatS_{t+1} - \hatS_t\|_{\B_t} = \pas_{t+1}^2 \error_{t+1}$, an upper bound for \eqref{prop:lyap:tool2} is
\[
 \frac{\pas_{t+1}}{2} \| H_{t+1} - \mf(\hatS_t) \|^2_{\B_t} + \frac{\pas_{t+1}}{2}  \error_{t+1}.
\]
Under \Cref{hyp:DL2}, an upper bound for \eqref{prop:lyap:tool3}  is
\[
C_\star \, \sup_{\calS} \| \mf(\cdot)\| \, \| \hatS_{t+1} - \hatS_t \|^2 = v_{\min}^{-1} \, C_\star \, \sup_{\calS} \| \mf(\cdot)\| \, \pas_{t+1}^2 \error_{t+1}
 \]
Apply \Cref{coro:monotony}  again to obtain a bound of $- \error_{t+1}$; then the proof is concluded by  application of the conditional expectation. 
\end{proof}

The field $H_{t+1}$ is a random approximation of $\mf(\hatS_t)$. \Cref{prop:update_error} shows that the oracle $H_{t+1}$ is unbiased, conditionally to the past $\F_t$. The proof of \Cref{prop:update_error} establishes that the error $H_{t+1} - \mf(\hatS_t)$ is the sum of two conditionally centered and uncorrelated terms: the first one is the mean value of the local oracles, and the second one is the mean value of the local  compression operators. Therefore, two terms contribute to the conditional variance of the oracle $H_{t+1}$: the first one is the
weighted average of the conditional variance of the local oracles, and
the second one is the variance of the weighted compression errors.  The control variate $V_{t,i}$ plays a role in this second term.

\begin{proposition}\label{prop:update_error}   Assume  \Cref{hyp:oracle},  \Cref{hyp:URVB}($\omega$) and \Cref{assumption:PP}($\proba$).
 With probability one, it holds for any $t \geq 0$ 
 \begin{align} 
& 	\CPE{H_{t+1} }{\F_t}= \mf(\hatS_t),  \label{eq:unbiasedoracle} \\
& \CPE {\| H_{t+1} - \mf(\hatS_t) \|^2_{B_t}}{\F_t}  \le  (1+ \omega_\proba) v_{\max} \, \frac{\sigma^2}{n} + \omega_\proba v_{\max} \, \sum_{i=1}^n \mu_i^2 \| V_{t,i} - \mf_i(\hatS_t) \|^2. \label{eq:varianceoracle}
	\end{align}
  $\mf$, $\mf_i$,  $\omega_\proba$ and $\sigma^2$ are defined in \Cref{sec:proof:notations}. 
\end{proposition}
\begin{proof}
 Let $t \geq 0$. We have $H_{t+1} = V_t + \sum_{i=1}^n \mu_i
\Qtilde_{t+1,i}(\Delta_{t+1,i})$, where, by definition of the filtration,  $V_{t,i} \in \F_{t-1+\nicefrac{2}{3},i} \subset \F_{t}$,  $\hatS_t \in \F_t$, $\Delta_{t+1,i} \in \F_{t+\nicefrac{1}{3},i}$, 
$\Qtilde_{t+1,i}(\Delta_{t+1,i}) \in \F_{t+\nicefrac{2}{3},i}$ and $V_t \in \F_t$. Therefore,
\begin{align*}
	 \CPE{H_{t+1} }{ \F_t} &  = V_t +
         \sum_{i=1}^n \mu_i \CPE{\CPE{
             \Qtilde_{t+1,i}(\Delta_{t+1,i})}{\F_{t+\nicefrac{1}{3},i}}}{\F_t} = V_t +
         \sum_{i=1}^n \mu_i \CPE{ \Delta_{t+1,i}}{\F_t}
         \nonumber \\ & = V_t + \sum_{i=1}^n \mu_i \left(
         \CPE{\Smem_{t+1,i}}{\F_t} - \hatS_t - V_{t,i} \right) =
         \sum_{i=1}^n \mu_i \PE_{\pi_i} \left[ \barS(\sample,
           \map(\hatS_t)) \right] - \hatS_t = \mf(\hatS_t), \nonumber
	\end{align*}
where we used \Cref{hyp:oracle}, 
\Cref{hyp:URVB}($\omega$), \Cref{lem:PP}  and \Cref{prop:meanV:PP}.  This concludes the proof
of \eqref{eq:unbiasedoracle}.   We write, by using \Cref{prop:meanV:PP},
\begin{equation*}
  H_{t+1} - \mf(\hatS_t) = V_t + \sum_{i=1}^n \mu_i \{ 
  \Qtilde_{t+1,i}(\Delta_{t+1,i}) - \mf_i(\hatS_t) \} = \sum_{i=1}^n \mu_i \{
  \Xi_{t+1,i}^{(1)} + \Xi_{t+1,i}^{(2)} \}, 
  \end{equation*}
  with
  \begin{equation*}
 \Xi_{t+1,i}^{(1)} 
  \eqdef \Smem_{t+1,i} - \PE_{\pi_i}\left[ \barS(\sample, \map(\hatS_t)) \right],   \qquad  \Xi_{t+1,i}^{(2)} 
  \eqdef  \Qtilde_{t+1,i}(\Delta_{t+1,i}) -
  \Delta_{t+1,i};
\end{equation*}
By  \Cref{hyp:oracle}, \Cref{hyp:URVB}($\omega$)  and \Cref{lem:PP},
\begin{equation} \label{eq:centeredXi}
\Xi_{t+1,i}^{(1)} \in \F_{t+\nicefrac{1}{3},i}, \qquad \CPE{\Xi_{t+1,i}^{(1)} }{\F_t} =0, \qquad \CPE{\Xi_{t+1,i}^{(2)}}{\F_{t+\nicefrac{1}{3},i}}=0. 
\end{equation}
Since the workers are independent conditionally
to $\F_t$ (see \Cref{assumption:PP}($\proba$)) and $\B_t \in \F_t$, we have
\[
\CPE{\| H_{t+1} - \mf(\hatS_t)\|^2_{\B_t}}{\F_t} = \sum_{i=1}^n
\mu_i^2\CPE{\| \Xi_{t+1,i}^{(1)} + \Xi_{t+1,i}^{(2)} \|^2_{\B_t}}{\F_t} =   \sum_{i=1}^n
\mu_i^2 \sum_{\ell=1}^2 \CPE{\| \Xi_{t+1,i}^{(\ell)}\|^2_{\B_t}}{\F_t}
\]
where we used that    for all $i \in [n]$,   $\CPE{\pscal{\Xi_{t+1,i}^{(1)}}{\Xi_{t+1,i}^{(2)}}_{\B_t}}{\F_t} = 0$ by \eqref{eq:centeredXi}. 
We have by \Cref{hyp:oracle} 
\[
\CPE{\| \Xi_{t+1,i}^{(1)}\|^2_{\B_t}}{\F_t} =
\CPE{\| \Smem_{t+1,i} - \PE_{\pi_i} \left[ \barS(\sample,
    \map(\hatS_t)) \right] \|^2_{\B_t}}{\F_t} \leq v_{\max} \, \sigma_i^2 ;
\]
and by \Cref{lem:PP},
\[
\CPE{\| \Xi_{t+1,i}^{(2)}\|^2_{\B_t}}{\F_t} \leq    v_{\max} \CPE{\| \Xi_{t+1,i}^{(2)}\|^2}{\F_t}    \leq \omega_\proba \, v_{\max} \, 
\CPE{\|\Delta_{t+1,i}\|^2}{\F_t}.
\]
As a conclusion, we proved
\begin{equation}\label{eq:decomposition:globalvariance}
\CPE{\| H_{t+1} - \mf(\hatS_t)\|^2_{\B_t}}{\F_t} \leq v_{\max} \,  \sum_{i=1}^n
\mu_i^2 \sigma_i^2  +  v_{\max} \,   \omega_\proba \, \sum_{i=1}^n
\mu_i^2\CPE{\|\Delta_{t+1,i}\|^2}{\F_t}.
\end{equation}
Let $i \in [n]$.  By \Cref{hyp:oracle},
  $\CPE{\Smem_{t+1,i} }{\F_{t}} = \mf_i(\hatS_t)+ \hatS_t$. In
  addition, $\hatS_t \in \F_t$ and $V_{t,i} \in \F_{t}$, thus implying that
  \[
\CPE{\pscal{\mf_i(\hatS_t)- V_{t,i}}{\Smem_{t+1,i} - \hatS_t
    -\mf_i(\hatS_t)}}{\F_t} = 0.
\]
Hence, we get by \Cref{hyp:oracle} again,
\begin{align*}
  \nonumber \CPE{\|\Delta_{t+1,i}\|^2}{\F_{t}} &= \CPE{\|\Smem_{t+1,i}
     - \hatS_t- V_{t,i}\|^2}{\F_{t}} \\ &= \|\mf_i(\hatS_t)-
  V_{t,i}\|^2 +\CPE{\| \Smem_{t+1,i} - \hatS_t -\mf_i(\hatS_t)
    \|^2}{\F_{t}} \nonumber \\ & = \|\mf_i(\hatS_t)- V_{t,i}\|^2
  +\CPE{\| \Smem_{t+1,i} - \PE_{\pi_i}\left[\barS(\sample,
      \map(\hatS_t)) \right] \|^2}{\F_{t}} \nonumber \\ &\leq \|
  \mf_i(\hatS_t)- V_{t,i}\|^2 +\sigma_i^2 .
\end{align*}
This concludes the proof of \eqref{eq:varianceoracle}. 
\end{proof}

\Cref{prop:varcont-PP2} shows that the sequence of random variables $\{\sum_{i=1}^n \mu_i^2 \| V_{t,i} - \mf_i(\hatS_t)
        \|^2, t \geq 0\}$ satisfies a Lyapunov inequality. The important property is that for small enough stepsizes and for convenient values of $\alpha \in \ocint{0,1}$ depending on the coefficient $\omega_\proba$, it is of the form
        \[
       \CPE{\sum_{i=1}^n \mu_i^2 \| V_{t+1,i} - \mf_i(\hatS_{t+1})
        \|^2}{\F_t} \leq \lambda \sum_{i=1}^n \mu_i^2 \| V_{t,i} - \mf_i(\hatS_t)
        \|^2 + \Upsilon_{t}
        \]
where $\lambda \in \ooint{0,1}$ and $\Upsilon_{t} \geq 0$.  Such an inequality, combined with \Cref{prop:lyapunov} will be crucial for the proof of \Cref{theo:main}.

  \begin{proposition}	\label{prop:varcont-PP2}
	Assume 
        \Cref{hyp:oracle}, \Cref{hyp:URVB}($\omega$), \Cref{hyp:lipschitz} and   \Cref{hyp:DL2}. Define for all $t \geq 0$,
	\[
	G_t \eqdef \sum_{i=1}^n \mu_i^2 \| V_{t,i} - \mf_i(\hatS_t)
        \|^2. 
	\]
    Choose $\alpha \in \ccint{0,1/(1+\omega_\proba)}$ and $ \beta > 1$.
         With probability
        one, for any $t \geq 0$ it holds
        \begin{multline}
	\CPE{G_{t+1}}{\F_t} \le  \beta \, \left( 1 -\alpha+   \frac{2 \omega_\proba  }{\beta-1}  \frac{v_{\max}}{v_{\min}}  \frac{L^2}{n}
 \pas_{t+1}^2 \right)
 G_t +   \frac{2\beta}{\beta-1}  v_{\min}^{-1}  \frac{L^2}{n}
 \pas_{t+1}^2   \CPE{\barerror_{t+1}}{\F_t} \\
  +  \beta \left( \alpha + \frac{ 2 (1+\omega_\proba)}{\beta-1}  \frac{v_{\max}}{v_{\min}}  \frac{L^2}{n}
 \pas_{t+1}^2  \right) \frac{\sigma^2}{n}. \label{eq:induction:Gt}
  \end{multline}
  $\omega_\proba$, $\mf_i$, $L^2$, $\sigma^2$ and $\barerror_{t+1}$ are defined in \Cref{sec:proof:notations}. 
\end{proposition}
\begin{proof}
  Let $t \geq 0$ and $i \in [n]$.  For any $\beta>1$, using that $\|a+b\|^2
  \leq \nicefrac{\beta}{(\beta-1)} \|a\|^2 + \beta\|b\|^2$, we have 
	\begin{align*}
	\|V_{t+1,i} - \mf_i(\hatS_{t+1}) \|^2  &
          \le \beta \|V_{t+1,i} - \mf_i(\hatS_{t}) \|^2
          + \nicefrac{\beta}{(\beta-1)} \|\mf_i(\hatS_{t+1}) -
            \mf_i(\hatS_{t}) \|^2  \\ &\leq \beta
          \|V_{t+1,i} - \mf_i(\hatS_{t}) \|^2  +
          \nicefrac{\beta}{(\beta-1)}  v_{\min}^{-1} L_i^2 \pas_{t+1}^2  \error_{t+1},
	\end{align*}
    where $\error_{t+1} \eqdef \| \hatS_{t+1} - \hatS_t\|^2_{\B_t} / \pas_{t+1}^2$.
        We used \Cref{hyp:lipschitz} in the last inequality.  By
        \Cref{lem:controlvariate} and the condition $\alpha(1+\omega_\proba) - 1 \leq 0$, we have 
       	\begin{align*}
	 \CPE{ \|V_{t+1,i} - \mf_i(\hatS_{t+1}) \|^2 }{ \F_{t} } &
          \le \beta (1-\alpha) \| V_{t,i}  -
          \mf_i(\hatS_t) \|^2  \\
          &+ \beta \alpha
          \sigma_i^2 + \frac{\beta}{(\beta-1)} v_{\min}^{-1}  L_i^2 \pas_{t+1}^2
         \CPE{\error_{t+1} }{ \F_{t} }.
	\end{align*}
We now multiply by $\mu_i^2$ and sum in $i$; this yields
\begin{align*}
 \CPE{G_{t+1}}{ \F_{t} } & {\le} \beta (1- \alpha)\,
 G_t +  \beta \alpha \frac{\sigma^2}{n} + \frac{\beta}{\beta-1}  v_{\min}^{-1}  \frac{L^2}{n}
 \pas_{t+1}^2   \CPE{\error_{t+1}}{\F_t} .
\end{align*}
By using \Cref{coro:monotony}
\begin{align*}
 \CPE{G_{t+1}}{ \F_{t} } & {\le} \beta (1-\alpha)
 G_t +  \beta \alpha \frac{\sigma^2}{n} +  \frac{2\beta}{\beta -1}  v_{\min}^{-1}  \frac{L^2}{n}
 \pas_{t+1}^2   \CPE{\barerror_{t+1}}{\F_t} \\
 & + \frac{2 \beta}{\beta-1}  v_{\min}^{-1}  \frac{L^2}{n}
 \pas_{t+1}^2\CPE{\|H_{t+1} - \mf(\hatS_t) \|^2_{\B_t}}{\F_t}.
\end{align*}
 The proof  of \eqref{eq:induction:Gt} is concluded by \Cref{prop:update_error}.
\end{proof}

\paragraph{Conclusion: Proof of \Cref{theo:main} in the case $\omega_\proba=0$.} Let $\alpha \in \ccint{0,1}$.  From \Cref{prop:lyapunov} and \Cref{prop:update_error}, it holds
\begin{multline*}
\ce{f(\map(\hatS_{t+1})) + g(\map(\hatS_{t+1}))}{\F_t}  \leq
  f(\map(\hatS_{t})) + g(\map(\hatS_{t})) \\
   - \frac{\pas_{t+1}}{4} \left( 1  -
 2 \bar C  v_{\min}^{-1} \pas_{t+1} \right)  \CPE{\barerror_{t+1}}{\F_t} + \pas_{t+1}  \left( 1 -\bar C  v_{\min}^{-1} \pas_{t+1}  \right)  v_{\max} \frac{\sigma^2}{n}.
\end{multline*}
By summing from $t=0$ to $t=T-1$ and computing the expectation, we obtain the result.

\paragraph{Proof of \Cref{theo:main} in the case  $\omega_\proba>0$.}
 From \Cref{prop:lyapunov}, \Cref{prop:update_error} and \Cref{prop:varcont-PP2}, we have for any positive sequence $\{\rho_{t+1}, t \geq 0\}$
\begin{align*}
& \ce{f(\map(\hatS_{t+1})) + g(\map(\hatS_{t+1}))}{\F_t} + \rho_{t+1} \CPE{G_{t+1}}{\F_t}
 \leq
  f(\map(\hatS_{t})) + g(\map(\hatS_{t}))  + \rho_t G_t  \\
  & +  \pas_{t+1} \left( \frac{\rho_{t+1}}{\pas_{t+1}}  \beta \, \left( 1 -\alpha+   \frac{2 \omega_\proba  }{\beta-1}  \frac{v_{\max}}{v_{\min}}  \frac{L^2}{n}
 \pas_{t+1}^2 \right) +   \left( 1 -\bar C  v_{\min}^{-1} \pas_{t+1}  \right) \omega_\proba v_{\max}  - \frac{\rho_t}{\pas_{t+1}} \right) G_t  \\
  & -    \frac{\pas_{t+1}}{4} \left(  1  -
 2 \bar C  v_{\min}^{-1} \pas_{t+1}   -  \frac{8 \beta}{\beta-1}  v_{\min}^{-1}  \frac{L^2}{n}
 \rho_{t+1} \pas_{t+1}  \right)  \CPE{\barerror_{t+1}}{\F_t}\\
  &+  \pas_{t+1} \left( \frac{\rho_{t+1}}{\pas_{t+1}}\beta \left( \alpha + \frac{ 2 (1+\omega_\proba)}{\beta-1}  \frac{v_{\max}}{v_{\min}}  \frac{L^2}{n} 
 \pas_{t+1}^2 \right)  +   \left( 1 -\bar C  v_{\min}^{-1} \pas_{t+1}  \right) (1+\omega_\proba) v_{\max}  \right) \frac{\sigma^2}{n}
\end{align*}
Let $\alpha \in \ocint{0, 1/(1+\omega_\proba)}$. Choose $\beta>1$ and $\rho_{t+1} = \rho \pas_{t+1}$. Then, 
\begin{align*}
& \ce{f(\map(\hatS_{t+1})) + g(\map(\hatS_{t+1}))}{\F_t} + \rho \pas_{t+1} \CPE{G_{t+1}}{\F_t}
 \leq
  f(\map(\hatS_{t})) + g(\map(\hatS_{t}))  + \rho \pas_t G_t  \\
  & +  \pas_{t+1} \left( \rho  \beta \, \left( 1 -\alpha+   \frac{2 \omega_\proba  }{\beta-1}  \frac{v_{\max}}{v_{\min}}  \frac{L^2}{n}
 \pas_{t+1}^2 \right) +   \left( 1 -\bar C  v_{\min}^{-1} \pas_{t+1}  \right) \omega_\proba v_{\max}  - \rho \frac{\pas_t}{\pas_{t+1}} \right) G_t  \\
  & -    \frac{\pas_{t+1}}{4} \left(  1  -
 2 \bar C  v_{\min}^{-1} \pas_{t+1}   -  \frac{8 \rho \beta}{\beta-1}  v_{\min}^{-1}  \frac{L^2}{n}
 \pas_{t+1}^2  \right)  \CPE{\barerror_{t+1}}{\F_t}\\
  &+  \pas_{t+1} \left( \rho \beta \left( \alpha + \frac{ 2 (1+\omega_\proba)}{\beta-1}  \frac{v_{\max}}{v_{\min}}  \frac{L^2}{n} 
 \pas_{t+1}^2 \right)  +   \left( 1 -\bar C  v_{\min}^{-1} \pas_{t+1}  \right) (1+\omega_\proba) v_{\max}  \right) \frac{\sigma^2}{n}
\end{align*}
Now, choose $\beta$ and $\rho$  such that  
\[
1 < \beta \leq 2, \qquad \beta (1-\alpha) \leq 1-\alpha/2, \qquad \beta/( \beta-1) \leq 2/\alpha, \qquad  \rho \eqdef 4 \omega_\proba v_{\max}/\alpha;
\] for example, $\beta=2$ if $\alpha \geq 2/3$ and $\beta (1-\alpha) = 1-\alpha/2$ otherwise.   This yields, by using $\pas_{t+1} \leq  \pas_t \leq \pas_1$,
\begin{align*}
& \rho  \beta \, \left( 1 -\alpha+   \frac{2 \omega_\proba  }{\beta-1}  \frac{v_{\max}}{v_{\min}}  \frac{L^2}{n}
 \pas_{t+1}^2 \right) \leq  \rho \left(1-  \frac{\alpha}{2} \right) + \rho 4 \frac{\omega_\proba}{\alpha} \frac{v_{\max}}{v_{\min}}  \frac{L^2}{n}
 \pas_{1}^2 \leq \rho - \rho \frac{\alpha}{4}  ;  \\
 & \left( 1 -\bar C  v_{\min}^{-1} \pas_{t+1}  \right) \omega_\proba v_{\max}  - \rho \frac{\pas_t}{\pas_{t+1}} \leq \omega_\proba v_{\max} - \rho = \rho \frac{\alpha}{4} - \rho; \\
 & 8 \rho \frac{\beta}{\beta-1} \leq \frac{64}{\alpha^2} \omega_\proba v_{\max}; \\
 & \rho \beta \left( \alpha + \frac{ 2 (1+\omega_\proba)}{\beta-1}  \frac{v_{\max}}{v_{\min}}  \frac{L^2}{n} 
 \pas_{t+1}^2 \right) \leq 4 \omega_\proba v_{\max} \left(2 \wedge \frac{1- \alpha/2}{1-\alpha} \right)+ (1+\omega_\proba) v_{\max}.
\end{align*}
Hence,
\begin{align*}
& \ce{f(\map(\hatS_{t+1}))  + g(\map(\hatS_{t+1}))}{\F_t} + \frac{4 \omega_\proba v_{\max}}{\alpha} \pas_{t+1} \CPE{G_{t+1}}{\F_t} \\
& \qquad \qquad \leq
  f(\map(\hatS_{t})) + g(\map(\hatS_{t}))  + \frac{4 \omega_\proba v_{\max}}{\alpha} \pas_t G_t  \\
  & \qquad  \qquad -    \frac{\pas_{t+1}}{4} \left(  1  -
 2 \bar C  v_{\min}^{-1} \pas_{t+1}   -  64 \frac{\omega_\proba}{\alpha^2} \frac{v_{\max}}{ v_{\min}}  \frac{L^2}{n}
 \pas_{t+1}^2  \right)  \CPE{\barerror_{t+1}}{\F_t}\\
  & \qquad \qquad +  2\pas_{t+1}  v_{\max}  \left(  2 \omega_\proba \left(2 \wedge \frac{1-\alpha/2}{1-\alpha} \right)+    (1+\omega_\proba)   \right) \frac{\sigma^2}{n}.
\end{align*}
Observe that 
\[
2 \pas_{t+1}  v_{\max}  \left(  2 \omega_\proba \left(2 \wedge \frac{1-\alpha/2}{1-\alpha} \right)+   (1+\omega_\proba)   \right) \frac{\sigma^2}{n} \leq 2\pas_{t+1}  v_{\max}   (1+ 5 \omega_\proba)   \frac{\sigma^2}{n}.
\]
By summing from $t=0$ to $t=T-1$ (with the convention that $\pas_0 = \pas_1$) and computing the expectations, we obtain the result.

\subsection{Technical lemmas}\label{sec:prelimlemma}
\begin{lemma} \label{lem:monotony} Assume {\tt MM-1} and {\tt MM-2}.
  For $s,s' \in \calS$, set $ \mathcal{L}(s,s') \eqdef g(\map(s)) + \psi(\map(s)) -
\pscal{s'}{\phi(\map(s))}$. Then, for any $s,s' \in \calS$, it holds
  \begin{eqnarray}\
  & \mathcal{L}(s',s') \leq \mathcal{L}(s,s)
 -   \pscal{s'-s}{\phi(\map(s))}, \label{eq:theo:tool1}   \\ &
    \pscal{s'-s}{\phi(\map(s'))-\phi(\map(s))} \geq 0. \label{eq:theo:tool1bis}
  \end{eqnarray}
\end{lemma}

\begin{proof}
Let $s,s' \in \calS$.  By definition of $\map(s')$ (see {\tt MM-2}),
we have for any $\param \in \paramset$
\begin{align*}
\mathcal{L}(s',s') = g(\map(s')) + \psi(\map(s')) -
\pscal{s'}{\phi(\map(s'))} &\leq g(\param) + \psi(\param) -
\pscal{s'}{\phi(\param)} \\ &  = g(\param) + \psi(\param) -
\pscal{s}{\phi(\param)} - \pscal{s'-s}{\phi(\param)}.
\end{align*}
Applied with $\param = \map(s) \in \paramset$, this yields
\eqref{eq:theo:tool1}. 
By symmetry, we have $\mathcal{L}(s,s) \leq \mathcal{L}(s',s')+
    \pscal{s'-s}{\phi(\map(s'))}$ which, added to \eqref{eq:theo:tool1}, yields \eqref{eq:theo:tool1bis}. 
\end{proof}

\Cref{lem:weightedproj} generalizes classical results about projection
on a closed convex set (see e.g. \cite[Lemma 2.11]{polyak:2021}) to the
case $\B$ is not the identity matrix.

\begin{lemma}\label{lem:weightedproj}
  Let $\mathcal{B}$ be a closed convex subset of $\rset^q$ and $\B$ be a $q \times q$ positive definite 
  matrix.  Set  $\pscal{a}{b}_\B \eqdef \pscal{\B a}{b}$, 
  $\|a\|_{\B} \eqdef \left( a^\top \B a \right)^{1/2}$ and 
\[
\Pi(\hat{s},\B) \eqdef \mathrm{argmin}_{s \in \mathcal{B}} (s-\hat{s})^\top \B (s- \hat{s}).
\]
Then for any  $\hat{s},s' \in \mathcal{B}$:  $\pscal{ \Pi(\hat{s},\B) - \hat{s}}{s' - \Pi(\hat{s},\B)}_\B \geq 0$. 
In addition,  for any $\hat{s}_1, \hat{s}_2 \in \mathcal{B}$: 
\[
\| \hat{s}_1 - \hat{s}_2\|_\B  \geq  \| \Pi(\hat{s}_1, \B) - \Pi(\hat{s}_2, \B)\|_\B. 
\] 
\end{lemma}
    \begin{proof} 
  Let $\hat{s} \in \mathcal{B}$.  A descent direction at $\Pi(\hat{s},\B)$ of the function $s'
  \mapsto (s' - \hat{s})^\top \B (s'-\hat{s})$ does not exist so that  
  \begin{equation} \label{eq:projection:caracterization}
\forall s' \in \mathcal{B}, \qquad \pscal{\Pi(\hat{s},\B) - \hat{s}}{s' - \Pi(\hat{s},\B)}_\B \geq 0.
\end{equation}
This concludes the proof of the first claim.  This property implies
\[
\pscal{\Pi(\hat{s}_2, \B) -  \hat{s}_2}{  \Pi(\hat{s}_1, \B) - \Pi(\hat{s}_2, \B)}_\B \geq 0, \quad \pscal{\Pi(\hat{s}_1, \B) -  \hat{s}_1}{  \Pi(\hat{s}_2, \B) - \Pi(\hat{s}_1, \B)}_\B\geq 0,
\]
from which we deduce
\[
\pscal{\hat{s}_1 - \hat{s}_2}{  \Pi(\hat{s}_1, \B) - \Pi(\hat{s}_2, \B)  }_\B  \geq  \| \Pi(\hat{s}_1, \B) - \Pi(\hat{s}_2, \B)\|^2_\B. 
\]
The Cauchy-Schwarz inequality concludes the proof. 
\end{proof}

\begin{corollary}[of \Cref{lem:weightedproj}] \label{coro:monotony}
For any $t \geq 0$, define
\[
\error_{t+1}  \eqdef \frac{\| \hatS_{t+1} - \hatS_t \|^2_{\B_t}}{\pas^2_{t+1}} = \frac{\| \Pi(\hatS_{t+ \nicefrac{1}{2}}, \B_t) - \hatS_t \|^2_{\B_t}}{\pas^2_{t+1}}.
\]
Then, for any $t\geq 0$, $\pas_{t+1}  \error_{t+1} \leq \pscal{  \hatS_{t+1} - \hatS_t  }{H_{t+1}}_{\B_t}$ and 
\begin{align*}
 \barerror_{t+1} \leq 2 \error_{t+1} + 2 \| H_{t+1} - \mf(\hatS_t) \|_{\B_t}, \qquad   \error_{t+1} \leq 2 \barerror_{t+1} + 2 \| H_{t+1} - \mf(\hatS_t) \|_{\B_t}.
\end{align*}
\end{corollary}
\begin{proof} Let $t \geq 0$.  We write
\begin{multline*}
    \pscal{ \Pi(\hatS_{t+ \nicefrac{1}{2}}, \B_t) - \hatS_t  }{\hatS_{t+\nicefrac{1}{2}} - \hatS_t}_{\B_t} = \pscal{ \Pi(\hatS_{t+ \nicefrac{1}{2}}, \B_t) - \hatS_t  }{\hatS_{t+\nicefrac{1}{2}} - \Pi(\hatS_{t+ \nicefrac{1}{2}}, \B_t)  }_{\B_t}  \\
    +  \| \Pi(\hatS_{t+ \nicefrac{1}{2}}    , \B_t) - \hatS_t \|_{\B_t}^2
\end{multline*}
and by \Cref{lem:monotony}, it holds 
\[
\pscal{ \Pi(\hatS_{t+ \nicefrac{1}{2}}, \B_t) - \hatS_t  }{\hatS_{t+\nicefrac{1}{2}} - \hatS_t}_{\B_t}  \geq \| \Pi(\hatS_{t+ \nicefrac{1}{2}}    , \B_t) - \hatS_t \|_{\B_t}^2 = \pas_{t+1}^2 \error_{t+1}, 
\]
from which we deduce the first inequality upon noting that $\Pi(\hatS_{t+\nicefrac{1}{2}}, \B_t) = \hatS_{t+1}$ and  $\hatS_{t+\nicefrac{1}{2}} - \hatS_t = \pas_{t+1} H_{t+1}$.   \Cref{lem:monotony} applied with $\hat{s}_1 \eqdef \hatS_t + \pas_{t+1} H_{t+1}$ and $\hat{s}_2 \eqdef \hatS_t + \pas_{t+1} \mf(\hatS_{t})$ yields
\[
\| \hatS_{t+1}  - \Pi(\hatS_t +\pas_{t+1} \mf(\hatS_t), \B_t)  \|_{\B_t} \leq  \pas_{t+1} \| H_{t+1} - \mf(\hatS_t) \|_{\B_t}.
\] The second  and third inequalities follow from  $\| a+b\|_{\B}^2 \leq 2 \|a\|^2_{\B} + 2 \| b \|_{\B}^2$.
\end{proof}

\Cref{prop:meanV:PP} proves that the control variate $V_t$ computed by the central server is equal to the mean value of the control variates computed locally by the workers.
\begin{proposition} \label{prop:meanV:PP} 
  With probability one, for any $t \geq 0$, we have $V_{t} = \sum_{i=1}^n \mu_i V_{t,i}$.
  \end{proposition}
\begin{proof}
  By definition of $V_0$, the property holds true when $t=0$.  Assume
  this holds true for some $t \geq 0$. By definition of $V_{t+1}$ at
  the central server in~\Cref{algo:FedSS-MM:fullP},  and  by the induction assumption we write
  \[
  V_{t+1} = \sum_{i=1}^n \mu_i V_{t,i} + \alpha \sum_{i =1}^n \mu_i \,
  \Qtilde_{t+1,i}(\Delta_{t+1,i}).
    \]
    The definition of $V_{t+1,i}$ at
    each local worker  yields
    \[
    V_{t+1} = \sum_{i=1}^n \mu_i V_{t,i} + \sum_{i =1}^n \mu_i \left(
    V_{t+1,i} - V_{t,i} \right) = \sum_{i=1}^n \mu_i V_{t+1,i} .
    \]
  This concludes the induction.
\end{proof}

\begin{lemma} \label{lem:controlvariate}
	Assume 
         \Cref{hyp:oracle} and \Cref{hyp:URVB}($\omega$). Choose $\alpha \in \ccint{0,1}$.  With probability one, for any $t \geq
        0$ and $i \in \ccint{n}$, it holds
        \begin{align*}
	\CPE{\| V_{t+1,i} - \mf_i(\hatS_t) \|^2}{\F_t} & \le
(1-\alpha) \| V_{t,i} - \mf_i(\hatS_t) \|^2 +\alpha
        \sigma_i^2 \\
        &+ \alpha \left( \alpha (1+\omega_\proba) -1\right) \CPE{ \|
  \Delta_{t+1,i}\|^2}{ \F_{t} }.
  \end{align*}
  $\omega_\proba$ and $\mf_i$ are defined in \Cref{sec:proof:notations}. 
\end{lemma}
\begin{proof}
  Let $t \geq 0$ and $i \in \ccint{n}$. We write, 
	\begin{align*}
           V_{t+1,i} - \mf_i(\hatS_t) & =  V_{t,i} -
\mf_i(\hatS_t) + \alpha \ \Qtilde_{t+1,i}(\Delta_{t+1,i})
        \end{align*}
        which yields \begin{align} &\PE\left[ \|V_{t+1,i} - \mf_i(\hatS_{t})
          \|^2 \vert \F_{t} \right]
        = \CPE{ \|V_{t,i} -
          \mf_i(\hatS_{t})+\alpha \, \Qtilde_{t+1,i}(\Delta_{t+1,i}) \|^2}{
          \F_{t} }. \label{eq:controlvariate:tool1}
        \end{align}
Upon noting that by \Cref{lem:PP}, $\CPE{\Qtilde_{t+1,i}(\Delta_{t+1,i})
          - \Delta_{t+1,i}}{\F_{t+\nicefrac{1}{3},i}} = 0$, and $V_{t,i} -
        \mf_i(\hatS_{t}) + \alpha \Delta_{t+1,i} \in \F_{t + \nicefrac{1}{3},i}$,    we first write 
        \begin{align}
& \CPE{ \|V_{t,i} - \mf_i(\hatS_{t})+\alpha \, \Qtilde_{t+1,i}(\Delta_{t+1,i})
  \|^2}{ \F_{t} }  \nonumber \\ & = \CPE{ \|V_{t,i} - \mf_i(\hatS_{t})
  +\alpha \Delta_{t+1,i} \|^2}{ \F_{t} } +
\alpha^2\CPE{ \| \Qtilde_{t+1,i}(\Delta_{t+1,i}) - \Delta_{t+1,i}\|^2}{
  \F_{t} }  \nonumber \\
  & \leq \CPE{ \|V_{t,i} - \mf_i(\hatS_{t})
  +\alpha \Delta_{t+1,i} \|^2}{ \F_{t} } +
\alpha^2 \omega_\proba \CPE{ \| \Delta_{t+1,i}\|^2}{
  \F_{t} } \label{eq:controlvariate:tool2}
   \end{align}
        by using \Cref{hyp:URVB}($\omega$) and \Cref{lem:PP}. 
Second, we use the  equality: for $x, y \in
        \rset^q$ and $\alpha \in \rset$, $\|(1-\alpha) x + \alpha
        y\|^2 = (1-\alpha) \|x\|^2 + \alpha \|y\|^2 - \alpha
        (1-\alpha) \|x-y\|^2$; this  yields by using \Cref{hyp:oracle}
 \begin{align} 
 &\PE\left[ \|V_{t,i} - \mf_i(\hatS_{t})
  +\alpha \Delta_{t+1,i} 
          \|^2 \vert \F_{t} \right]  \nonumber  \\
          & \qquad \qquad = \CPE{ \|\left( 1 - \alpha\right) \left(
  V_{t,i} - \mf_i(\hatS_t) \right) + \alpha
  \left(\Smem_{t+1,i} - \hatS_t - \mf_i(\hatS_{t}) \right) \|^2}{
  \F_{t} } \nonumber \\
  & \qquad \qquad \leq \left( 1 - \alpha\right) \|
          V_{t,i} - \mf_i(\hatS_t) \|^2 + \alpha \sigma_i^2
           - \alpha \left( 1 - \alpha\right)
           \CPE{ \| \Delta_{t+1,i}\|^2}{ \F_{t} }. \label{eq:controlvariate:tool3}
   \end{align}
        The proof follows from  \eqref{eq:controlvariate:tool1}, \eqref{eq:controlvariate:tool2} and \eqref{eq:controlvariate:tool3}. 
\end{proof}

\subsection{\Cref{theo:main} in the case $\omega_\proba>0$ and $\alpha =0$} 
\label{sec:theowithnullalpha}

Apply \Cref{prop:varcont-PP2} with a time varying coefficient $\beta$, denoted by $\beta_{t+1}$. Combined with \Cref{prop:lyapunov} and \Cref{prop:update_error}, this yields
\begin{align}
& \ce{f(\map(\hatS_{t+1})) + g(\map(\hatS_{t+1}))}{\F_t} + \rho_{t+1} \CPE{G_{t+1}}{\F_t}
 \leq
  f(\map(\hatS_{t})) + g(\map(\hatS_{t}))  + \rho_t G_t  \\
  & +\pas_{t+1} \left( \frac{\rho_{t+1}}{\pas_{t+1}}  \beta_{t+1} \!\! \, \left( 1 +   \frac{2 \omega_\proba  }{\beta_{t+1}-1}  \frac{v_{\max}}{v_{\min}}  \frac{L^2}{n}
 \pas_{t+1}^2 \right) \!+ \!  \left( 1 -\bar C  v_{\min}^{-1} \pas_{t+1}  \right) \omega_\proba v_{\max}  - \frac{\rho_t}{\pas_{t+1}} \right) G_t  \label{eq:Gt:coeff:alphanull}  \\
  & -    \frac{\pas_{t+1}}{4} \left(  1  -
 2 \bar C  v_{\min}^{-1} \pas_{t+1}   -  \frac{8 \beta_{t+1}}{\beta_{t+1}-1}  v_{\min}^{-1}  \frac{L^2}{n}
 \rho_{t+1} \pas_{t+1}  \right)  \CPE{\barerror_{t+1}}{\F_t} \label{eq:bareps:alphanull} \\
  &+  \pas_{t+1} \left( \rho_{t+1} \beta_{t+1} \frac{ 2 (1+\omega_\proba)}{\beta_{t+1}-1}  \frac{v_{\max}}{v_{\min}}  \frac{L^2}{n} 
 \pas_{t+1}  +   \left( 1 -\bar C  v_{\min}^{-1} \pas_{t+1}  \right) (1+\omega_\proba) v_{\max}  \right) \frac{\sigma^2}{n}. \label{eq:sigma2:alphanull}
\end{align}
Following the same lines as in the proof of the case $\omega_\proba >0, \alpha>0$, choose $\pas_{t+1}>0, \rho_{t+1}>0$ and $\beta_{t+1}>1$ such that the term in \eqref{eq:Gt:coeff:alphanull} is non positive, the term between parenthesis in \eqref{eq:bareps:alphanull} is positive and the term in \eqref{eq:sigma2:alphanull} is proportional to $\pas_{t+1}^u$ for some $u \geq 1$.

\noindent $\star$ When $\alpha=0$, we can not choose $\pas_{t+1} = \pas$ and $\rho_{t+1} = \rho \pas$ as in the case $\alpha >0$. Observe indeed that with this strategy, \eqref{eq:Gt:coeff:alphanull} gets into
\[
\rho  \beta_{t+1} \!\! \, \left( 1 +   \frac{2 \omega_\proba  }{\beta_{t+1}-1}  \frac{v_{\max}}{v_{\min}}  \frac{L^2}{n}
 \pas^2 \right) \!+ \!  \left( 1 -\bar C  v_{\min}^{-1} \pas  \right) \omega_\proba v_{\max}  - \rho
\]
and it can not be made non positive whatever  $\pas, \rho$ positive (we have $\beta_{t+1} >1$).

\noindent $\star$ Let us consider the case when $\lim_t \pas_{t+1} =0$.  Let $\epsilon>0$, $u \in \ooint{0,1}$ and choose 
\begin{equation}\label{eq:stepsize:alphanull}
(1+2\epsilon) \pas_{t+1}^u \leq \pas_t^u,\qquad \rho_{t+1} \eqdef \pas_{t+1}^u, \qquad \beta_{t+1} \eqdef \frac{1}{1+\epsilon}\frac{\pas_t^u}{\pas_{t+1}^u};
 \end{equation}
 by convention, set $\pas_0 \eqdef  (1+2\epsilon)^{1/u}\pas_1$.
Then, \eqref{eq:Gt:coeff:alphanull} is upper bounded by 
\[
\left( \frac{\pas_t^u}{{1+\epsilon} } + 2 \omega_\proba (2 + \epsilon^{-1}) \frac{v_{\max}}{v_{\min}} \frac{L^2}{n} \pas_{t+1}^{2+u}  + \omega_\proba v_{\max} \pas_{t+1} - \pas_t^u\right) G_t,
\]
where we used that $\beta_{t+1} /(\beta_{t+1}-1)$ is upper bounded by  $2 +\epsilon^{-1}$. This term is nonpositive by choosing $\pas_1$ small enough; details are left to the interested reader. \eqref{eq:bareps:alphanull} is upper bounded by 
\[
-    \frac{\pas_{t+1}}{4} \left(  1  -
 2 \bar C  v_{\min}^{-1} \pas_{t+1}   -  8 (2 + \epsilon^{-1}) v_{\min}^{-1}  \frac{L^2}{n}
  \pas_{t+1}^{1+u}  \right)  \CPE{\barerror_{t+1}}{\F_t} 
\]
which is upper bounded by $ -\nicefrac{\pas_{t+1}}{8}  \CPE{\barerror_{t+1}}{\F_t}$ as soon as $\pas_1$  is chosen small enough. Finally, \eqref{eq:sigma2:alphanull} is upper bounded by
\[
\pas_{t+1} \left(    2 (1+\omega_\proba)(2 + \epsilon^{-1})  \frac{v_{\max}}{v_{\min}}  \frac{L^2}{n} 
 \pas_{t+1}^{1+u}  +   \left( 1 -\bar C  v_{\min}^{-1} \pas_{t+1}  \right) (1+\omega_\proba) v_{\max}  \right) \frac{\sigma^2}{n},
\]
which is upper bounded by $ 2 \pas_{t+1} (1+\omega_\proba) v_{\max} \sigma^2/n$ by choosing $\pas_1$ small enough. Taking the expectations and summing from $t=0$ to $t=T-1$ yields
\begin{align*}
\frac{1}{8}\sum_{t=1}^T \pas_t   \PE\left[\barerror_{t} \right] & \leq  \PE\left[ (f+g)\circ \map(\hatS_0) \right] -\min (f+g)  +
 (1+2\epsilon)^2 \pas_1 G_0 \\
& + 2 (1+\omega_\proba) v_{\max} \frac{\sigma^2}{n}
\sum_{t=1}^T \pas_t. \end{align*}
In this strategy, the stepsize sequence $\{\pas_t,t \geq 1 \}$  is majorized by a sequence decreasing to zero at a geometric rate (see \eqref{eq:stepsize:alphanull}) and $\lim_{T \to \infty} \sum_{t=1}^T \pas_t < \infty$.  Dividing the LHS and the RHS by $\sum_{t=1}^T \pas_t$ shows that the  RHS can not be made small, even when the number of iterations $T$ is large.

\section{The assumption \Cref{hyp:DL2} on few examples} \label{sec:checkDL2examples}

\subsection{Quadratic surrogate}
Let us prove \Cref{hyp:DL2} for \Cref{ex:gdt}. Since $g$ is a proper lower semicontinuous convex function, $\partial g$ is a monotone operator (see e.g. \cite[Example 20.3]{bauschke:combettes:2011}) which implies that the proximal operator $\Prox_{\rho g}$ is firmly nonexpansive. Therefore, for all $s,s'$, we have 
\begin{equation} \label{eq:A8:gdtprox}
\| \left( s' - \Prox_{\rho g}(s')  \right) - \left( s - \Prox_{\rho g}(s) \right) \|^2 \leq  \|s' -s \|^2,
\end{equation}
and
\begin{equation} \label{eq:A8:gdtprox:2}
\left| \pscal{ \Prox_{\rho g}(s') - \Prox_{\rho g}(s)}{s'-s} \right| \leq \|s' -s \| \, \|\Prox_{\rho g}(s') - \Prox_{\rho g}(s) \|  \leq \| s'-s \|^2.
\end{equation}
Since $\phi(\map(s)) = \rho^{-1} \Prox_{\rho g}(s) $, we deduce from \eqref{eq:A8:gdtprox} and \eqref{eq:A8:gdtprox:2} that \Cref{hyp:DL2} is verified with
\[
\B(s) \eqdef \rho^{-1} \, \Id,  \qquad v_{\min} = v_{\max} = \rho^{-1}, \qquad C_\star \eqdef \rho^{-2}, \qquad C_{\star \star} \eqdef \rho^{-1}. 
\]

\subsection{Jensen surrogate}
EM is a special case of Jensen surrogate (see \Cref{app:exampleEM}).

We check \Cref{hyp:DL2} for EM applied to Poisson observations when $\PE_\pi[\sample]$  is explicit (see \Cref{app:em_first_example}).  We choose $\calS$ to be equal to the compact interval $\ccint{-M,0}$ for some $M>0$. This implies that $\lambda-s \in \ccint{\lambda,\lambda+M}$. For any $s,s' \in \calS$, we write
\begin{align*}
  \phi(\map(s')) - \phi(\map(s)) & = \PE_\pi\left[ \sample \right] \left( \frac{1}{\lambda
    -s'} - \frac{1}{\lambda-s} \right) \\ & = \frac{\PE_\pi\left[ \sample \right] }{\lambda
    -s} \frac{(s'-s)}{\lambda -s'} = \frac{\PE_\pi\left[ \sample \right] }{(\lambda -s)^2}
  (s'-s) + \frac{\PE_\pi\left[ \sample \right] }{(\lambda -s)^2} \frac{(s'-s)^2}{\lambda-s'}.
\end{align*}
Hence \Cref{hyp:DL2} is verified with
\[
\B(s) \eqdef \frac{\PE_\pi\left[ \sample \right] }{(\lambda-s)^2}, \quad v_{\min} \eqdef
\frac{1}{(\lambda+M)^2}, \quad v_{\max} \eqdef \frac{1}{\lambda^2},
\quad C_\star \eqdef \frac{\PE_\pi\left[ \sample \right] }{\lambda^3}, \quad C_{\star \star}
\eqdef \frac{\PE_\pi\left[ \sample \right] }{\lambda^2}.
\]

\subsection{Variational surrogate}\label{app:hyp:DL}
Let us consider the variational surrogate applied to Dictionary Learning / Matrix Factorization (see \Cref{subsec:ssmmeg}).  We have
\[
\phi(\param) = \left[ \begin{matrix}- \param^\top \param  \\
2 \param  \end{matrix} \right]
\]
and $s = (s^{(1)}, s^{(2)}) \in \mathcal{M}^+_K \times \rset^{p \times K}$.  Here, $\B(s)$ is not an invertible
matrix so that the condition $v_{\min} >0$ does not hold. To be
convinced, consider the case when $p=K=1$ so that $\phi: \rset \to \rset^2$ is equal to $(-u^2, 2u)$; 
and $s \eqdef (s^{(1)}, s^{(2)}) \in \rset_{>0} \times \rset$. Then
${\map}(s) = s^{(2)}/s^{(1)}$. It is easily checked that we have 
\[
\B(s) = \left[ \begin{matrix}  \frac{1}{s^{(1)}} & - \frac{s^{(2)}}{(s^{(1)})^2} \\ 
- \frac{s^{(2)}}{(s^{(1)})^2} &  \frac{(s^{(2)})^2}{(s^{(1)})^3}
\end{matrix} \right] = \frac{1}{s^{(1)}} \left[ \begin{matrix} 1 \\ -
  \frac{s^{(2)}}{s^{(1)}}\end{matrix} \right] \, \left[ \begin{matrix}
  1 \\ - \frac{s^{(2)}}{s^{(1)}}\end{matrix} \right]^\top.
\]
We developed the convergence analysis under the assumption that $\B(s)$
in invertible (with a positive minimal eigenvalue uniformly lower
bounded for $s \in \calS$). When this assumption does not hold,
the proof has to be generalized in order to define projections and
scalar product in a space defined by the image of $\B(\hatS_t)$. This
makes the proof more intricate, it is out of the scope of
this paper.

%% file: main.bbl
\begin{thebibliography}{92}
\providecommand{\natexlab}[1]{#1}
\providecommand{\url}[1]{\texttt{#1}}
\expandafter\ifx\csname urlstyle\endcsname\relax
  \providecommand{\doi}[1]{doi: #1}\else
  \providecommand{\doi}{doi: \begingroup \urlstyle{rm}\Url}\fi

\bibitem[Alistarh et~al.(2017)Alistarh, Grubic, Li, Tomioka, and Vojnovic]{alistarh_qsgd_2017}
D.~Alistarh, D.~Grubic, J.~Li, R.~Tomioka, and M.~Vojnovic.
\newblock {QSGD}: {Communication}-{Efficient} {SGD} via {Gradient} {Quantization} and {Encoding}.
\newblock \emph{Advances in Neural Information Processing Systems}, 30:\penalty0 1709--1720, 2017.

\bibitem[Amos et~al.(2017)Amos, Xu, and Kolter]{amos17bICNN}
B.~Amos, L.~Xu, and J.~Kolter.
\newblock {Input Convex Neural Networks}.
\newblock In D.~Precup and Y.~W. Teh, editors, \emph{Proceedings of the 34th International Conference on Machine Learning}, volume~70, pages 146--155. PMLR, 06--11 Aug 2017.

\bibitem[Andrieu et~al.(2005)Andrieu, Moulines, and Priouret]{andrieu:moulines:priouret:2005}
C.~Andrieu, E.~Moulines, and P.~Priouret.
\newblock {Stability of Stochastic Approximation under Verifiable Conditions}.
\newblock \emph{SIAM Journal on Control and Optimization}, 44\penalty0 (1):\penalty0 283--312, 2005.

\bibitem[Atchad{{\'e}} et~al.(2017)Atchad{{\'e}}, Fort, and Moulines]{atchade:fort:moulines:2017}
Y.~Atchad{{\'e}}, G.~Fort, and E.~Moulines.
\newblock {On Perturbed Proximal Gradient Algorithms}.
\newblock \emph{Journal of Machine Learning Research}, 18\penalty0 (10):\penalty0 1--33, 2017.

\bibitem[Bauschke and Combettes(2011)]{bauschke:combettes:2011}
H.~H. Bauschke and P.~L. Combettes.
\newblock \emph{Convex Analysis and Monotone Operator Theory in Hilbert Spaces}.
\newblock Springer Publishing Company, Incorporated, 2011.

\bibitem[Beck and Teboulle(2009{\natexlab{a}})]{beck2009fast}
A.~Beck and M.~Teboulle.
\newblock {A Fast Iterative Shrinkage-Thresholding Algorithm for Linear Inverse Problems}.
\newblock \emph{SIAM journal on imaging sciences}, 2\penalty0 (1):\penalty0 183--202, 2009{\natexlab{a}}.

\bibitem[Beck and Teboulle(2009{\natexlab{b}})]{beck:teboulle:2010}
A.~Beck and M.~Teboulle.
\newblock \emph{Gradient-based algorithms with applications to signal-recovery problems}, page 42–88.
\newblock Cambridge University Press, 2009{\natexlab{b}}.

\bibitem[Benveniste et~al.(1990)Benveniste, M{\'{e}}tivier, and Priouret]{benveniste:etal:1990}
A.~Benveniste, M.~M{\'{e}}tivier, and P.~Priouret.
\newblock \emph{{Adaptive Algorithms and Stochastic Approximations}}.
\newblock Springer Verlag, 1990.

\bibitem[Borkar(2008)]{borkar:2008}
V.~S. Borkar.
\newblock \emph{{Stochastic Approximation. A Dynamical Systems Viewpoint}}.
\newblock Cambridge University Press, Cambridge; Hindustan Book Agency, New Delhi, 2008.

\bibitem[Bradley et~al.(2000)Bradley, Bennett, and Demiriz]{bradley2000constrained}
P.~Bradley, K.~Bennett, and A.~Demiriz.
\newblock {Constrained K-Means Clustering}.
\newblock \emph{Microsoft Research, Redmond}, 20\penalty0 (0):\penalty0 0, 2000.

\bibitem[Brown(1986)]{brown:1986}
L.~D. Brown.
\newblock \emph{{Fundamentals of Statistical Exponential Families with Applications in Statistical Decision Theory}}.
\newblock Lecture notes-monograph series Fundamentals of statistical exponential families. Institute of Mathematical Statistics, 1986.

\bibitem[Canh et~al.(2020)Canh, Nguyen, and Josh]{t2020personalize}
T.~Canh, T.~Nguyen, and N.~Josh.
\newblock {Personalized Federated Learning with Moreau Envelopes}.
\newblock In H.~Larochelle, M.~Ranzato, R.~Hadsell, M.~Balcan, and H.~Lin, editors, \emph{Advances in Neural Information Processing Systems}, volume~33, pages 21394--21405. Curran Associates, Inc., 2020.

\bibitem[Capp{\'e} and Moulines(2009)]{cappe:moulines:2009}
O.~Capp{\'e} and {\'E}.~Moulines.
\newblock {On-Line Expectation–Maximization Algorithm for Latent Data Models}.
\newblock \emph{Journal of the Royal Statistical Society: Series B (Statistical Methodology)}, 71\penalty0 (3):\penalty0 593--613, 2009.

\bibitem[Celeux and Diebolt(1992)]{celeux:diebolt:1992}
G.~Celeux and J.~Diebolt.
\newblock {A stochastic approximation type EM algorithm for the mixture problem}.
\newblock \emph{Stochastics and Stochastic Reports}, 41\penalty0 (1-2):\penalty0 119--134, 1992.

\bibitem[Chen et~al.(2018)Chen, Zhu, Teh, and Zhang]{chen:etal:2018}
J.~Chen, J.~Zhu, Y.~Teh, and T.~Zhang.
\newblock {Stochastic Expectation Maximization with Variance Reduction}.
\newblock In S.~Bengio, H.~Wallach, H.~Larochelle, K.~Grauman, N.~Cesa-Bianchi, and R.~Garnett, editors, \emph{Advances in Neural Information Processing Systems 31}, pages 7967--7977, 2018.

\bibitem[Collins et~al.(2002)Collins, Schapire, and Singer]{collins2002logistic}
M.~Collins, R.~E. Schapire, and Y.~Singer.
\newblock {Logistic Regression, AdaBoost and Bregman Distances}.
\newblock \emph{Machine Learning}, 48\penalty0 (1):\penalty0 253--285, 2002.

\bibitem[Combettes and Pesquet(2011)]{combettes:pesquet:2011}
P.~L. Combettes and J.-C. Pesquet.
\newblock \emph{Proximal Splitting Methods in Signal Processing}, pages 185--212.
\newblock Springer New York, New York, NY, 2011.

\bibitem[Condat and Richt{\'a}rik(2023)]{condat2022randprox}
L.~Condat and P.~Richt{\'a}rik.
\newblock {RandProx: Primal-Dual Optimization Algorithms with Randomized Proximal Updates}.
\newblock In \emph{The Eleventh International Conference on Learning Representations}. OpenReview.net, 2023.

\bibitem[Dai et~al.(2019)Dai, Yan, Zhou, Yang, Ng, Cheng, and Fan]{dai2019hyper}
X.~Dai, X.~Yan, K.~Zhou, H.~Yang, K.~K. Ng, J.~Cheng, and Y.~Fan.
\newblock {Hyper-Sphere Quantization: Communication-Efficient SGD for Federated Learning}.
\newblock Technical report, arXiv:1911.04655, 2019.

\bibitem[Della~Pietra et~al.(2001)Della~Pietra, Della~Pietra, and Lafferty]{della2001duality}
S.~Della~Pietra, V.~Della~Pietra, and J.~Lafferty.
\newblock {Duality and Auxiliary Functions for Bregman Distances}.
\newblock Technical report, Carnegie-Mellon University, 2001.

\bibitem[Delyon et~al.(1999)Delyon, Lavielle, and Moulines]{Delyon:etal:1999}
B.~Delyon, M.~Lavielle, and E.~Moulines.
\newblock {Convergence of a Stochastic Approximation Version of the {EM} Algorithm}.
\newblock \emph{Ann. Statist.}, 27\penalty0 (1):\penalty0 94--128, 1999.

\bibitem[Dempster et~al.(1977)Dempster, Laird, and Rubin]{dempster:etal:1977}
A.~P. Dempster, N.~M. Laird, and D.~B. Rubin.
\newblock {Maximum Likelihood from Incomplete Data via the {EM} Algorithm}.
\newblock \emph{Journal of the Royal Statistical Society: Series B}, 39:\penalty0 1--38, 1977.

\bibitem[Dieuleveut et~al.(2021)Dieuleveut, Fort, Moulines, and Robin]{dieuleveut:etal:FedEM:2021}
A.~Dieuleveut, G.~Fort, E.~Moulines, and G.~Robin.
\newblock {Federated-EM with heterogeneity mitigation and variance reduction}.
\newblock In M.~Ranzato, A.~Beygelzimer, Y.~Dauphin, P.~Liang, and J.~W. Vaughan, editors, \emph{Advances in Neural Information Processing Systems}, volume~34, pages 29553--29566. Curran Associates, Inc., 2021.

\bibitem[Dieuleveut et~al.(2023)Dieuleveut, Fort, Moulines, and Wai]{dieuleveut:etal:2023:TSP}
A.~Dieuleveut, G.~Fort, E.~Moulines, and H.-T. Wai.
\newblock {Stochastic Approximation Beyond Gradient for Signal Processing and Machine Learning}.
\newblock \emph{IEEE Transactions on Signal Processing}, 71:\penalty0 3117--3148, 2023.

\bibitem[Dinh et~al.(2021)Dinh, Vu, Tran, Dao, and Zhang]{dinh2021fedu}
C.~Dinh, T.~T. Vu, N.~H. Tran, M.~N. Dao, and H.~Zhang.
\newblock {FedU: A Unified Framework for Federated Multi-Task Learning with Laplacian Regularization}.
\newblock Technical report, arXiv:2102.07148, 2021.

\bibitem[Dudley(1969)]{dudley1969speed}
R.~M. Dudley.
\newblock The speed of mean glivenko-cantelli convergence.
\newblock \emph{The Annals of Mathematical Statistics}, 40\penalty0 (1):\penalty0 40--50, 1969.

\bibitem[Efron et~al.(2004)Efron, Hastie, Johnstone, and Tibshirani]{Efron_2004}
B.~Efron, T.~Hastie, I.~Johnstone, and R.~Tibshirani.
\newblock {Least Angle Regression}.
\newblock \emph{The Annals of Statistics}, 32\penalty0 (2), April 2004.

\bibitem[Flanagan et~al.(2020)Flanagan, Oyomno, Grigorievskiy, Tan, Khan, and Ammad-Ud-Din]{flanagan2020federated}
A.~Flanagan, W.~Oyomno, A.~Grigorievskiy, K.~E. Tan, S.~Khan, and M.~Ammad-Ud-Din.
\newblock {Federated Multi-view Matrix Factorization for Personalized Recommendations}.
\newblock In \emph{Joint European Conference on Machine Learning and Knowledge Discovery in Databases}, pages 324--347. Springer, 2020.

\bibitem[Fort and Moulines(2003)]{Fort:moulines:2003}
G.~Fort and E.~Moulines.
\newblock {Convergence of the Monte Carlo Expectation Maximization for Curved Exponential Families}.
\newblock \emph{Ann. Statist.}, 31\penalty0 (4):\penalty0 1220--1259, 2003.

\bibitem[Fort and Moulines(2023)]{fort:moulines:2023}
G.~Fort and E.~Moulines.
\newblock Stochastic variable metric proximal gradient with variance reduction for non-convex composite optimization.
\newblock \emph{Stat Comput}, 33\penalty0 (65), 2023.

\bibitem[Fort et~al.(2021{\natexlab{a}})Fort, Gach, and Moulines]{gach:fort:moulines:2020}
G.~Fort, P.~Gach, and E.~Moulines.
\newblock {Fast incremental expectation maximization for finite-sum optimization: nonasymptotic convergence}.
\newblock \emph{Statistics and Computing}, 31, 2021{\natexlab{a}}.

\bibitem[Fort et~al.(2021{\natexlab{b}})Fort, Moulines, and Wai]{fort:moulines:wai:2021}
G.~Fort, E.~Moulines, and H.-T. Wai.
\newblock {Geom-Spider-EM: Faster Variance Reduced Stochastic Expectation Maximization for Nonconvex Finite-Sum Optimization}.
\newblock In \emph{ICASSP 2021 - 2021 IEEE International Conference on Acoustics, Speech and Signal Processing (ICASSP)}, pages 3135--3139, 2021{\natexlab{b}}.

\bibitem[Gao et~al.(2020)Gao, Tan, Ju, Zheng, and Yang]{gao2020privacy}
D.~Gao, B.~Tan, C.~Ju, V.~W. Zheng, and Q.~Yang.
\newblock {Privacy Threats Against Federated Matrix Factorization}.
\newblock Technical report, arXiv:2007.01587, 2020.

\bibitem[Ghadimi and Lan(2013)]{ghadimi:lan:2013}
S.~Ghadimi and G.~Lan.
\newblock {Stochastic First- and Zeroth-Order Methods for Nonconvex Stochastic Programming}.
\newblock \emph{SIAM J. Optimiz.}, 23\penalty0 (4):\penalty0 2341--2368, 2013.

\bibitem[Ghosh et~al.(2020)Ghosh, Chung, Yin, and Ramchandran]{ghosh2020efficient}
A.~Ghosh, J.~Chung, D.~Yin, and K.~Ramchandran.
\newblock {An Efficient Framework for Clustered Federated Learning}.
\newblock \emph{Advances in Neural Information Processing Systems}, 33:\penalty0 19586--19597, 2020.

\bibitem[Gorbunov et~al.(2021)Gorbunov, Burlachenko, and Richt{\'a}rik]{gorbunov2021marina}
E.~Gorbunov, Z.~Burlachenko, K. P.and~Li, and P.~Richt{\'a}rik.
\newblock {MARINA: Faster Non-Convex Distributed Learning with Compression}.
\newblock In \emph{International Conference on Machine Learning}, pages 3788--3798. PMLR, 2021.

\bibitem[Hanzely et~al.(2020)Hanzely, Hanzely, Horv{\'a}th, and Richt{\'a}rik]{hanzely2020lower}
F.~Hanzely, S.~Hanzely, S.~Horv{\'a}th, and P.~Richt{\'a}rik.
\newblock {Lower Bounds and Optimal Algorithms for Personalized Federated Learning}.
\newblock \emph{Advances in Neural Information Processing Systems}, 33:\penalty0 2304--2315, 2020.

\bibitem[Harper and Konstan(2015)]{harper2015movielens}
F.~Harper and J.~A. Konstan.
\newblock {The Movielens Datasets: History and context}.
\newblock \emph{ACM Transactions on interactive intelligent systems}, 5\penalty0 (4):\penalty0 1--19, 2015.

\bibitem[Heged{\H{u}}s et~al.(2019)Heged{\H{u}}s, Danner, and Jelasity]{hegedHus2019decentralized}
I.~Heged{\H{u}}s, G.~Danner, and M.~Jelasity.
\newblock {Decentralized Recommendation Based on Matrix Factorization: A Comparison of Gossip and Federated Learning}.
\newblock In \emph{Joint European Conference on Machine Learning and Knowledge Discovery in Databases}, pages 317--332. Springer, 2019.

\bibitem[Islamov et~al.(2021)Islamov, Qian, and Richt{\'a}rik]{islamov2021distributed}
R.~Islamov, X.~Qian, and P.~Richt{\'a}rik.
\newblock {Distributed Second Order Methods with Fast Rates and Compressed Communication}.
\newblock In \emph{International Conference on Machine Learning}, pages 4617--4628. PMLR, 2021.

\bibitem[Jain et~al.(1999)Jain, Murty, and Flynn]{jain1999data}
A.~K. Jain, M.~N. Murty, and P.~J. Flynn.
\newblock {Data clustering: a review}.
\newblock \emph{ACM computing surveys (CSUR)}, 31\penalty0 (3):\penalty0 264--323, 1999.

\bibitem[Kairouz et~al.(2021)Kairouz, McMahan, et~al.]{kairouz_advances_2019}
P.~Kairouz, H.~McMahan, et~al.
\newblock \emph{Advances and Open Problems in Federated Learning}.
\newblock New Foundations and Trends, 2021.

\bibitem[Karimireddy et~al.(2019)Karimireddy, Rebjock, Stich, and Jaggi]{karimireddy_error_2019}
S.~P. Karimireddy, Q.~Rebjock, S.~Stich, and M.~Jaggi.
\newblock Error {Feedback} {Fixes} {SignSGD} and other {Gradient} {Compression} {Schemes}.
\newblock In \emph{International {Conference} on {Machine} {Learning}}, pages 3252--3261. PMLR, May 2019.
\newblock ISSN: 2640-3498.

\bibitem[Karimireddy et~al.(2020)Karimireddy, Kale, Mohri, Reddi, Stich, and Suresh]{karimireddy_scaffold_2019}
S.~P. Karimireddy, S.~Kale, M.~Mohri, S.~Reddi, S.~Stich, and A.~T. Suresh.
\newblock {SCAFFOLD: Stochastic Controlled Averaging for Federated Learning}.
\newblock In \emph{International Conference on Machine Learning}, pages 5132--5143. PMLR, 2020.

\bibitem[Khaled and Richtárik(2020)]{khaled_gradient_2020}
A.~Khaled and P.~Richtárik.
\newblock Gradient {Descent} with {Compressed} {Iterates}.
\newblock Technical report, arXiv:1909.04716, 2020.

\bibitem[Kingma and Ba(2017)]{kingma_adam_2017}
D.~P. Kingma and J.~Ba.
\newblock Adam: {A} {Method} for {Stochastic} {Optimization}.
\newblock \emph{arXiv:1412.6980 [cs]}, January 2017.
\newblock arXiv: 1412.6980.

\bibitem[Koloskova et~al.(2019)Koloskova, Stich, and Jaggi]{koloskova2019decentralized}
A.~Koloskova, S.~Stich, and M.~Jaggi.
\newblock {Decentralized Stochastic Optimization and Gossip Algorithms with Compressed Communication}.
\newblock In \emph{International conference on machine learning}, pages 3478--3487. PMLR, 2019.

\bibitem[Konečný et~al.(2016)Konečný, McMahan, Yu, Richtarik, Suresh, and Bacon]{konecny_federated_2016-1}
J.~Konečný, H.~B. McMahan, F.~X. Yu, P.~Richtarik, A.~T. Suresh, and D.~Bacon.
\newblock Federated {Learning}: {Strategies} for {Improving} {Communication} {Efficiency}.
\newblock In \emph{{NIPS} {Workshop} on {Private} {Multi}-{Party} {Machine} {Learning}}, 2016.

\bibitem[Korotin et~al.(2021{\natexlab{a}})Korotin, Egiazarian, Asadulaev, Safin, and Burnaev]{korotin2019wasserstein}
A.~Korotin, V.~Egiazarian, A.~Asadulaev, A.~Safin, and E.~Burnaev.
\newblock {Wasserstein-2 Generative Networks}.
\newblock In \emph{International Conference on Learning Representations}, 2021{\natexlab{a}}.

\bibitem[Korotin et~al.(2021{\natexlab{b}})Korotin, Li, Genevay, Solomon, Filippov, and Burnaev]{korotin2021do}
A.~Korotin, L.~Li, A.~Genevay, J.~Solomon, A.~Filippov, and E.~Burnaev.
\newblock {Do Neural Optimal Transport Solvers Work? A Continuous Wasserstein-2 Benchmark}.
\newblock In A.~Beygelzimer, Y.~Dauphin, P.~Liang, and J.~W. Vaughan, editors, \emph{Advances in Neural Information Processing Systems}, 2021{\natexlab{b}}.

\bibitem[Kunstner et~al.(2021)Kunstner, Kumar, and Schmidt]{kunstner2021homeomorphic}
F.~Kunstner, R.~Kumar, and M.~Schmidt.
\newblock {Homeomorphic-Invariance of EM: Non-Asymptotic Convergence in KL Divergence for Exponential Families via Mirror Descent}.
\newblock In \emph{International Conference on Artificial Intelligence and Statistics}, pages 3295--3303. PMLR, 2021.

\bibitem[Lange(2013)]{lange:2013}
K.~Lange.
\newblock \emph{Optimization}.
\newblock Springer Texts in Statistics. Springer New York, 2013.

\bibitem[Lange et~al.(2000)Lange, Hunter, and Yang]{lange2000optimization}
K.~Lange, D.~Hunter, and I.~Yang.
\newblock {Optimization Transfer Using Surrogate Objective Functions}.
\newblock \emph{Journal of computational and graphical statistics}, 9\penalty0 (1):\penalty0 1--20, 2000.

\bibitem[Leconte et~al.(2021)Leconte, Dieuleveut, Oyallon, Moulines, and Pages]{leconte2021texttt}
L.~Leconte, A.~Dieuleveut, E.~Oyallon, E.~Moulines, and G.~Pages.
\newblock {DoStoVoQ: Doubly Stochastic Voronoi Vector Quantization SGD for Federated Learning}.
\newblock Technical report, openreview.net, 2021.

\bibitem[Lee and Seung(1999)]{lee1999learning}
D.~Lee and H.~S. Seung.
\newblock Learning the parts of objects by non-negative matrix factorization.
\newblock \emph{Nature}, 401\penalty0 (6755):\penalty0 788--791, 1999.

\bibitem[Li et~al.(2020)Li, Huang, Yang, Wang, and Zhang]{li_convergence_2019}
X.~Li, K.~Huang, W.~Yang, S.~Wang, and Z.~Zhang.
\newblock On the {Convergence} of {FedAvg} on {Non}-{IID} {Data}.
\newblock In \emph{International Conference on Learning Representations}, 2020.

\bibitem[Mairal(2013)]{mairal:2013}
J.~Mairal.
\newblock {Stochastic Majorization-Minimization Algorithms for Large-Scale Optimization}.
\newblock In C.~Burges, L.~Bottou, M.~Welling, Z.~Ghahramani, and K.~Weinberger, editors, \emph{Advances in Neural Information Processing Systems}, volume~26. Curran Associates, Inc., 2013.

\bibitem[Mairal(2015)]{mairal:2015}
J.~Mairal.
\newblock {Incremental Majorization-Minimization Optimization with Application to Large-Scale Machine Learning}.
\newblock \emph{SIAM Journal on Optimization}, 25\penalty0 (2):\penalty0 829--855, 2015.

\bibitem[Mairal et~al.(2010)Mairal, Bach, Ponce, and Sapiro]{mairal:etal:2010}
J.~Mairal, F.~Bach, J.~Ponce, and G.~Sapiro.
\newblock Online learning for matrix factorization and sparse coding.
\newblock \emph{J. Mach. Learn. Res.}, 11:\penalty0 19–60, 2010.

\bibitem[Makkuva et~al.(2020)Makkuva, Taghvaei, Oh, and Lee]{makkuva20a_icnn_ot}
A.~Makkuva, A.~Taghvaei, S.~Oh, and J.~Lee.
\newblock Optimal transport mapping via input convex neural networks.
\newblock In H.~D. III and A.~Singh, editors, \emph{Proceedings of the 37th International Conference on Machine Learning}, volume 119 of \emph{Proceedings of Machine Learning Research}, pages 6672--6681. PMLR, 2020.

\bibitem[Mansour et~al.(2020)Mansour, Mohri, Ro, and Suresh]{mansour2020three}
Y.~Mansour, M.~Mohri, J.~Ro, and A.~Suresh.
\newblock {Three Approaches for Personalization with aApplications to Federated Learning}.
\newblock Technical report, arXiv:2002.10619, 2020.

\bibitem[Marfoq et~al.(2021)Marfoq, Neglia, Bellet, Kameni, and Vidal]{marfoq2021federated}
O.~Marfoq, G.~Neglia, A.~Bellet, L.~Kameni, and R.~Vidal.
\newblock {Federated Multi-Task Learning under a Mixture of Distributions}.
\newblock \emph{Advances in Neural Information Processing Systems}, 34, 2021.

\bibitem[McLachlan and Krishnan(2008)]{maclachlan:2008}
G.~McLachlan and T.~Krishnan.
\newblock \emph{{The EM Algorithm and Extensions}}.
\newblock Wiley series in probability and statistics. Wiley, 2008.

\bibitem[McMahan et~al.(2017{\natexlab{a}})McMahan, Moore, Ramage, Hampson, and Aguera~y Arcas]{mcmahan2016communication}
B.~McMahan, E.~Moore, D.~Ramage, S.~Hampson, and B.~Aguera~y Arcas.
\newblock {Communication-Efficient Learning of Deep Networks from Decentralized Data}.
\newblock In \emph{Artificial intelligence and statistics}, pages 1273--1282. PMLR, 2017{\natexlab{a}}.

\bibitem[McMahan et~al.(2017{\natexlab{b}})McMahan, Moore, Ramage, Hampson, and Arcas]{mcmahan_communication-efficient_2017}
B.~McMahan, E.~Moore, D.~Ramage, S.~Hampson, and B.~A.~y. Arcas.
\newblock Communication-{Efficient} {Learning} of {Deep} {Networks} from {Decentralized} {Data}.
\newblock In \emph{Artificial {Intelligence} and {Statistics}}, pages 1273--1282. PMLR, April 2017{\natexlab{b}}.
\newblock ISSN: 2640-3498.

\bibitem[Mishchenko et~al.(2019)Mishchenko, Gorbunov, Takáč, and Richtárik]{mishchenko_distributed_2019}
K.~Mishchenko, E.~Gorbunov, M.~Takáč, and P.~Richtárik.
\newblock Distributed {Learning} with {Compressed} {Gradient} {Differences}.
\newblock Technical report, arXiv:1901.09269, 2019.

\bibitem[Mishchenko et~al.(2022)Mishchenko, Malinovsky, Stich, and Richtarik]{MishchenkoProxSkip22}
K.~Mishchenko, G.~Malinovsky, S.~Stich, and P.~Richtarik.
\newblock {P}rox{S}kip: Yes! {L}ocal gradient steps provably lead to communication acceleration! {F}inally!
\newblock In K.~Chaudhuri, S.~Jegelka, L.~Song, C.~Szepesvari, G.~Niu, and S.~Sabato, editors, \emph{Proceedings of the 39th International Conference on Machine Learning}, volume 162 of \emph{Proceedings of Machine Learning Research}, pages 15750--15769, 2022.

\bibitem[Moreau(1962)]{Moreau:1962}
J.-J. Moreau.
\newblock Fonctions convexes duales et points proximaux dans un espace hilbertien.
\newblock \emph{C. R. Acad. Sci., Paris}, 255:\penalty0 2897--2899, 1962.

\bibitem[Neal and Hinton(1998)]{Neal:hinton:1998}
R.~M. Neal and G.~E. Hinton.
\newblock {A View of the EM Algorithm that Justifies Incremental, Sparse, and other Variants}.
\newblock In M.~I. Jordan, editor, \emph{Learning in Graphical Models}, pages 355--368. Springer Netherlands, Dordrecht, 1998.

\bibitem[Nguyen et~al.(2023)Nguyen, Forbes, Fort, and Capp{\'e}]{nguyen:etal:2022}
H.~Nguyen, F.~Forbes, G.~Fort, and O.~Capp{\'e}.
\newblock {An Online Minorization-Maximization Algorithm}.
\newblock In P.~Brito, J.~G. Dias, B.~Lausen, A.~Montanari, and R.~Nugent, editors, \emph{Classification and Data Science in the Digital Age}, pages 263--271. Springer International Publishing, 2023.

\bibitem[Paatero and Tapper(1994)]{paatero1994positive}
P.~Paatero and U.~Tapper.
\newblock Positive matrix factorization: A non-negative factor model with optimal utilization of error estimates of data values.
\newblock \emph{Environmetrics}, 5\penalty0 (2):\penalty0 111--126, 1994.

\bibitem[Parikh and Boyd(2014)]{parikh:boyd:2013}
N.~Parikh and S.~Boyd.
\newblock {Proximal Algorithms}.
\newblock \emph{Found. Trends Optim.}, 1\penalty0 (3):\penalty0 127–239, 2014.

\bibitem[Philippenko and Dieuleveut(2021)]{philippenko_preserved_2021}
C.~Philippenko and A.~Dieuleveut.
\newblock Preserved central model for faster bidirectional compression in distributed settings.
\newblock In A.~Beygelzimer, Y.~Dauphin, P.~Liang, and J.~W. Vaughan, editors, \emph{Advances in Neural Information Processing Systems}, 2021.

\bibitem[Polyak(2021)]{polyak:2021}
R.~Polyak.
\newblock \emph{{Introduction to Continuous Optimization}}.
\newblock Springer Optimization and Its Applications. Springer, April 2021.

\bibitem[Reddi et~al.(2021)Reddi, Charles, Zaheer, Garrett, Rush, Kone{\v{c}}n{\'y}, S., and McMahan]{reddi2021adaptive}
S.~Reddi, Z.~Charles, M.~Zaheer, Z.~Garrett, K.~Rush, J.~Kone{\v{c}}n{\'y}, K.~S., and H.~McMahan.
\newblock {Adaptive Federated Optimization}.
\newblock In \emph{International Conference on Learning Representations}, 2021.

\bibitem[Richt{\'a}rik et~al.(2021)Richt{\'a}rik, Sokolov, and Fatkhullin]{richtarik2021ef21}
P.~Richt{\'a}rik, I.~Sokolov, and I.~Fatkhullin.
\newblock {{EF}21: A New, Simpler, Theoretically Better, and Practically Faster Error Feedback}.
\newblock In A.~Beygelzimer, Y.~Dauphin, P.~Liang, and J.~W. Vaughan, editors, \emph{Advances in Neural Information Processing Systems}, 2021.

\bibitem[Robbins and Monro(1951)]{robbins_stochastic_1951}
H.~Robbins and S.~Monro.
\newblock A {Stochastic} {Approximation} {Method}.
\newblock \emph{Annals of Mathematical Statistics}, 22\penalty0 (3):\penalty0 400--407, 1951.

\bibitem[Sato and Ishii(2000)]{sato:ishii:2000}
M.-A. Sato and S.~Ishii.
\newblock {On-line EM Algorithm for the Normalized Gaussian Network}.
\newblock \emph{Neural Computation}, 12\penalty0 (2):\penalty0 407--432, 2000.

\bibitem[Sattler et~al.(2019)Sattler, Wiedemann, Müller, and Samek]{sattler_robust_2019}
F.~Sattler, S.~Wiedemann, K.-R. Müller, and W.~Samek.
\newblock Robust and {Communication}-{Efficient} {Federated} {Learning} {From} {Non}-i.i.d. {Data}.
\newblock \emph{IEEE Transactions on Neural Networks and Learning Systems}, pages 1--14, 2019.

\bibitem[Sattler et~al.(2020)Sattler, M{\"u}ller, and Samek]{sattler2020clustered}
F.~Sattler, K.-R. M{\"u}ller, and W.~Samek.
\newblock {Clustered Federated Learning: Model-Agnostic Distributed Multitask Optimization under Privacy Constraints}.
\newblock \emph{IEEE Transactions on Neural Networks and Learning Systems}, 32\penalty0 (8):\penalty0 3710--3722, 2020.

\bibitem[Singhal et~al.(2021)Singhal, Sidahmed, Garrett, Wu, Rush, and Prakash]{singhal2021federated}
K.~Singhal, H.~Sidahmed, Z.~Garrett, S.~Wu, J.~Rush, and S.~Prakash.
\newblock {Federated Reconstruction: Partially Local Federated Learning}.
\newblock In M.~Ranzato, A.~Beygelzimer, Y.~Dauphin, P.~Liang, and J.~W. Vaughan, editors, \emph{Advances in Neural Information Processing Systems}, volume~34, pages 11220--11232. Curran Associates, Inc., 2021.

\bibitem[Stich(2019)]{stich_local_2019}
S.~Stich.
\newblock {Local {SGD} Converges Fast and Communicates Little}.
\newblock In \emph{International Conference on Learning Representations}, 2019.

\bibitem[Tian et~al.(2024)Tian, Weng, and Feng]{tiantowards2024icml}
Y.~Tian, H.~Weng, and Y.~Feng.
\newblock {Towards the Theory of Unsupervised Federated Learning: Non-asymptotic Analysis of Federated EM Algorithms}.
\newblock In \emph{Forty-first International Conference on Machine Learning}, 2024.

\bibitem[Villani(2021)]{villani2021topics}
C.~Villani.
\newblock \emph{{Topics in Optimal Transportation}}, volume~58.
\newblock American Mathematical Soc., 2021.

\bibitem[Wang et~al.(2021)Wang, Charles, et~al.]{wang2021field}
J.~Wang, Z.~Charles, et~al.
\newblock {A Field Guide to Federated Optimization}.
\newblock Technical report, arXiv:2107.06917, 2021.

\bibitem[Wang and Chang(2022)]{wang:chang:2022}
S.~Wang and T.-H. Chang.
\newblock {Federated Matrix Factorization: Algorithm Design and Application to Data Clustering}.
\newblock \emph{IEEE Transactions on Signal Processing}, 70:\penalty0 1625--1640, 2022.

\bibitem[Wangni et~al.(2018)Wangni, Wang, Liu, and Zhang]{wangni2018gradient}
J.~Wangni, J.~Wang, J.~Liu, and T.~Zhang.
\newblock {Gradient Sparsification for Communication-Efficient Distributed Optimization}.
\newblock In S.~Bengio, H.~Wallach, H.~Larochelle, K.~Grauman, N.~Cesa-Bianchi, and R.~Garnett, editors, \emph{Advances in Neural Information Processing Systems}, volume~31. Curran Associates, Inc., 2018.

\bibitem[Wei and Tanner(1990)]{Wei:tanner:1990}
G.~Wei and M.~Tanner.
\newblock {A Monte Carlo Implementation of the EM Algorithm and the Poor Man's Data Augmentation Algorithms}.
\newblock \emph{J. Am. Stat. Assoc.}, 85\penalty0 (411):\penalty0 699--704, 1990.

\bibitem[Woodworth et~al.(2020)Woodworth, Patel, Stich, Dai, Bullins, Mcmahan, Shamir, and Srebro]{woodworth_is_2020}
B.~Woodworth, K.~Patel, S.~Stich, Z.~Dai, B.~Bullins, B.~Mcmahan, O.~Shamir, and N.~Srebro.
\newblock Is {Local} {SGD} {Better} than {Minibatch} {SGD}?
\newblock In \emph{International {Conference} on {Machine} {Learning}}, pages 10334--10343. PMLR, 2020.

\bibitem[Wright et~al.(2009)Wright, Nowak, and Figueiredo]{wright2009sparse}
S.~J. Wright, R.~Nowak, and M.~Figueiredo.
\newblock {Sparse Reconstruction by Separable Approximation}.
\newblock \emph{IEEE Transactions on signal processing}, 57\penalty0 (7):\penalty0 2479--2493, 2009.

\bibitem[Yuan et~al.(2021)Yuan, Zaheer, and Reddi]{yuan2021federated}
H.~Yuan, M.~Zaheer, and S.~Reddi.
\newblock {Federated Composite Optimization}.
\newblock In \emph{International Conference on Machine Learning}, pages 12253--12266. PMLR, 2021.

\bibitem[Zhixu et~al.(2024)Zhixu, Rajita, and Sanjeev]{tao2024convergence}
T.~Zhixu, C.~Rajita, and K.~Sanjeev.
\newblock {On the Convergence of a Federated Expectation-Maximization Algorithm}.
\newblock Technical report, arXiv:2408.05819, 2024.

\end{thebibliography}
